\documentclass[10pt]{article}
\usepackage[utf8]{inputenc}
\usepackage{geometry}
\usepackage{bbm}
\usepackage{amsmath,amssymb,amsthm}
\newtheorem{theorem}{Theorem}[section]
\newtheorem{proposition}[theorem]{Proposition}

\newtheorem{corollary}[theorem]{Corollary}
\newtheorem{lemma}[theorem]{Lemma}
\newtheorem{definition}{Definition}
\newtheorem{remark}{Remark}

\newtheorem{example}{Example}
\newtheorem{assumption}{Assumption}[section]
\usepackage{graphicx}
\usepackage{subfig}
\usepackage{float}
\usepackage[toc,page]{appendix}
\usepackage[table,xcdraw]{xcolor}
\newcommand{\R}{\mathbb{R}}
\newcommand{\E}{\mathbb{E}}
\newcommand{\F}{\mathcal{F}}

\newcommand{\PS}{\mathcal{P}}
\newcommand{\prob}{\mathbb{P}}
\newcommand{\N}{\mathcal{N}}
\newcommand{\HS}{\mathcal{H}}

\renewcommand{\a}{\mathbf{a}}
\newcommand{\x}{\mathbf{x}}
\newcommand{\y}{\mathbf{y}}
\newcommand{\z}{\mathbf{z}}
\newcommand{\w}{\mathbf{w}}

\newcommand{\s}{\mathbf{s}}
\newcommand{\vb}{\mathbf{v}}
\newcommand{\ep}{\epsilon}
\newcommand{\lb}{\langle}
\newcommand{\rb}{\rangle}
\newcommand{\KL}{\text{KL}}
\newcommand{\argmin}{\text{argmin}}
\newcommand{\diag}{\text{diag}}
\newcommand{\Lip}{\text{Lip}}
\newcommand{\TV}{\text{TV}}
\newcommand{\sprt}{\text{sprt}}
\newcommand{\law}{\text{law}}
\newcommand{\tr}{\text{Tr}}
\newcommand{\vs}{\vspace{0.5em}}
\usepackage{hyperref}
\allowdisplaybreaks
\usepackage{authblk}

\title{A Mathematical Framework for Learning Probability Distributions}
\author{Hongkang Yang
\thanks{
    {\tt hongkang@princeton.edu}.
    We thank Weinan E for helpful discussions.
}}
\affil{Program in Applied and Computational Mathematics, Princeton University}

\date{}

\begin{document}

\maketitle

\begin{abstract}
The modeling of probability distributions, specifically generative modeling and density estimation, has become an immensely popular subject in recent years by virtue of its outstanding performance on sophisticated data such as images and texts.
Nevertheless, a theoretical understanding of its success is still incomplete.
One mystery is the paradox between memorization and generalization: In theory, the model is trained to be exactly the same as the empirical distribution of the finite samples, whereas in practice, the trained model can generate new samples or estimate the likelihood of unseen samples.
Likewise, the overwhelming diversity of distribution learning models calls for a unified perspective on this subject.
This paper provides a mathematical framework such that all the well-known models can be derived based on simple principles.
To demonstrate its efficacy, we present a survey of our results on the approximation error, training error and generalization error of these models, which can all be established based on this framework.
In particular, the aforementioned paradox is resolved by proving that these models enjoy implicit regularization during training, so that the generalization error at early-stopping avoids the curse of dimensionality.
Furthermore, we provide some new results on landscape analysis and the mode collapse phenomenon.
\end{abstract}

\tableofcontents

\section{Introduction}
\label{sec. introduction}

The popularity of machine learning models in recent years is largely attributable to their remarkable versatility in solving highly diverse tasks with good generalization power.
Underlying this diversity is the ability of the models to learn various mathematical objects such as functions, probability distributions, dynamical systems, actions and policies, and often a sophisticated architecture or training scheme is a composition of these modules.
Besides fitting functions, learning probability distributions is arguably the most widely-adopted task and constitutes a great portion of the field of unsupervised learning.
Its applications range from the classical density estimation \cite{wang2019nonparametric,valsson2014potential} which is important for scientific computing \cite{bonati2019enhanced,wang2022efficient,li2018renormalization}, to generative modeling with superb performance in image synthesis and text composition \cite{brock2018BigGAN,ramesh2021zero,brown2020language,rombach2022high}, and also to pretraining tasks such as masked reconstruction that are crucial for large-scale models \cite{devlin2018bert,brown2020language,chowdhery2022palm}.

Despite the impressive performance of machine learning models in learning probability measures, this subject is less understood than the learning of functions or supervised learning.
Specifically, there are several mysteries:

\textbf{1. Unified framework.}
There are numerous types of models for representing and estimating distributions, making it difficult to gain a unified perspective for model design and comparison.
One traditional categorization includes five model classes: the generative adversarial networks (GAN) \cite{goodfellow2014generative,arjovsky2017wasserstein}, variational autoencoders (VAE) \cite{kingma2013auto}, normalizing flows (NF) \cite{tabak2010density,rezende2015inference}, autoregressive models \cite{radford2018improving,oord2018parallel}, and diffusion models \cite{sohl2015deep,song2020score}.
Within each class, there are further variations that complicate the picture, such as the choice of integral probability metrics for GANs and the choice of architectures for normalizing flows that enable likelihood computations.
Ideally, instead of a phenomenological categorization, one would prefer a simple theoretical framework that can derive all these models in a straightforward manner based on a few principles.

\textbf{2. Memorization and curse of dimensionality.}
Perhaps the greatest difference between learning functions and learning probability distributions is that, conceptually, the solution of the latter problem must be trivial.
On one hand, since the target distribution $P_*$ can be arbitrarily complicated, any useful model must satisfy the property of universal convergence, namely the modeled distribution can be trained to converge to any given distribution (e.g. Section \ref{sec. memorization} will show that this property holds for several models).
On the other hand, the target $P_*$ is unknown in practice and only a finite sample set $\{\x_i\}_{i=1}^n$ is given (with the empirical distribution denoted by $P_*^{(n)}$).
As a result, the modeled distribution $P_t$ can only be trained with $P_*^{(n)}$ and inevitably exhibits memorization, i.e.
\begin{equation*}
\lim_{t\to\infty} P_t = P_*^{(n)}
\end{equation*}
Hence, training results in a trivial solution and does not provide us with anything beyond the samples we already have.
This is different from regression where the global minimizer (interpolating solution) can still generalize well \cite{e2019min}.

One related problem is the curse of dimensionality, which becomes more severe when estimating distributions instead of functions.
In general, the distance between the hidden target and the empirical distribution scales badly with dimension $d$: For any absolutely continuous $P_*$ and any $\delta>0$ \cite{weed2019sharp}
\begin{equation*}
W_2(P_*,P_*^{(n)}) ~\gtrsim~ n^{-\frac{1}{d-\delta}}
\end{equation*}
where $W_2$ is the Wasserstein metric.
This slow convergence sets a limit on the performance of all possible models:
for instance, the following worst-case lower bound \cite{singh2018minimax}
\begin{equation*}
\inf_{A} \sup_{P_*} ~\E_{\{X_i\}}\big[W_2^2\big(P_*,A(\{X_i\}_{i=1}^n)\big)\big]^{1/2} ~\gtrsim ~n^{-\frac{1}{d}}
\end{equation*}
where $P_*$ is any distribution supported on $[0,1]^d$ and $A$ is any estimator, i.e. a mapping from every $n$ sample set $\{X_i\}_{i=1}^n \sim P_*$ to an estimated distribution $A(\{X_i\})$.
Hence, to achieve a generalization error of $\ep$, an astronomical sample size $\Omega(\ep^d)$ could be necessary in high dimensions.

These theoretical difficulties form a seeming paradox with the empirical success of distribution learning models, for instance, models that can generate novel and highly-realistic images \cite{brock2018BigGAN,karras2019style,ramesh2021zero} and texts \cite{radford2018improving,brown2020language}.

\textbf{3. Training and mode collapse.}
The training of distribution learning models is known to be more delicate than training supervised learning models, and exhibits several novel forms of failures.
For instance, for the GAN model, one common issue is mode collapse \cite{salimans2016improved,kodali2017convergence,mao2019mode,pei2021alleviating}, when a positive amount of mass in $P_t$ becomes concentrated at a single point, e.g. an image generator could consistently output the same image.
Another issue is mode dropping \cite{yazici2020empirical}, when $P_t$ fails to cover some of the modes of $P_*$.
In addition, training may suffer from oscillation and divergence \cite{radford2015unsupervised,chavdarova2018sgan}.
These problems are the main obstacle to global convergence, but the underlying mechanism remains largely obscure.

\vs
The goal of this paper is to provide some insights into these mysteries from a mathematical point of view.
Specifically,
\begin{enumerate}
\item We establish a unified theoretical framework from which all the major distribution learning models can be derived.
The diversity of these models is largely determined by two simple factors, the distribution representation and loss type.
This formulation greatly facilitates our analysis of the approximation error, training error and generalization error of these models.

\item We survey our previous results on generalization error, and resolve the paradox between memorization and generalization.
As illustrated in Figure \ref{fig: path of generalization}, despite that the model eventually converges to the global minimizer, which is the memorization solution, the training trajectory comes very close to the hidden target distribution.
With early-stopping or regularized loss, the generalization error scales as
\begin{equation*}
W_2(P_*,P_t) \text{ or } \KL(P_*\|P_t) \lesssim n^{-\alpha}
\end{equation*}
for some constant $\alpha>0$ instead of dimension-dependent terms such as $\alpha/d$.
Thereby, the model escapes from the curse of dimensionality.

\item We discuss our previous results on the rates of global convergence for some of the models.
For the other models, we establish new results on landscape and critical points, and identify two mechanisms that can lead to mode collapse.
\end{enumerate}

\begin{figure}[H]
    \centering
    \includegraphics[scale=0.35]{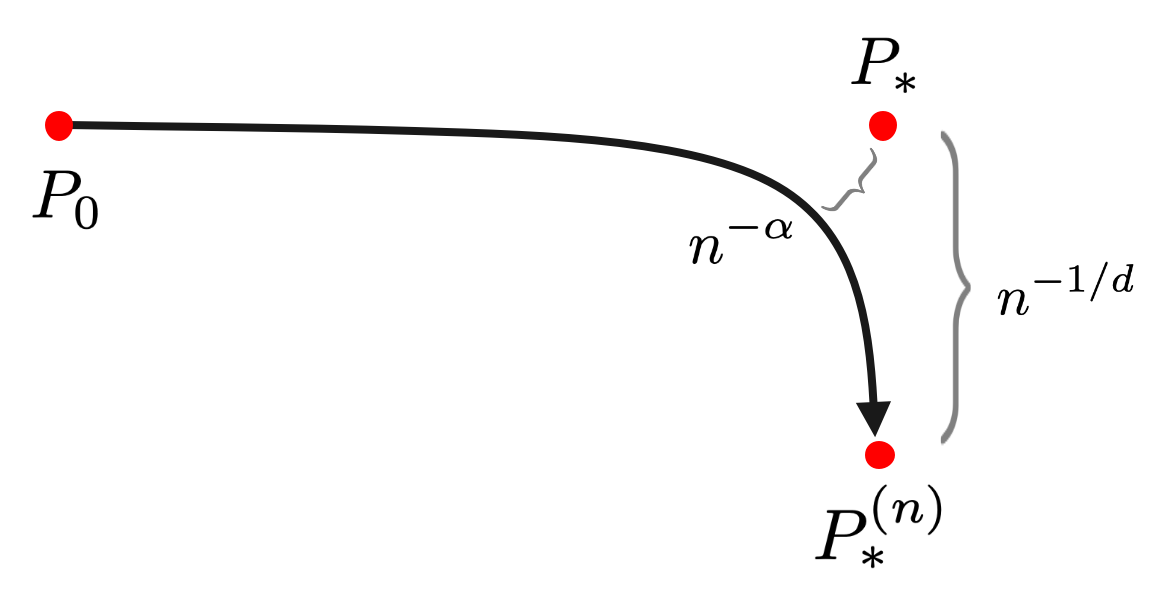}
    \caption{Generalization error during training.}
    \label{fig: path of generalization}
\end{figure}

This paper is structured as follows.
Section \ref{sec. model overview} presents a sketch of the popular distribution learning models.
Section \ref{sec. framework} introduces our theoretical framework and the derivations of the models.
Section \ref{sec. approximation} establishes the universal approximation theorems.
Section \ref{sec. memorization} analyzes the memorization phenomenon.
Section \ref{sec. generalization} discusses the generalization error of several representative distribution learning models.
Section \ref{sec. training} analyzes the training process, loss landscape and mode collapse.
All proofs are contained in Section \ref{sec. proof}.
Section \ref{sec. discussion} concludes this paper with discussion on the remaining mysteries.

\vs
Here is a list of related works.

\textbf{Mathematical framework}: A framework for supervised learning has been proposed by \cite{e2019machine,e2020NNML} with focus on the function representations, namely function spaces that can be discretized into neural networks.
This framework is helpful for the analysis of supervised learning models, in particular, the estimation of generalization errors \cite{e2018priori,e2019residual,e2019min} that avoid the curse of dimensionality, and the determination of global convergence \cite{chizat2018global,rotskoff2019trainability}.
Similarly, our framework for distribution learning emphasizes the function representation, as well as the new factor of distribution representation, and bound the generalization error through analogous arguments.
Meanwhile, there are frameworks that characterizes distribution learning from other perspectives, for instance, statistical inference \cite{hu2017unifying,su2018variational}, graphical models and rewards \cite{zhang2022unifying,bengio2021gflownet}, energy functions \cite{zhao2016energy} and biological neurons \cite{ororbia2022neural}.

\textbf{Generalization ability}:
The section on generalization reviews our previous results on the generalization error estimates for potential-based model \cite{yang2022potential}, GAN \cite{yang2022GAN} and normalizing flow with stochastic interpolants \cite{yang2022flow}.
The mechanism is that function representations defined by integral transforms or expectations \cite{e2019barron,e2019machine} enjoy small Rademacher complexity and thus escape from the curse of dimensionality.
Earlier works \cite{e2018priori,e2019residual,e2019min,e2020kolmogorov} used this mechanism to bound the generalization error of supervised learning models.
Our analysis combines this mechanism with the training process to show that early-stopping solutions generalize well, and is related to concepts from supervised learning literature such as the frequency principle \cite{xu2019frequency,xu2022overview} and slow deterioration \cite{ma2020slow}.

\textbf{Training and convergence}:
The additional factor of distribution representation further complicates the loss landscape, and makes training more difficult to analyze, especially for the class of ``free generators" that will be discussed later.
The model that attracted the most attention was GAN, and convergence has only been established in simplified settings \cite{mescheder2018training,wu2019onelayer,feizi2020LQG,balaji2021understanding,yang2022GAN} or for local saddle points \cite{nagarajan2017gradient,heusel2017gans,lin2020gradient}.
In practice, GAN training is known to suffer from failures such as mode collapse and divergence \cite{arora2018GAN,che2016mode,mescheder2018training}.
Despite that these issues can be fixed using regularizations and tricks \cite{gulrajani2017improved,kodali2017convergence,mao2019mode,pei2021alleviating,hoang2018mgan}, the mechanism underlying this training instability is not well understood.

\section{Model Overview}
\label{sec. model overview}

This section offers a quick sketch of the prominent models for learning probability distributions, while their derivations will be presented in Section \ref{sec. framework}.
These models are commonly grouped into five categories: the generative adversarial networks (GAN), autoregressive models, variational autoencoders (VAE), normalizing flows (NF), and diffusion models.
We assume access to samples drawn from the target distribution $P_*$, and the task is to train a model to be able to generate more samples from the distribution (generative modeling) or compute its density function (density estimation).

\vspace{0.5em}
\textbf{1. GAN.}
The generative adversarial networks \cite{goodfellow2014generative} model a distribution by transport
\begin{equation}
\label{eq. overview generator}
P = \law(X), \quad X=G(Z), \quad Z\sim\prob
\end{equation}
The map $G:\R^d\to\R^d$ is known as the generator and $\prob$ is a base distribution that is easy to sample (e.g. unit Gaussian).
To solve for a generator such that $P=P_*$, the earliest GAN model considers the following optimization problem \cite{goodfellow2014generative}
\begin{equation*}
\min_G \max_D \int \log \Big( \frac{e^{D(\x)}}{1+e^{D(\x)}} \Big) dP_*(\x) + \int \log \Big( \frac{1}{1+e^{D(G(\x))}} \Big) d\prob(\x)
\end{equation*}
where $D:\R^d\to\R$ is known as the discriminator and this type of min-max losses are known as the adversarial loss.
A well-known variant is the WGAN \cite{arjovsky2017wasserstein} defined by
\begin{equation*}
\min_G \max_{\|\theta\|_{\infty} \leq 1} \int D_{\theta}(\x) dP_*(\x)- \int D_{\theta}\big(G(\x)\big) d\prob(\x)
\end{equation*}
where the discriminator $D_{\theta}$ is a neural network with parameter $\theta$, which is bounded in $l^{\infty}$ norm.
For the other variants, a survey on the GAN models is given by \cite{gui2021review}.

\vspace{0.5em}
\textbf{2. VAE.}
The variational autoencoder proposed by \cite{kingma2013auto} uses a randomized generator and its approximate inverse, known as the decoder and encoder, and we denote them by the conditional distributions $P(\cdot|\z)$ and $Q(\cdot|\x)$.
Similar to (\ref{eq. overview generator}), the distribution is modeled by $P=\int P(\cdot|\z)d\prob(\z)$ and can be sampled by $X\sim P(\cdot|Z), Z\sim \prob$.
VAE considers the following optimization problem
\begin{align*}
\min_{P(\cdot|\z)} \min_{Q(\cdot|\x)} \iint -\log P(\x|\z) dQ(\z|\x) + \KL\big(Q(\cdot|\x) \big\| \prob \big) dP_*(\x)
\end{align*}
where $\KL$ is the Kullback–Leibler divergence.
To simplify computation, $\prob$ is usually set to be the unit Gaussian $\N$, and the decoder $P(\cdot|\z)$ and encoder $Q(\cdot|\x)$ are parametrized as diagonal Gaussians \cite{kingma2013auto}.
For instance, consider
\begin{equation*}
P(\cdot|\z) = \N\big(G(\z), s^2 I\big), \quad Q(\cdot|\x) = \N\big(F(\x), v^2 I\big)
\end{equation*}
where $G,F$ are parametrized functions and $s,v>0$ are scalars.
Then, we have
\begin{align*}
\int -\log P(\x|\z) dQ(\z|\x) &= \int -\log P\big(\x\big|F(\x)+v\omega\big) d\N(\omega)\\
&= \int \frac{\big\|\x-G\big(F(\x)+v\omega\big)\big\|^2}{2s^2} d\N(\omega) + \frac{d}{2} \log(2\pi s^2)\\
\KL\big(Q(\cdot|\x)\big\|\prob\big) &= \frac{\|F(\x)\|^2}{2} + \frac{d}{2}\big(v^2 -\log v^2 -1\big)
\end{align*}
Up to constant, the VAE loss becomes
\begin{align*}
\min_{G,s} \min_{F,v} &\iint \frac{\big\|\x-G\big(F(\x)+v\omega\big)\big\|^2}{2s^2} d\N(\omega) + \frac{\|F(\x)\|^2}{2} dP_*(\x) + d(\log s -\log v) + \frac{d}{2}v^2
\end{align*}

\vs
\textbf{3. Autoregressive.}
Consider sequential data $\mathbf{X}=[\x_1, \dots \x_l]$ such as text and audio.
The autoregressive models represent a distribution $P$ through factorization $P(\mathbf{X})=\prod_{i=1}^l P(\x_i|\x_{<i})$, and can be sampled by sampling iteratively from $X_i \sim P(X_i|X_{<i})$.
These models minimize the loss
\begin{equation*}
-\int \sum_{i=1}^l \log P(\x_i|\x_{<i}) dP_*(\mathbf{X})
\end{equation*}
There are several approaches to the parametrization of the conditional distributions $P(\x_i|\x_{<i})$, depending on how to process the variable-length input $\x_{<i}$.
The common options are the Transformer networks \cite{vaswani2017attention,radford2018improving}, recurrent networks \cite{oord2016pixel}, autoregressive networks \cite{larochelle2011neural, oord2016pixel} and causal convolution \cite{oord2016wavenet}.
For instance, consider Gaussian distributions parametrized by a recurrent network
\begin{equation*}
P(\cdot|\x_{<i}) = \N\big(m(\mathbf{h}_i), s^2(\mathbf{h}_i)I\big), \quad \mathbf{h}_i = f(\mathbf{h}_{i-1}, \x_{i-1})
\end{equation*}
where $m,s,f$ are parametrized functions, and $\mathbf{h}$ is the hidden feature.
Since
\begin{equation*}
-\log P(\x_i|\x_{<i}) = \frac{\|\x_i-m(\mathbf{h}_i)\|^2}{2s(\mathbf{h}_i)^2} + \frac{d}{2} \log \big( 2\pi s(\mathbf{h}_i)^2 \big)
\end{equation*}
the loss is equal, up to constant, to
\begin{equation*}
\min_{m,s,f,\mathbf{h}_1} \int \sum_{i=1}^l \frac{\|\x_i-m(\mathbf{h}_i)\|^2}{2s(\mathbf{h}_i)^2} + d\log s(\mathbf{h}_i) ~dP_*(\mathbf{X})
\end{equation*}

\vs
\textbf{4. NF.}
The normalizing flows proposed by \cite{tabak2010density, tabak2013family} use a generator $G$ (\ref{eq. overview generator}) similar to GAN and VAE.
%which is modeled as a flow (i.e. time-dependent diffeomorphism) in the data space
The optimization problem is given by
\begin{equation*}
\min_G - \int \log\det \nabla G^{-1}(\x) + \log \prob(G^{-1}(\x)) dP_*(\x)
\end{equation*}
where $\prob=\N$ is set to be the unit Gaussian and $\det \nabla G^{-1}$ is the Jacobian determinant of $G^{-1}$.
%This model was popularized by \cite{rezende2015inference} that trains all parameters of the flow simultaneously instead of incrementally as in \cite{tabak2010density, tabak2013family}.
To enable the calculation of these terms, the earliest approach \cite{tabak2010density,tabak2013family,rezende2015inference} considers only the inverse $F=G^{-1}$ and models it by a concatenation of simple maps $F = F_1 \circ \dots \circ F_T$ such that each $\det \nabla F_{\tau}$ is easy to compute.
The modeled distribution (\ref{eq. overview generator}) cannot be sampled, but can serve as a density estimator, with the density given by
\begin{equation*}
P(\x) = \N(\x_0) \prod_{\tau=1}^T \det\nabla F_{\tau}(\x_{\tau-1}), \quad \x_{\tau-1} := F_{\tau} \circ \dots \circ F_T(\x)
\end{equation*}
Up to constant, the loss becomes
\begin{equation*}
\min_{F_1, \dots F_T} \int \frac{\|\x_0\|^2}{2} + \sum_{\tau=1}^T -\log\det\nabla F_{\tau}(\x_{\tau-1}) ~dP_*(\x)
\end{equation*}
A later approach \cite{dinh2016density,kingma2016improved, papamakarios2017masked, huang2018neural} models the generator by $G = G_T \circ \dots \circ G_1$
such that each $G_{\tau}$ is designed to be easily invertible with $\nabla G_{\tau}$ being a triangular matrix.
Then, the loss becomes
\begin{equation*}
\min_{G_1, \dots G_T} \int \frac{\|\x_0\|^2}{2} + \sum_{\tau=1}^T \tr\big[\log\nabla G_{\tau}(\x_{\tau-1})\big] ~dP_*(\x), \quad \x_{\tau-1} := G_{\tau}^{-1} \circ \dots \circ G_T^{-1}(\x)
\end{equation*}
Another approach \cite{e2017flow,chen2018neural,grathwohl2018ffjord} defines $G$ as a continuous-time flow, i.e. solution to an ordinary differential equation (ODE)
\begin{equation}
\label{eq. ODE generator}
G=G_0^1, \quad G_s^{\tau}(\x_s) = \x_{\tau}, \quad \frac{d}{d\tau} \x_{\tau} = V(\x_{\tau}, \tau)
\end{equation}
for some time-dependent velocity field $V$.
Then, the loss becomes
\begin{equation*}
\min_V \int \frac{\|\x_0\|^2}{2} + \int_0^1 \tr \big[\nabla_{\x} V(\x_{\tau}, \tau)\big] d\tau dP_*(\x), \quad \x_{\tau} := (G_{\tau}^1)^{-1}(\x)
\end{equation*}
A survey on normalizing flows is given by \cite{kobyzev2020normalizing}.
%In practice, the continuous flow can be approximated by forward Euler or Runge-Kutta scheme. One can further reduce the error using higher-order Taylor series \cite{chen2020residual}.

\vs
\textbf{5. Monge-Amp\`{e}re flow.}
A model that is closely related to the normalizing flows is the Monge-Amp\`{e}re flow \cite{zhang2018monge}.
It is parametrized by a time-dependent potential function $\phi_{\tau}, \tau\in[0,1]$, and defines a generator by the ODE
\begin{equation*}
G=G_0^1, \quad G_s^{\tau}(\x_s) = \x_{\tau}, \quad \frac{d}{d\tau} \x_{\tau} = \nabla \phi_{\tau}(\x_{\tau})
\end{equation*}
such that the flow is driven by a gradient field.
The model minimizes the following loss
\begin{equation*}
\min_{\phi} \int \frac{\|\x_0\|^2}{2} + \int_0^1 \Delta \phi_{\tau}(\x_{\tau}) d\tau dP_*(\x), \quad \x_{\tau} := (G_{\tau}^1)^{-1}(\x)
\end{equation*}
where $\Delta \phi = \sum_{i=1}^d \partial_i^2 \phi$ is the Laplacian.

\vs
\textbf{6. Diffusion.}
The diffusion models \cite{sohl2015deep,song2019generative,ho2020denoising} define the generator by a reverse-time SDE (stochastic differential equation)
\begin{equation*}
G(X_T)=X_0, \quad X_T\sim \prob, \quad d X_{\tau} = -\frac{\beta_{\tau}}{2} \big( X_{\tau} + 2\mathbf{s}(X_{\tau}, \tau)\big) d\tau + \sqrt{\beta_{\tau}} d\overline{W}_{\tau}
\end{equation*}
where $\mathbf{s}:\R^d\times[0,T]\to\R^d$ is a time-dependent velocity field known as the score function \cite{hyvarinen2005estimation}, $\beta_{\tau}>0$ is some noise scale, and $\overline{W}_{\tau}$ is a reverse-time Wiener process.
The modeled distribution is sampled by solving this SDE backwards in time from $T$ to $0$.
The score function $\mathbf{s}$ is learned from the optimization problem
\begin{equation*}
\min_{\mathbf{s}} \int_0^T \frac{\lambda_{\tau}}{2} \iint \Big\|\mathbf{s}\big(e^{-\frac{1}{2}\int_0^{\tau}\beta_s ds}\x_0 + \sqrt{1-e^{-\int_0^{\tau}\beta_s ds}} \omega,\tau\big) + \frac{\omega}{\sqrt{1-e^{-\int_0^{\tau}\beta_s ds}}} \Big\|^2 d\N(\omega) dP_*(\x_0) d\tau
\end{equation*}
where $\N$ is the unit Gaussian and $\lambda_{\tau}>0$ is any weight.
Besides the reverse-time SDE, another way to sample from the model is to solve the following reverse-time ODE, which yields the same distribution \cite{song2020score}
\begin{equation*}
G(X_T)=X_0, \quad X_T\sim \prob, \quad \frac{d}{d\tau} X_{\tau} = -\frac{\beta_{\tau}}{2} \big( X_{\tau} + \mathbf{s}(X_{\tau}, \tau)\big) d\tau
\end{equation*}
A survey on diffusion models is given by \cite{yang2022diffusion}.

\vs
\textbf{7. NF interpolant.}
Finally, we introduce a model called normalizing flow with stochastic interpolants \cite{albergo2022building,liu2022flow}.
It is analogous to the diffusion models, and yet is conceptually simpler.
This model learns a velocity field $V:\R^d\times[0,1]\to\R^d$ from the optimization problem
\begin{align*}
\min_V &\int_0^1\iint \big\|V\big((1-\tau)\x_0+\tau\x_1, \tau\big) - (\x_1-\x_0)\big\|^2 d\prob(\x_0)dP_*(\x_1)d\tau
\end{align*}
Then, the generator is defined through the ODE (\ref{eq. ODE generator}) and the modeled distribution is sampled by (\ref{eq. overview generator}).

\section{Framework}
\label{sec. framework}

Previously, a mathematical framework for supervised learning was proposed by \cite{e2019machine}, which was effective for estimating the approximation and generalization errors of supervised learning models \cite{e2019barron,e2018priori,e2019residual}.
In particular, it helped to understand how neural network-based models manage to avoid the curse of dimensionality.
The framework characterizes the models by four factors: the function representation (abstract function spaces built from integral transformations), the loss, the training scheme, and the discretization (e.g. how the continuous representations are discretized into neural networks with finite neurons).

This section presents a similar framework that unifies models for learning probability distributions.
We focus on two factors: the distribution representation, which is a new factor and determines how distributions are parametrized by abstract functions, and the loss type, which specifies which metric or topology is imposed upon the distributions.
A sketch of the categorization is given in Table \ref{table. categorization}.
We show that the diverse families of distribution learning models can be simultaneously derived from this framework.
The theoretical results in the latter sections, in particular the generalization error estimates, are also built upon this mathematical foundation.

\begin{table}[H]
\centering
\begin{tabular}{ |c||c|c|c| } 
\hline
 & Density & Expectation & Regression \\
\hline\hline
Potential & \cellcolor{gray!20}bias potential model & feasible & unknown \\ 
Free generator & \cellcolor{gray!20}NF, VAE, autoregressive & \cellcolor{gray!20}GAN & unknown \\ 
Fixed generator & upper bound & feasible & \cellcolor{gray!20}diffusion, NF interpolant, OT \\ 
\hline
\end{tabular}
\caption{Categorization of distribution learning models based on distribution representation (row) and loss type (column) with the representative models. See Section \ref{sec. combination} for a detailed description. Our theoretical results will focus on the marked categories.}
\label{table. categorization}
\end{table}

\subsection{Background}
\label{sec. background}

The basic task is to estimate a probability distribution given i.i.d. samples $\{\x_i\}_{i=1}^n$.
We denote this unknown target distribution by $P_*$ and the empirical distribution by
\begin{equation*}
P_*^{(n)} = \frac{1}{n}\sum_{i=1}^n \delta_{\x_i}
\end{equation*}
The underlying space is assumed to be Euclidean $\R^d$.
To estimate a distribution may have several meanings depending on its usage: e.g. to obtain a random variable $X\sim P_*$, estimate the density function $P_*(\x)$, or compute expectations $\int f dP_*$. The first task is known as generative modeling and the second as density estimation; these two problems are the focus of this paper, while the third task can be solved by them.

\begin{figure}[h]
\centering
\subfloat{\includegraphics[scale=0.32]{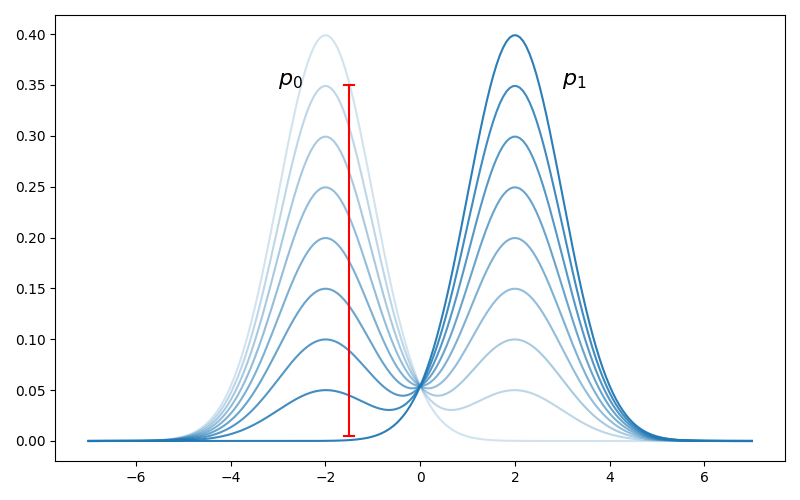}}~
\subfloat{\includegraphics[scale=0.32]{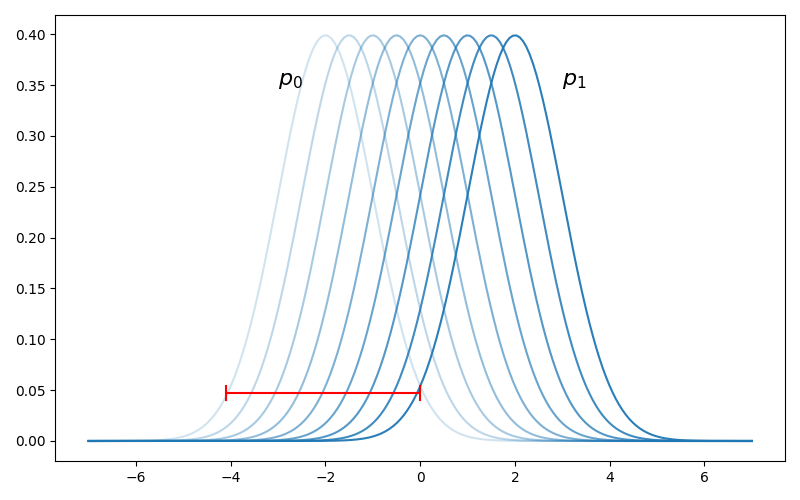}}
\caption{The vertical and horizontal perspectives on probability distributions. Left: the distribution $P_0=\N(-2,1)$ is transformed to $P_1=\N(2,1)$ by reweighing and the distance $d(P_0,P_1)$ is measured by the difference between densities. Right: $P_0$ is transformed to $P_1$ by transport, and the distance is measure by the displacement of mass.}
\label{fig. V and H}
\end{figure}

There are two general approaches to modeling a distribution, which can be figuratively termed as ``vertical" and ``horizontal", and an illustration is given by Figure \ref{fig. V and H}.
Given a base distribution $\prob$, the vertical approach reweighs the density of $\prob$ to approximate the density of the target $P_*$, while the horizontal approach transports the mass of $\prob$ towards the location of $P_*$.
When a modeled distribution $P$ is obtained and a distance $d(P,P_*)$ is needed to compute either the training loss or test error, the vertical approach measures the difference between the densities of $P$ and $P_*$ over each location, and is exemplified by the KL-divergence
\begin{equation}
\label{eq. KL}
\KL(P_*\|P) := \int \log\frac{P_*(\x)}{P(\x)} dP_*(\x)
\end{equation}
while the horizontal approach measures the distance between the ``particles" of $P$ and $P_*$, and is exemplified by the 2-Wasserstein metric \cite{kantorovich1960mathematical,villani2003topics}
\begin{equation}
\label{eq. W2 coupling}
W_2(P,P_*) := \inf_{\pi} \Big(\int \|\x_0-\x_1\|^2 d\pi(\x_0,\x_1)\Big)^{1/2}
\end{equation}
where $\pi$ is any coupling between $P,P_*$ (i.e. a joint distribution in $\R^d\times\R^d$ whose marginal distributions are $P,P_*$).
We will see that the vertical and horizontal approaches largely determine the distribution representation and loss type.

Finally, consider the operator $law$
\begin{equation*}
P = \law(X)
\end{equation*}
which maps a random variable $X$ to its distribution $P$.
Similarly, given a random path $\{X_{\tau}, \tau\in[0,T]\}$, we obtain a path $P_{\tau}=\law(X_{\tau})$ in the distribution space.
In general, there can be infinitely many random variables that are mapped to the same distribution, e.g. let $P$ and $\prob$ be uniform over $[0,1]$, for any $k\in\mathbb{N}$, we can define the random variable $X_k=kZ\text{mod}1$ with $Z\sim\prob$, which all satisfy $P=\law(X_k)$.
One drawback of this non-uniqueness is that, for many generative models, there can be plenty of global minima, which make the loss landscape non-convex and may lead to training failures, as we will show that this is inherent in mode collapse.
One benefit is that, if the task is to learn some time-dependent distribution $P_{\tau}$, one can select from the infinitely many possible random paths $X_{\tau}$ the one that is the easiest to compute, and therefore define a convenient loss.

\vspace{0.5em}
For the notations, for any measureable subset $\Omega$ of $\R^d$, denote by $\PS(\Omega)$ the space of probability measures over $\Omega$, $\PS_2(\Omega)$ the subspace of measures with finite second moments, and $\PS_{ac}(\Omega)$ the subspace of absolutely continuous measures (i.e. have density functions).
Denote by $\sprt P$ the support of a distribution.
Given any two measures $m_0, m_1$, we denote by $m_0\times m_1$ the product measure over the product space.
We denote by $t$ the training time and by $\tau$ the time that parametrizes flows.

\subsection{Distribution representation}
\label{sec. distribution representation}

Since machine learning models at the basic level are good at learning functions, the common approach to learning distributions is to parametrize distributions by functions.
There are three common approaches:

\textbf{1. Potential function.}
Given any base distribution $\prob$, define the modeled distribution $P$ by
\begin{equation}
\label{eq. potential representation}
P = \frac{1}{Z} e^{-V} \prob, \quad Z=\int e^{-V} d\prob
\end{equation}
where $V$ is a potential function and $Z$ is for normalization.
This parametrization is sometimes known as the Boltzmann distribution or exponential family.

\vs
\textbf{2. Free generator.}
Given any measureable function $G:\R^d\to\R^d$, define the modeled distribution $P$ by
\begin{equation*}
P=\law(X),\quad X=G(Z), \quad Z\sim\prob
\end{equation*}
$P$ is known as the transported measure or pushforward measure and denoted by $P=G\#\prob$, while $G$ is called the generator or transport map.
Equivalently, $P$ is defined as the measure that satisfies
\begin{equation}
\label{eq. pushforward set}
P(A) = \prob(G^{-1}(A))
\end{equation}
for all measurable sets $A$.

The name ``free generator" is used to emphasize that the task only specifies the target distribution $P_*$ to estimate, and we are free to choose any generator from the possibly infinite set of solutions $\{G~|~P_*=G\#\prob\}$.

There are several common extensions to the generator.
First, $G$ can be modeled as a random function, such that $G(\z)\sim P(\cdot|\z)$ for some conditional distribution $P(\cdot|\z)$.
%As an abuse of notation, we denote the pushforward distribution by $P=P(\cdot|\z)\#\prob$ and the joint distribution by $\pi = P(\cdot|\z)\prob$.
Second, $G$ can be induced by a flow.
Let $V:\R^d\times[0,T]\to\R^d$ be a Lipschitz velocity field, and define $G$ as the unique solution to the ODE
\begin{equation}
\label{eq. flow generator}
G=G_T, \quad G_{\tau}(\x)=\x_{\tau}, \quad \x_0=\x,\quad \frac{d}{d\tau} \x_{\tau} = V(\x_{\tau},\tau)
\end{equation}
where $G_{\tau}$ is the flow map.
Furthermore, if we define the interpolant distributions $P_{\tau}=G_{\tau}\#\prob$, then they form a (weak) solution to the continuity equation
\begin{equation}
\label{eq. continuity equation}
\partial_{\tau} P_{\tau} + \nabla \cdot (V_{\tau} P_{\tau}) = 0
\end{equation}
Specifically, for any smooth test function $\phi$
\begin{align*}
\int \phi(\x) d(\partial_{\tau} P_{\tau})(\x) &= \frac{d}{d\tau} \int \phi(\x) dP_{\tau}(\x) = \frac{d}{d\tau} \int \phi\big(G_{\tau}(\x)\big) d\prob(\x)\\
&= \int V\big(G_{\tau}(\x),\tau\big) \cdot \nabla \phi\big(G_{\tau}(\x)\big) d\prob(\x)\\
&= \int V(\x,\tau) \cdot \nabla \phi(\x) dP_{\tau}(\x) = -\int \phi(\x) d\big(\nabla\cdot (V_{\tau} P_{\tau})\big)(\x)
\end{align*}

Third, one can restrict to a subset of the possibly infinite set of solutions $\{G~|~P_*=G\#\prob\}$, specifically generators that are gradients of some potential functions $\{G=\nabla\psi\}$.
By Brennier's theorem \cite{brenier1991polar,villani2003topics}, such potential function exists in very general conditions.
Similarly, one can restrict the velocity fields in (\ref{eq. flow generator}) to time-dependent gradient fields
\begin{equation}
\label{eq. gradient field velocity}
V_{\tau} = \nabla\phi_{\tau}
\end{equation}
By Theorem 5.51 of \cite{villani2003topics}, in general there exists a potential function $\phi_{\tau}$ such that the flow $G$ induced by $\nabla\phi_{\tau}$ satisfies $P_*=G\#\prob$.
Specifically, the interpolant distribution $P_{\tau}=G_{\tau}\#\prob$ is the Wasserstein geodesic that goes from $\prob$ to $P_*$.
Finally, it is interesting to note that there is also a heuristic argument from \cite{albergo2022building} that justifies the restriction to gradient fields:
Given any velocity field $V_{\tau}$ with the interpolant distribution $P_{\tau}$ generated by (\ref{eq. flow generator}), consider the equation
\begin{equation*}
\nabla \cdot (P_{\tau} \nabla\phi_{\tau}) = \nabla\cdot (P_{\tau}V_{\tau})
\end{equation*}
By the theory of elliptic PDEs, the solution $\phi_{\tau}$ exists.
It follows from (\ref{eq. continuity equation}) that we can always replace a velocity field by a gradient field that induces the same interpolant distribution $P_{\tau}$.

\vs
\textbf{3. Fixed generator.}
Contrary to the free generator, another approach is to choose a specific coupling between the base and target distributions $\pi\in\Pi(\prob,P_*)$, where
\begin{equation*}
\Pi(\prob,P_*) = \Big\{\pi\in\PS(\R^d\times\R^d) ~\Big|~ \int \pi d\x_0 = \prob, ~\int \pi d\x_1=P_*\Big\}
\end{equation*}
and the generator $G$ is represented as the conditional distribution $\pi(\cdot|\x_0)$.

One can further extend $\pi$ into a random path $\{X_{\tau},\tau\in[0,T]\}$ such that $\pi=\law(X_0,X_T)$.
Then, analogous to the construction (\ref{eq. flow generator}), $G$ can be represented as the ODE or SDE that drives the trajectories $X_{\tau}$.
Thanks to the non-uniqueness of $law$, one can further consider the interpolant distributions $P_{\tau}=\law(X_{\tau})$ and solve for the velocity field $V$ in the continuity equation (\ref{eq. continuity equation}).
Then, $G$ can be represented as the solution to the ODE (\ref{eq. flow generator}) with velocity $V$.

Currently, models of this category belong to either of the two extremes:

\textbf{Fully deterministic}: For some measureable function $G$,
\begin{equation}
\label{eq. deterministic coupling}
\pi(\x_0,\x_1)=\delta_{G(\x_0)}(\x_1) \prob(\x_0)
\end{equation}
The generator $G$ is usually set to be the optimal transport map from $\prob$ to $P_*$.
The idea is simple in one dimension, such that we sort the ``particles" of $\prob$ and $P_*$ and match according to this ordering.
This monotonicity in $\R$ can be generalized to the cyclic monotonicity in higher dimensions \cite{villani2003topics}.
Couplings $\pi\in\Pi(\prob,P_*)$ that are cyclic monotonic are exactly the optimal transport plans with respect to the squared Euclidean metric \cite{villani2003topics}, namely minimizers of (\ref{eq. W2 coupling}).
Then, Brennier's theorem \cite{brenier1991polar,villani2003topics} implies that, under general conditions, the problem (\ref{eq. W2 coupling}) has unique solution, which has the form (\ref{eq. deterministic coupling}), and furthermore the generator is a gradient field $G=\nabla\psi$ of a convex function $\psi$.
%If we define the random path $X_{\tau}$ by linear interpolation, then the interpolant distributions $P_{\tau}$ are known as McCann interpolation \cite{mccann1997convexity,villani2003topics}.

\textbf{Fully random}: The coupling is simply the product measure
\begin{equation}
\label{eq. product coupling}
\pi=\prob\times P_*
\end{equation}
At first sight, this choice is trivial and intractable, but the trick is to choose an appropriate random path $X_{\tau}$ such that either the dynamics of $X_{\tau}$ or the continuity equation (\ref{eq. continuity equation}) is easy to solve.

One of the simplest constructions, proposed by \cite{liu2022flow,albergo2022building}, is to use the linear interpolation
\begin{equation}
\label{eq. linear interpolation}
X_{\tau}=(1-\tau)X_0+\tau X_1, \quad (X_0,X_1)\sim\pi, ~\tau\in[0,1]
\end{equation}
Then, to solve for the target velocity field in (\ref{eq. continuity equation}), define a joint distribution $M_*$ over $\R^d\times[0,1]$
\begin{align}
\label{eq. joint interpolant distribution}
\begin{split}
\int \phi(\x,\tau) dM_*(\x,\tau) &= \int_0^1 \int \phi(\x,\tau) dP_{\tau}d\tau\\
&= \int_0^1 \iint \phi\big((1-\tau)\x_0+\tau\x_1, \tau\big) d\prob(\x_0)dP_*(\x_1)d\tau
\end{split}
\end{align}
for any test function $\phi$.
Similarly, define the current density $J_*$, a vector-valued measure, by
\begin{equation}
\label{eq. current density}
\int \mathbf{f}(\x,\tau) \cdot dJ_*(\x,\tau) = \int_0^1 \iint (\x_1-\x_0) \cdot \mathbf{f}\big((1-\tau)\x_0+\tau\x_1, \tau\big) d\prob(\x_0)dP_*(\x_1)d\tau
\end{equation}
for any test function $\mathbf{f}$.
Then, we can define a velocity field $V_*$ by the Radon-Nikodym derivative
\begin{equation}
\label{eq. Radon-Nikodym velocity}
V_* = \frac{d J_*}{d M_*}
\end{equation}
Each $V_*(\x,\tau)$ is the weighted average of the velocities of the random lines (\ref{eq. linear interpolation}) that pass through the point $(\x,\tau)$ in spacetime.
As shown in \cite{albergo2022building,yang2022flow}, under general assumptions, $V_*$ is the solution to the continuity equation (\ref{eq. continuity equation}) and satisfies
\begin{equation*}
G_*\#\prob = P_*
\end{equation*}
where $G_*$ is the generator defined by the flow (\ref{eq. flow generator}) of $V_*$.

A more popular construction by \cite{sohl2015deep,song2019generative,ho2020denoising} uses the diffusion process
\begin{equation}
\label{eq. diffusion SDE}
X_0\sim P_*,\quad d X_{\tau} = -\frac{\beta_{\tau}}{2} X_{\tau} d\tau + \sqrt{\beta_{\tau}} dW_{\tau}
\end{equation}
where $\beta_{\tau}>0$ is a non-decreasing function that represents the noise scale.
Consider the coupling $\pi_{\tau}=\law(X_{\tau},X_0)$.
The conditional distribution $\pi_{\tau}(\cdot|\x_0)$ is an isotropic Gaussian \cite{song2020score}
\begin{equation}
\label{eq. diffusion conditional initial}
\pi_{\tau}(\cdot|\x_0) = \N\Big(e^{-\frac{1}{2}\int_0^{\tau}\beta_s ds}\x_0, ~(1-e^{-\int_0^{\tau}\beta_s ds})I\Big)
\end{equation}
and the interpolant distributions $P_{\tau}=\law(X_{\tau})$ are given by
\begin{equation}
\label{eq. diffusion interpolant distribution}
P_{\tau} = \int \pi_{\tau}(\cdot|\x_0) dP_*(\x_0)
\end{equation}
Then,
\begin{align*}
\KL(\pi\|\pi_{\tau}) &= \int \ln \frac{d\pi}{d\pi_{\tau}} d\pi = \iint \ln \frac{\N(\x)}{\pi_{\tau}(\x|\x_0)} d\N(\x) dP_*(\x_0)\\
&= \int \KL\big(\N~\big\|~\N(e^{-\frac{1}{2}\int_0^{\tau}\beta_s ds}\x_0, (1-e^{-\int_0^{\tau}\beta_s ds})I)\big) dP_*(\x_0)\\
&= \frac{e^{-\int_0^{\tau}\beta_s ds}}{1-e^{-\int_0^{\tau}\beta_s ds}} \Big(\int \|\x_0\|^2 dP_*(\x_0) + d\Big) + d \ln(1-e^{-\int_0^{\tau}\beta_s ds})
\end{align*}
where the last line follows from the formula for the KL divergence between multivariable Gaussians.
Let $\prob$ be the unit Gaussian $\N$.
It follows that if $P_*$ has finite second moments, then the coupling $\pi_{\tau}$ converges to the product measure $\pi=\prob\times P_*$ exponentially fast.

By choosing $T$ sufficiently large, we have $P_T\approx\prob$.
Then, a generative model can be defined by sampling from $X_T\sim\prob$ and then going through a reverse-time process $X_0=G(X_T)$ to approximate the target $P_*$.
One approach is to implement the following reverse-time SDE \cite{anderson1982reverse}
\begin{equation*}
X_T\sim \prob, \quad d X_{\tau} = -\frac{\beta_{\tau}}{2} ( X_{\tau} + 2\nabla_{\x}\log P_{\tau}(X_{\tau})) d\tau + \sqrt{\beta_{\tau}} d\overline{W}_{\tau}
\end{equation*}
which is solved from time $T$ to $0$, and $\overline{W}$ is the reverse-time Wiener process.
This backward SDE is equivalent to the forward SDE (\ref{eq. diffusion SDE}) in the sense that, if $P_T=\prob$, then they induce the same distribution of paths $\{X_{\tau}, \tau\in[0,T]\}$ \cite{anderson1982reverse}, and in particular $P_*=\law(X_0)$.
(An analysis that accounts for the approximation error between $P_T$ and $\prob$ is given by \cite{song2021maximum}.)
The gradient field $\nabla_{\x}\log P_{\tau}$ is known as the score function \cite{hyvarinen2005estimation}, which is modeled by a velocity field $\mathbf{s}:\R^d\times[0,T]\to\R^d$, and then the generator $G$ can be defined as the following random function
\begin{equation}
\label{eq. diffusion generator SDE}
G(\x_T)=X_0, \quad X_T=\x_T, \quad d X_{\tau} = -\frac{\beta_{\tau}}{2} ( X_{\tau} + 2 \mathbf{s}(X_{\tau}, \tau)) d\tau + \sqrt{\beta_{\tau}} d\overline{W}_{\tau}
\end{equation}

Another approach is to implement the following reverse-time ODE \cite{song2020score}
\begin{equation*}
X_T \sim \prob, \quad \frac{d}{d\tau}X_{\tau} = V(X_{\tau},\tau), \quad V(\x,\tau) = -\frac{\beta_{\tau}}{2} ( \x_{\tau} + \nabla_{\x}\log P_{\tau}(\x_{\tau}))
\end{equation*}
This $V$ is the solution to the continuity equation (\ref{eq. continuity equation}), and we similarly have $P_*=\law(X_0)$ if $P_T=\prob$.
Then, the generator $G$ can be defined as a deterministic function
\begin{equation}
\label{eq. diffusion generator ODE}
G(\x_T)=\x_0, \quad \frac{d}{d\tau} \x_{\tau} = -\frac{\beta_{\tau}}{2} ( \x_{\tau} + \mathbf{s}(\x_{\tau},\tau))
\end{equation}

\vs
\textbf{4. Mixture.}
Finally, we remark that it is possible to use a combination of these representations.
For instance, \cite{noe2019boltzmann} uses a normalizing flow model reweighed by a Boltzmann distribution, which is helpful for sampling from distributions with multiple modes while maintaining accurate density estimation.
Another possibility is that one can first train a model with fixed generator representation as a stable initialization, and then finetune the trained generator as a free generator (e.g. using GAN loss) to improve sample quality.

\subsection{Loss type}
\label{sec. loss}
There are numerous ways to define a metric or divergence on the space of probability measures, which greatly contribute to the diversity of distribution learning models.
One requirement, however, is that since the target distribution $P_*$ is replaced by its samples $P_*^{(n)}$ during training, the term $P_*$ almost always appears in the loss as an expectation.

The commonly used losses belong to three categories:

\vspace{0.5em}
\textbf{1. Density-based loss.}
The modeled distribution $P$ participates in the loss as a density function.
The default choice is the KL divergence (\ref{eq. KL}), which is equivalent up to constant to the negative log-likelihood (NLL)
\begin{equation}
\label{eq. NLL}
L(P) = -\int \log P(\x) dP_*(\x)
\end{equation}
In fact, we can show that NLL is in a sense the only possible density-based loss.

\begin{proposition}
\label{prop. density loss NLL}
Let $L$ be any loss function on $\PS_{ac}(\R^d)$ that has the form
\begin{equation*}
L(P) = \int f\big(P(\x)\big) dP_*(\x)
\end{equation*}
where $f$ is some $C^1$ function on $(0,+\infty)$.
If for any $P_*\in\PS_{ac}(\R^d)$, the loss $L$ is minimized by $P_*$, then
\begin{equation*}
f(p) = c\log p + c'
\end{equation*}
for $c \leq 0$ and $c' \in \R$. The converse is obvious.
\end{proposition}

Besides the KL divergence, there are several other well-known divergences in statistics such as the Jensen–Shannon divergence and $\chi^2$ divergence.
Despite that they are infeasible by Proposition \ref{prop. density loss NLL}, certain weakened versions of these divergences can still be used as will be discussed later.

\vspace{0.5em}
\textbf{2. Expectation-based loss.}
The modeled distribution participate in the loss through expectations.
Since both $P_*$ and $P$ are seen as linear operators over test functions, it is natural to define the loss as a dual norm
\begin{equation}
\label{eq. dual norm}
L(P) = \sup_{\|D\|\leq 1} \int D(\x) d(P-P_*)(\x)
\end{equation}
where $\|\cdot\|$ is some user-specified functional norm.
The test function $D$ is often called the discriminator, and such loss is called an adversarial loss.
If $\|\cdot\|$ is a Hilbert space norm, then the loss can also be defined by
\begin{equation}
\label{eq. dual norm squared}
L(P) = \sup_D \int D(\x) d(P-P_*)(\x) - \|D\|^2
\end{equation}
There are several classical examples of adversarial losses:
If $\|\cdot\|$ is the $C_0$ norm, then (\ref{eq. dual norm}) becomes the total variation norm $\|P-P_*\|_{\TV}$.
If $\|\cdot\|$ is the Lipschitz semi-norm, then (\ref{eq. dual norm}) becomes the 1-Wasserstein metric by Kantorovich-Rubinstein theorem \cite{kantorovich1958W1,villani2003topics}.
If $\|\cdot\|$ is the RKHS norm with some kernel $k$, then (\ref{eq. dual norm}) becomes the maximum mean discrepancy (MMD) \cite{gretton2012kernel}, and (\ref{eq. dual norm squared}) is the squared MMD:
\begin{equation*}
L(P) = \frac{1}{2} \iint k(\x,\x') ~d(P_*-P)(\x) d(P_*-P)(\x')
\end{equation*}
which gives rise to the moment matching network \cite{li2015GMMM}.

In practice, the discriminator $D$ is usually parametrized by a neural network, denoted by $D_{\theta}$ with parameter $\theta$.
One common choice of the norm $\|\cdot\|$ is simply the $l^{\infty}$ norm on $\theta$
\begin{equation}
\label{eq. WGAN loss P}
L(P) = \sup_{\|\theta\|_{\infty} \leq 1} \int D_{\theta} ~d(P-P_*)
\end{equation}
This formulation gives rise to the WGAN model \cite{arjovsky2017wasserstein}, and the $l^{\infty}$ bound can be conveniently implemented by weight clipping.
The loss (\ref{eq. WGAN loss P}) and its variants are generally known as the neural network distances \cite{arora2017generalization,zhang2017discrimination,gulrajani2020towards,e2020kolmogorov}.

The strength of the metric (e.g. fineness of its topology) is proportional to the size of the normed space of $\|\cdot\|$, or inversely proportional to the strength of $\|\cdot\|$.
Once some global regularity such as the Lipschitz norm applies to the space, then the dual norm or $L$ becomes continuous with respect to the underlying geometry (e.g. the $W_1$ metric), and is no longer permutation invariant like NLL (\ref{eq. NLL}) or total variation.
If we further restrict $D$ to certain sparse subspaces of the Lipschitz functions, in particular neural networks, then $L$ becomes insensitive to the ``high frequency" parts of $\PS(\R^d)$.
As we will demonstrate in Section \ref{sec. generalization}, this property is the source of good generalization.

Note that there are some variants of the GAN loss that resemble (\ref{eq. dual norm squared}) but whose norms $\|D\|$ are so weak that the dual norms are no longer well-defined. For instance, the loss with $L^2$ penalty from \cite{xu2020understanding}
\begin{equation*}
L(P) = \sup_{\theta} \int D_{\theta} ~d(P-P_*) - \|D_{\theta}\|_{L^2(P+P_*)}^2
\end{equation*}
or the loss with Lipschitz penalty $\|1-\|\nabla D_{\theta}\| \|_{L^2(P)}^2$ from \cite{gulrajani2017improved}.
By \cite[Proposition 5]{yang2022GAN}, in general we have $L(P)=\infty$.
Nevertheless, if we consider one-time-scale training such that $P$ and $D$ are trained with similar learning rates, then this blow-up can be avoided \cite{yang2022GAN}.

Beyond the dual norms (\ref{eq. dual norm}), one can also consider divergences.
Despite that Proposition \ref{prop. density loss NLL} has ruled out the use of divergences other than the KL divergence,
one can consider the weakened versions of the dual of these divergences.
For instance, given any parametrized discriminator $D_{\theta}$, the Jensen–Shannon divergence can be bounded below by \cite{goodfellow2014generative}
\begin{align}
\text{JS}(P,P_*) &= \frac{1}{2}\KL\Big(P\Big\|\frac{P+P_*}{2}\Big) + \frac{1}{2}\KL\Big(P_*\Big\|\frac{P+P_*}{2}\Big) \nonumber \\
&= \sup_{q:\R^d\to[0,1]} \int \log q(\x) dP(\x) + \int \log (1-q(\x)) dP_*(\x) + 2\ln 2 \nonumber \\
&\geq \sup_{\theta} \int \log \Big( \frac{e^{D_{\theta}(\x)}}{1+e^{D_{\theta}(\x)}} \Big) dP(\x) + \int \log \Big( \frac{1}{1+e^{D_{\theta}(\x)}} \Big) dP_*(\x) + 2\ln 2 \label{eq. classical GAN loss P}
\end{align}
%If the parametrization $D_{\theta}$ is sufficiently regular (e.g. neural networks with bounded parameters), then the lower bound (\ref{eq. classical GAN loss P}) is well-defined for general distributions.
This lower bound gives rise to the earliest version of GAN \cite{goodfellow2014generative}.
GANs based on other divergences have been studied in \cite{nowozin2016f,mao2018effectiveness}.

\vspace{0.5em}
\textbf{3. Regression loss.}
The regression loss is used exclusively by the fixed generator representation discussed in Section \ref{sec. distribution representation}.
If a target generator $G_*$ has been specified, then we simply use the $L^2$ loss over the base distribution $\prob$
\begin{equation}
\label{eq. generator regression}
L(G) = \frac{1}{2} \|G-G_*\|_{L^2(\prob)}^2
\end{equation}
If a target velocity field $V_*$ has been specified, then the loss is integrated over the interpolant distributions $P_{\tau}$
\begin{equation}
\label{eq. velocity regression}
L(V) = \frac{1}{2} \int_0^T \big\| V(\cdot,\tau)-V_*(\cdot,\tau) \big\|_{L^2(P_{\tau})}^2 d\tau
\end{equation}
or equivalently, we use the $L^2(M_*)$ loss with the joint distribution $M_*$ defined by (\ref{eq. joint interpolant distribution}).

\subsection{Combination}
\label{sec. combination}

Having discussed the distribution representations and loss types, we can now combine them to derive the distribution learning models in Table \ref{table. categorization}.
Our focus will be on the highlighted four classes in the table.

\vs
\textbf{Density + Potential.}
Since Proposition \ref{prop. density loss NLL} indicates that the negative log-likelihood (NLL) (\ref{eq. NLL}) is the only feasible density-based loss,
we simply insert the potential-based representation (\ref{eq. potential representation}) into NLL, and obtain a loss in the potential function $V$,
\begin{equation}
\label{eq. bias-potential model}
L(V) = \int V dP_* + \ln \int e^{-V} d\prob
\end{equation}
This formulation gives rise to the bias-potential model \cite{valsson2014potential,bonati2019enhanced}, also known as variationally enhanced sampling.

\vs
\textbf{Density + Free generator.}
In order to insert the transport representation $P=G\#\prob$ into NLL (\ref{eq. NLL}), we need to be able to compute the density $P(\x)$.
For simple cases such as when $\prob$ is Gaussian and $G$ is affine, the density $P(\x)$ has closed form expression.
Yet, in realistic scenarios when $P$ needs to satisfy the universal approximation property and thus has complicated forms, one has to rely on indirect calculations.
There are three common approaches:
\begin{enumerate}
\item Change of variables (for normalizing flows): If $G$ is a $C^1$ diffeomorphism, the density of $P$ is given by the change of variables formula
\begin{equation*}
P(\x) = \det\nabla G^{-1}(\x) ~\prob(G^{-1}(\x))
\end{equation*}
Usually $\prob$ is set to the unit Gaussian $\N$.
Then, the NLL loss (\ref{eq. NLL}) becomes
\begin{align*}
L(G) &= -\int \log \det\nabla G^{-1}(\x) + \log \prob(G^{-1}(\x)) dP_*(\x)\\
&= \int \log \det\nabla G\big(G^{-1}(\x)\big) + \frac{1}{2} \|G^{-1}(\x)\|^2 dP_*(\x) + \text{constant}
\end{align*}
If $G$ is modeled by a flow $\{G_{\tau}, \tau\in[0,1]\}$ (\ref{eq. flow generator}) with velocity field $V$, then its Jacobian satisfies
\begin{align*}
\frac{d}{d\tau}\det\nabla G_{\tau}(\x_0) &= \frac{d}{d\tau} \det \nabla \Big(\x_0+\int_0^{\tau} V\big(G_s(\x_0), s\big) ds \Big)\\
&= \frac{d}{d\tau} \det \Big(I+\int_0^{\tau} \nabla V\big(G_s(\x_0), s\big) \nabla G_s(\x_0) ds \Big)\\
&= \text{Tr}\Big[ \big(\nabla V(G_s(\x_0), s) \nabla G_s(\x_0)\big) \big(\nabla G_s(\x_0)\big)^{-1}\Big] \det\nabla G_{\tau}(\x_0)\\
&= \text{Tr}\big[\nabla V(G_s(\x_0), s)\big] \det\nabla G_{\tau}(\x_0)
\end{align*}
It follows that
\begin{equation*}
\log \det\nabla G(\x_0) = \int_0^1 \text{Tr}\big[\nabla V(G_{\tau}(\x_0), \tau)\big] d\tau
\end{equation*}
and this is known as Abel's formula \cite{teschl2012ODE}.
Hence, we obtain the loss of the normalizing flow model \cite{tabak2010density,chen2018neural}
\begin{align}
\label{eq. NF NLL}
\begin{split}
L(V) &= \iint_0^1 \text{Tr}\big[\nabla V(\x_{\tau}, \tau)\big] d\tau + \frac{1}{2}\|\x_0\|^2 dP_*(\x_1)\\
\x_{\tau} &:= G_{\tau}(G^{-1}(\x_1))
\end{split}
\end{align}
Moreover, if the velocity field is defined by a gradient field $V(\cdot,{\tau})=\nabla\phi_{\tau}$ as discussed in (\ref{eq. gradient field velocity}), then the loss has the simpler form
\begin{equation*}
L(\phi) = \iint_0^1 \Delta\phi_{\tau}(\x_{\tau}) d\tau + \frac{1}{2}\|\x_0\|^2 dP_*(\x_1)
\end{equation*}
which leads to the Monge-Amp\`{e}re flow model \cite{zhang2018monge}.

One potential shortcoming of NF is that diffeomorphisms might not be suitable for the generator when the target distribution $P_*$ is singular, e.g. concentrated on low-dimensional manifolds, which is expected for real data such as images.
To approximate $P_*$, the generator $G$ needs to shrink the mass of $\prob$ onto negligible sets, and thus $G^{-1}$ blows up.
As $G^{-1}$ is involved in the loss, it can cause the training process to be unstable.

\item Variational lower bound (for VAE):
Unlike NF, the variational autoencoders do not require the generator to be invertible, and instead use its posterior distribution.
The generator can be generalized to allow for random output, and we define the conditional distribution
\begin{equation*}
P(\cdot|\z) = \law(X), \quad X=G(\z)
\end{equation*}
The generalized inverse can be defined as the conditional distribution $Q_*(\cdot|\x)$ that satisfies
\begin{equation*}
P(\x|\z)\prob(\z) = P(\x)Q_*(\z|\x), \quad P(\x)=\int P(\x|\z) d\prob(\z)
\end{equation*}
in the distribution sense.
If the generator is deterministic, i.e. $P(\cdot|\z)=\delta_{G(\z)}$, and invertible, then $Q_*(\cdot|\x)$ is simply $\delta_{G^{-1}(\x)}$.
It follows that the KL divergence (\ref{eq. KL}) can be written as
\begin{align*}
\KL(P_*\|P) &= \int \log \frac{P_*(\x)}{P(\x)} + \KL\big(Q_*(\cdot|\x)\big\|Q_*(\cdot|\x)\big) dP_*(\x)\\
&= \min_{Q(\cdot|\x)} \iint \log \frac{P_*(\x)}{P(\x)} + \KL\big(Q(\cdot|\x)\big\|Q_*(\cdot|\x)\big) dP_*(\x)\\
&= \min_{Q(\cdot|\x)} \iint \log \frac{P_*(\x)Q(\z|\x)}{P(\x)Q_*(\z|\x)} dQ(\z|\x)dP_*(\x)\\
&= \min_{Q(\cdot|\x)} \iint \log \frac{P_*(\x)Q(\z|\x)}{P(\x|\z)\prob(\z)} dQ(\z|\x)dP_*(\x)\\
&= \min_{Q(\cdot|\x)} \KL\big( P_*Q(\cdot|\x) ~\big\|~ P(\cdot|\z)\prob \big)
\end{align*}
This is an example of the variational lower bound \cite{kingma2013auto}, and the NLL loss (\ref{eq. NLL}) now becomes
\begin{equation*}
\min_{P(\cdot|\z)} \min_{Q(\cdot|\x)} \iint -\log P(\x|\z) dQ(\z|\x) + \KL\big(Q(\z|\x)\big\|\prob(\z)\big) dP_*(\x)
\end{equation*}
which is the loss of VAE.
To make the problem more solvable, the decoder $P(\cdot|\z)$ and encoder $Q(\cdot|\x)$ are usually parametrized by diagonal Gaussian distributions \cite{kingma2013auto}:
\begin{equation*}
P(\cdot|\z) = \N\big(G(\z), \diag(e^{\mathbf{s}(\z)}) \big), \quad Q(\cdot|\x) = \N\big(F(\x), \diag(e^{\mathbf{v}(\x)}) \big)
\end{equation*}
where $G,F,\mathbf{s},\mathbf{v}$ are parametrized functions $\R^d\to\R^d$ such as neural networks, and $\exp$ is taken entry-wise.
Using the formula for KL divergence between Gaussians
\begin{equation*}
\KL\big(\N(m_0,\Sigma_0)\big\|\N(m_1,\Sigma_1)\big) = \frac{1}{2} \Big[ \log\frac{\det \Sigma_1}{\det \Sigma_0} - d + \tr[\Sigma_1^{-1}\Sigma_0] + (m_1-m_0)^T\Sigma_1^{-1}(m_1-m_0) \Big]
\end{equation*}
we can show that, up to constant, the VAE loss equals
\begin{align*}
\min_{G,s} \min_{F,v} & ~\frac{1}{2}\iint \frac{\big\|\x-G\big(F(\x)+e^{\mathbf{v}(\x)}\odot\omega\big)\big\|^2}{e^{\mathbf{s}(F(\x)+e^{\mathbf{v}(\x)}\odot\omega)}} + \sum_{i=1}^d \mathbf{s}_i\big(F(\x)+e^{\mathbf{v}(\x)}\odot\omega\big) d\N(\omega)\\
&\quad + \|F(\x)\|^2 + \sum_{i=1}^d e^{\mathbf{v}_i(\x)} - \mathbf{v}_i(\x) ~dP_*(\x)
\end{align*}
where $\odot$ is entry-wise product.
This loss resembles the classical autoencoder \cite{ackley1985learning,schmidhuber2015deep}
\begin{equation*}
\min_{F,G} \int \frac{\|\x-G(F(\x))\|^2}{2} dP_*(\x)
\end{equation*}
and thus $P(\cdot|\z), Q(\cdot|\x)$ are addressed by the decoder and encoder.

\item Factorization (for autoregressive model): To model a distribution $P$ over sequential data $\mathbf{X} = [\x_1, \dots \x_l]$, one can choose a generator $G$ that is capable of processing variable-length inputs $[\x_1, \dots \x_i]$, such as the Transformer network \cite{vaswani2017attention} or recurrent networks \cite{salehinejad2017recent}, and define the distribution by
\begin{align*}
P(\mathbf{X}) &= \prod_{i=1}^l P(\x_i|\x_1, \dots \x_{i-1})\\
P(\cdot|\x_1, \dots \x_{i-1}) &= \law(X_i), \quad X_i \sim G(Z|\x_1, \dots \x_{i-1}), \quad Z \sim \prob
\end{align*}
Then, NLL (\ref{eq. NLL}) is reduced to
\begin{equation*}
-\int \log P(\mathbf{X}) dP_*(\mathbf{X}) = -\int \sum_{i=1}^l \log P(\x_i|\x_1,\dots\x_{i-1}) dP_*(\x)
\end{equation*}
Usually, each $P(\cdot|\x_1, \dots \x_{i-1})$ has a simple parametrization such as Gaussian or Softmax so that $\log P(\x_i|\x_1,\dots\x_{i-1})$ is tractable \cite{oord2016wavenet,radford2018improving}.
%while the expressivity of $P$ is maintained by the factorized representation.
\end{enumerate}

\vs
\textbf{Expectation + Free generator.}
By the definition (\ref{eq. pushforward set}) of the transport representation $P=G\#\prob$,
\begin{equation*}
\int f(\x) dP(\x) = \int f(G(\z)) d\prob(\z)
\end{equation*}
for all measureable functions $f$.
Then, the classical GAN loss (\ref{eq. classical GAN loss P}) becomes
\begin{equation}
\label{eq. classical GAN loss}
\min_G \max_D \int \log \Big( \frac{e^{D(G(\z))}}{1+e^{D(G(\z))}} \Big) d\prob(\z) + \int \log \Big( \frac{1}{1+e^{D(\x)}} \Big) dP_*(\x)
\end{equation}
Similarly, the WGAN loss (\ref{eq. WGAN loss P}) becomes
\begin{equation}
\label{eq. WGAN loss}
\min_G \max_{\|\theta\|_{\infty} \leq 1} \int D_{\theta}(G(\z)) d\prob(\z) - \int D_{\theta}(\x) dP_*(\x)
\end{equation}

\vs
\textbf{Regression + Fixed generator.}
For the case with fully deterministic coupling (\ref{eq. deterministic coupling}), a target generator $G_*$ is provided by numerical optimal transport, and then fitted by a parametrized function $G$ with the regression loss (\ref{eq. generator regression}).
This formulation leads to the generative model \cite{zhang2019optimal}.
(Moreover, a few models with some technical variations \cite{an2019ae,an2020GAN,rout2021generative} are related to this category, but for simplicity we do not describe them here.)

For the case with fully random coupling (\ref{eq. product coupling}), we fit either the score function $\nabla\log P_{\tau}$ from (\ref{eq. diffusion interpolant distribution}) or the velocity field $V_*$ from (\ref{eq. Radon-Nikodym velocity}) using the regression loss (\ref{eq. velocity regression}).
Note that the targets (\ref{eq. diffusion interpolant distribution}, \ref{eq. Radon-Nikodym velocity}) are both defined by expectations and thus the loss (\ref{eq. velocity regression}) cannot be computed directly.
Thanks to the linearity of expectation, we can expand the loss to make the computation tractable.

Model the score function $\nabla\log P_{\tau}$ by a velocity field $\mathbf{s}:\R^d\times[0,T]\to\R^d$ and let $\lambda_{\tau}>0$ be a user-specified weight.
The regression loss can be written as
\begin{align*}
L(\mathbf{s}) &:= \frac{1}{2} \int_0^T \lambda_{\tau} \|\mathbf{s}(\cdot,\tau)-\nabla\log P_{\tau}\|_{L^2(P_{\tau})}^2 d\tau \\
&= \int_0^T \lambda_{\tau} \int \frac{1}{2} \|\mathbf{s}(\x,\tau)\|^2 - \mathbf{s}(\x,\tau) \cdot \nabla\log P_{\tau}(\x) ~d P_{\tau}(\x) d\tau + C \\
&= \int_0^T \lambda_{\tau} \int \frac{1}{2} \|\mathbf{s}(\x,\tau)\|^2 dP_{\tau}(\x) - \lambda_{\tau} \int \mathbf{s}(\x,\tau) \cdot \nabla P_{\tau}(\x) d\x ~d\tau + C \\
&= \int_0^T \lambda_{\tau} \iint \frac{1}{2} \|\mathbf{s}(\x,\tau)\|^2 d\pi_{\tau}(\x|\x_0)dP_*(\x_0) \\
&\quad - \lambda_{\tau} \int \mathbf{s}(\x,\tau) \cdot \nabla \int \pi_{\tau}(\x|\x_0)dP_*(\x_0) d\x ~d\tau + C \\
&= \int_0^T \lambda_{\tau} \iint \frac{1}{2} \|\mathbf{s}(\x,\tau)\|^2 - \mathbf{s}(\x,\tau) \cdot \nabla \log \pi_{\tau}(\x|\x_0) ~d\pi_{\tau}(\x|\x_0) dP_*(\x_0) d\tau + C \\
&= \int_0^T \frac{\lambda_{\tau}}{2} \iint \|\mathbf{s}(\x,\tau) - \nabla \log \pi_{\tau}(\x|\x_0)\|^2 d\pi_{\tau}(\x|\x_0) dP_*(\x_0) d\tau + C
\end{align*}
Since the conditional distribution $\pi_{\tau}(\x|\x_0)$ is the isotropic Gaussian (\ref{eq. diffusion conditional initial}), this loss is straightforward to evaluate.
Thus, we obtain the loss of the score-based diffusion models \cite{sohl2015deep,song2019generative,ho2020denoising,song2020score}
\begin{equation}
\label{eq. diffusion Gaussian loss}
L(\mathbf{s}) = \int_0^T \frac{\lambda_{\tau}}{2} \iint \Big\|\mathbf{s}\big(e^{-\frac{1}{2}\int_0^{\tau}\beta_s ds}\x_0 + \sqrt{1-e^{-\int_0^{\tau}\beta_s ds}} \omega,\tau\big) + \frac{\omega}{\sqrt{1-e^{-\int_0^{\tau}\beta_s ds}}} \Big\|^2 d\N(\omega) dP_*(\x_0) d\tau
\end{equation}

Similarly, for the velocity field $V_*$ (\ref{eq. Radon-Nikodym velocity}), using the definitions (\ref{eq. joint interpolant distribution}, \ref{eq. current density}) of  the joint distribution $M_*$ and current density $J_*$, we can write the regression loss as
\begin{align}
L(V) &:= \frac{1}{2} \int_0^1 \|V(\cdot,\tau)-V_*(\cdot,\tau)\|_{L^2(P_{\tau})}^2 d\tau \nonumber \\
&= \int \frac{1}{2} \|V(\x,\tau)\|^2 - V(\x,\tau) \cdot V_*(\x,\tau) ~dM_*(\x,\tau) + C \nonumber \\
&= \int \frac{1}{2} \|V(\x,\tau)\|^2 dM_*(\x,\tau) - \int V \cdot dJ_* + C \nonumber \\
&= \int_0^1\iint \frac{1}{2} \big\|V\big((1-\tau)\x_0+\tau\x_1, \tau\big)\big\|^2 d\prob(\x_0)dP_*(\x_1)d\tau \nonumber \\
&\quad - \int_0^1\iint (\x_1-\x_0)\cdot V\big((1-\tau)\x_0+\tau\x_1, \tau\big) d\prob(\x_0)dP_*(\x_1)d\tau + C \nonumber \\
&= \frac{1}{2} \int_0^1\iint \big\| V\big((1-\tau)\x_0+\tau\x_1, \tau\big) - (\x_1-\x_0) \big\|^2 d\prob(\x_0)dP_*(\x_1)d\tau + C \label{eq. diffusion NF loss}
\end{align}
Thus, we obtain the loss of normalizing flow with stochastic interpolants \cite{albergo2022building,liu2022flow,yang2022flow}.

\vs
\textbf{Other classes.}
Finally, we briefly remark on the rest of the classes in Table \ref{table. categorization}.
For the combination ``Density + Fixed generator", it has been shown by \cite{song2021maximum} that the regression loss upper bounds the KL divergence.
Specifically, if in the loss $L$ (\ref{eq. diffusion Gaussian loss}) we set the weight by $\lambda_{\tau}=\beta_{\tau}$ where $\beta_{\tau}$ is the noise scale in the SDE (\ref{eq. diffusion SDE}), then given any score function $\mathbf{s}$,
\begin{equation}
\label{eq. diffusion KL bound}
\KL(P_*\|G_{\mathbf{s}}\#\prob) \leq L(\mathbf{s}) + \KL(P_T\|\prob)
\end{equation}
where $G_{\mathbf{s}}$ is the generator defined by the reverse-time SDE (\ref{eq. diffusion generator SDE}) with the score $\mathbf{s}$.
The result also holds for $G_{\mathbf{s}}$ defined by the reverse-time ODE (\ref{eq. diffusion generator ODE}) under a self-consistency assumption: let $(G_{\mathbf{s}})_1^{\tau}$ be the reverse-time flow map of (\ref{eq. diffusion generator ODE}), then
\begin{equation*}
\mathbf{s}(\x,\tau)=\nabla\log ((G_{\mathbf{s}})_1^{\tau}\#\prob)(\x)
\end{equation*}

The combinations ``Expectation + Potential" and ``Expectation + Fixed generator" are feasible, but we are not aware of representative models.
The combinations ``Regression + Potential" and ``Regression + Free generator" do not seem probable, since there is no clear target to perform regression.

\begin{remark}[Empirical loss]
As discussed in Section \ref{sec. loss}, the loss is almost always an expectation in the target distribution $P_*$.
Indeed, one can check that all the loss functions introduced in this section can be written in the abstract form
\begin{equation*}
L(f) = \int F(f,\x) dP_*(\x)
\end{equation*}
where $f$ is the parameter function and $F$ depends on the model.
Thus, if only a finite sample set of $P_*$ is available, as is usually the case in practice, one can define the empirical loss
\begin{equation}
\label{eq. empirical loss}
L^{(n)}(f) = \int F(f,\x) dP^{(n)}_*(\x) = \frac{1}{n} \sum_{i=1}^n F(f,\x_i)
\end{equation}
where $P_*^{(n)}$ is the empirical distribution.
\end{remark}

\subsection{Function representation}
\label{sec. function representation}

Having parametrized the distributions and losses by abstract functions, the next step is to parametrize these functions by machine learning models such as neural networks.
There is much freedom in this choice, such that any parametrization used in supervised learning should be applicable to most of the functions we have discussed.
These include the generators and discriminators of GANs, the means and variances of the decoder and encoder of VAE, the potential function of the bias potential model, the score function of score-based diffusion models, and the velocity field of normalizing flows with stochastic interpolants.
Some interesting applications are given by \cite{denton2015LAPGAN,radford2018improving,jiang2021transgan,rombach2022high}.

One exception is the generator $G$ of the normalizing flows (\ref{eq. NF NLL}), which needs to be invertible with tractable Jacobian.
As mentioned in Section \ref{sec. model overview}, one approach is to parametrize $G$ as a sequence of invertible blocks whose Jacobians have closed-form formula (Example designs can be found in \cite{dinh2016density,kingma2016improved, papamakarios2017masked, huang2018neural}).
Another approach is to represent $G$ as a flow, approximate this flow with numerical schemes, and solve for the traces $\text{Tr}[\nabla V]$ in (\ref{eq. NF NLL}) (Examples of numerical schemes are given by \cite{chen2018neural,zhang2018monge,grathwohl2018ffjord}).

\vs
For the theoretical analysis in the rest of this paper, we need to fix a function representation.
Since our focus is on the phenomena that are unique to learning distributions (e.g. memorization), we keep the function representation as simple as possible, while satisfying the minimum requirement of the universal approximation property among distributions (and thus the capacity for memorization).
Specifically, we use the random feature functions \cite{rahimi2008uniform,e2019machine,yang2022potential}.

\begin{definition}[Random feature functions]
\label{def. RFM}
Let $\HS(\R^d,\R^k)$ be the space of functions that can be expressed as
\begin{equation}
\label{eq. RFM}
f_{\a}(\x) = \E_{\rho(\w,b)} \big[ \a(\w,b) ~\sigma(\w\cdot\x + b) \big]
\end{equation}
where $\rho \in \PS(\R^{d+1})$ is a fixed parameter distribution and $\a \in L^2(\rho,\R^k)$ is a parameter function.
For simplicity, we use the notation $\HS$ when the input and output dimensions $d,k$ are clear.
\end{definition}

\begin{definition}[RKHS norm]
\label{def. RKHS}
For any subset $\Omega\subseteq \R^d$, consider the quotient space
\begin{equation*}
\HS(\Omega) = \HS / \{f_{\a} \equiv \mathbf{0} \text{ on } \Omega\}
\end{equation*}
with the norm
\begin{equation*}
\|f\|_{\HS(\Omega)} = \inf \{\|\a\|_{L^2(\rho)} ~|~ f = f_{\a} \text{ on } \Omega\}
\end{equation*}
We use the notation $\|\cdot\|_{\HS}$ if $\Omega$ is clear from context.
By \cite{cucker2002mathematical,rahimi2008uniform}, $\HS(\Omega)$ is a Hilbert space and $\|\cdot\|_{\HS(\Omega)}$ is equal to the RKHS norm (reproducing kernel Hilbert space) induced by the kernel
\begin{equation*}
k(\x,\x') = \E_{\rho} \big[ \sigma(\w\cdot\x + b) \sigma(\w\cdot\x' + b) \big]
\end{equation*}
Furthermore, given any distribution $\prob\in\PS(\Omega)$, we can define the following integral operator $K: L^2(\prob) \to L^2(\prob)$,
\begin{equation}
\label{eq. integral operator}
K(f)(\x) = \int k(\x,\x') f(\x') d\prob(\x')
\end{equation}
\end{definition}

\begin{definition}[Time-dependent random feature function]
\label{def. flow RFM}
Given any $V \in \HS(\R^{d+1},\R^d)$, one can define a flow by
\begin{align*}
G_V(\x_0) &= \x_1, \quad \frac{d}{d\tau}\x_{\tau} = V(\x_{\tau},\tau)
\end{align*}
Define the flow-induced norm
\begin{equation*}
\|V\|_{\F} = \exp \|V\|_{\HS}
\end{equation*}
\end{definition}

Our results adopt either of the following settings:
\begin{assumption}
\label{assume. ReLU}
Assume that the activation $\sigma$ is ReLU $\sigma(x)=\max(x,0)$.
Assume that the parameter distribution $\rho$ is supported on the $l^1$ sphere $\{\|\w\|_1+|b|=1\}$ and has a positive and continuous density over this sphere.
\end{assumption}

\begin{assumption}
\label{assume. sigmoid}
Assume that the activation $\sigma$ is sigmoid $\sigma(x)=\frac{e^x}{1+e^x}$.
Assume that $\rho$ has a positive and continuous density function over $\R^{d+1}$ and also bounded variance
\begin{equation*}
\int (\|\w\|^2 + b^2) d\rho(\w,b) \leq 1
\end{equation*}
%For instance, $\rho$ can be the Gaussian $\N(\mathbf{0},\frac{1}{d+1}I)$.
\end{assumption}

Given either assumption, the universal approximation theorems \cite{hornik1991approximation,sun2018RFM} imply that the space $\HS(K,\R^k)$ is dense among the continuous functions $C(K,\R^k)$ with respect to the $C_0$ norm for any compact subset $K\subseteq\R^d$.
Also, by Lemma \ref{lemma. universal approximation L2}, $\HS(\R^d,\R^k)$ is dense in $L^2(P,\R^k)$ for any distribution $P\in\PS(\R^d)$.

The random feature functions (\ref{eq. RFM}) can be seen as a simplified form of neural networks, e.g. if we replace the parameter distribution $\rho$ by a finite sample set $\{(\w_i,b_i)\}_{i=1}^m$, then (\ref{eq. RFM}) becomes a 2-layer network with $m$ neurons and frozen first layer weights.
Similarly, for the flow $G_V$ in Definition \ref{def. flow RFM}, if the ODE is replaced by a forward Euler scheme, then $G_V$ becomes a deep residual network whose layers share similar weights.
Beyond the random feature functions, one can extend the analysis to the Barron functions \cite{e2019barron,barron1991approximation} and flow-induced functions \cite{e2019residual}, which are the continuous representations of 2-layer networks and residual networks.

\subsection{Training rule}
\label{sec. training rule}

The training of distribution learning models is very similar to training supervised learning models, such that one chooses from the many algorithms for gradient descent and optimizes the function parameters.
One exception is the GANs, whose losses are min-max problems of the form (\ref{eq. classical GAN loss}, \ref{eq. WGAN loss}) and are usually solved by performing gradient descent on the generator and gradient ascent on the discriminator \cite{goodfellow2014generative}.

For the theoretical analysis in this paper, we use the continuous time gradient descent.
Specifically, given any loss $L(f)$ over $L^2(\prob,\R^k)$ for some $\prob\in\PS(\R^d)$, we parametrize $f$ by the random feature function $f_{\a}$ from Definition \ref{def. RFM} and denote the loss by $L(\a)=L(f_{\a})$.
Given any initialization $\a_0 \in L^2(\rho,\R^k)$, we define the trajectory $\{\a_t, t\geq 0\}$ by the dynamics
\begin{equation}
\label{eq. GD continuous time}
\frac{d}{dt}\a_t = -\nabla_{\a} L \big|_{\a_t} = -\int \nabla_f L(\x) ~\sigma(\w\cdot\x+b) d\prob(\x)
\end{equation}
It follows that the function $f_t = f_{\a_t}$ evolves by
\begin{equation}
\label{eq. population trajectory}
\frac{d}{dt}f_t = \E_{\rho(\w,b)} \Big[ \frac{d}{dt} \a_t(\w,b) ~\sigma(\w\cdot\x+b) \Big] = -K(\nabla_f L)
\end{equation}
where $K$ is the integral operator defined in (\ref{eq. integral operator}).
Similarly, given the empirical loss $L^{(n)}$ (\ref{eq. empirical loss}), we define the empirical training trajectory
\begin{equation}
\label{eq. empirical trajectory}
\frac{d}{dt}f_t^{(n)} = -K(\nabla_f L^{(n)})
\end{equation}
By default, we use the initialization:
\begin{equation}
\label{eq. RFM initialization}
\a_0=\a^{(n)}_0 \equiv \mathbf{0}
\end{equation}
or equivalently $f_0 = f^{(n)}_0 \equiv \mathbf{0}$.

\subsection{Test error}

For our theoretical analysis, given a modeled distribution $P$ and target distribution $P_*$, we measure the test error by either the Wasserstein metric $W_2(P_*, P)$ or KL-divergence $\KL(P_*\|P)$.
As discussed in Section \ref{sec. introduction}, $W_2$ exhibits the curse of dimensionality, while $\KL$ is stronger than $W_2$.
Thus, they are capable of detecting memorization and can distinguish the solutions that generalize well.

In addition, one advantage of the $W_2$ metric is that it can be related to the regression loss.

\begin{proposition}[Proposition 21 of \cite{yang2022GAN}]
\label{prop. W2 matching}
Given any base distribution $\prob \in \PS_{2,ac}(\R^d)$ and any target distribution $P_* \in \PS_2(\R^d)$,
for any $G \in L^2(\prob,\R^d)$
\begin{equation*}
W_2(P_*, G\#\prob) = \inf \big\{ \|G-G_*\|_{L^2(\prob)} ~|~ G_* \in L^2(\prob,\R^d), ~P_* = G_*\#\prob \big\}
\end{equation*}
\end{proposition}
So effectively, the $W_2$ test error is the $L^2$ error with the closest target generator.

One remark is that these test losses are only applicable to theoretical analysis.
In practice, we only have a finite sample set from $P_*$, and the curse of dimensionality becomes an obstacle to meaningful evaluation.

\section{Universal Approximation Theorems}
\label{sec. approximation}

It is not surprising that distribution learning models equipped with neural networks satisfy the universal approximation property in the space of probability distributions.
The significance is that the models in general have the capacity for memorization.
This section confirms that the universal approximation property holds for all three distribution representations introduced in Section \ref{sec. distribution representation}.
Since our results are proved with the random feature functions (Definitions \ref{def. RFM} and \ref{def. flow RFM}), they hold for more expressive function parametrizations such as 2-layer and deep neural networks.

For the free generator representation, the following result is straightforward.
\begin{proposition}
\label{prop. RFM W2 approximation}
Given either Assumption \ref{assume. ReLU} or \ref{assume. sigmoid}, for any base distribution $\prob \in \PS_{2,ac}(\R^d)$,
the set of distributions generated by the random feature functions $\HS(\R^d,\R^d)$ are dense with respect to the $W_2$ metric.
Specifically, for any $P_* \in \PS_2(\R^d)$,
\begin{equation*}
\inf_{G\in \HS} W_2(P_*, G\#\prob) = 0
\end{equation*}
\end{proposition}
In particular, $G\#\prob$ can approximate the empirical distribution $P_*^{(n)}$.

\subsection{Potential representation}
Consider the potential-based representation (\ref{eq. potential representation}).
Let $K \subseteq \R^d$ be any compact set with positive Lebesgue measure, let $\PS_{ac}(K) \cap C(K)$ be the space of distributions with continuous density functions, and let the base distribution $\prob$ be uniform over $K$.

\begin{proposition}
\label{prop. universal approximation potential}
Given either Assumption \ref{assume. ReLU} or \ref{assume. sigmoid}, the set of probability distributions
\begin{equation*}
\PS_{\HS} = \Big\{\frac{1}{Z} e^{-V}\prob ~\Big|~ V \in \HS(\R^d,\R) \Big\}
\end{equation*}
is dense in
\begin{itemize}
\item $\PS(K)$ under the Wasserstein metric $W_p$ ($1\leq p < \infty$),
\item $\PS_{ac}(K)$ under the total variation norm $\|\cdot\|_{\TV}$,
\item $\PS_{ac}(K) \cap C(K)$ under $\KL$ divergence.
\end{itemize}
\end{proposition}

\subsection{Flow-based free generator}

For the normalizing flows, we have seen in Section \ref{sec. combination} the two common approaches for modeling the generator $G$, i.e. continuous-time flow (\ref{eq. ODE generator}) or concatenation of simple diffeomorphisms.
Both approaches have an apparent issue, that they do not satisfy the universal approximation property among functions.
Since $G$ is always a diffeomorphism, it cannot approximate for instance functions that are overlapping or orientation-reversing (such as $x \mapsto |x|$ and $x \mapsto -x$).
Hence, the approach of Proposition \ref{prop. RFM W2 approximation} is not applicable.

Nevertheless, to transport probability distributions, it is sufficient to restrict to specific kinds of generators, for instance gradient fields $\nabla\phi$ according to Brennier's theorem \cite{brenier1991polar,villani2003topics}.
Using flows induced by random feature functions, we have the following result.

\begin{proposition}
\label{prop. flow approximation}
Given Assumption \ref{assume. sigmoid}, for any base distribution $\prob \in \PS_{2,ac}(\R^d)$, the following set of distributions is dense in $\PS_2(\R^d)$ with respect to the $W_2$ metric
\begin{equation*}
\PS_G = \big\{ G_V\#\prob ~\big|~ V\in\HS(\R^{d+1},\R^d) \big\}
\end{equation*}
where $G_V$ is given by Definition \ref{def. flow RFM}.
\end{proposition}

\subsection{Fixed generator}

For the fixed generator representation, we analyze the normalizing flow with stochastic interpolants (\ref{eq. Radon-Nikodym velocity}) instead of the score-based diffusion models (\ref{eq. diffusion SDE}), since the former has a simpler formulation.
As the target velocity field $V_*$ has been specified, we show a stronger result than Propositions \ref{prop. RFM W2 approximation} and \ref{prop. flow approximation}, such that we can simultaneously bound the $W_2$ test error and the $L^2$ training loss.

\begin{proposition}[Proposition 3.2 of \cite{yang2022flow}]
\label{prop. diffusion universal approximation}
Given Assumption \ref{assume. sigmoid}, assume that the base distribution $\prob$ is compactly-supported and has $C^2$ density.
For any target distribution $P_*$ that is compactly-supported, let $V_*$ be the velocity field (\ref{eq. Radon-Nikodym velocity}).
Then, for any $\ep > 0$, there exists a velocity field $V_{\ep} \in \HS(\R^{d+1},\R^d)$ with induced generator $G_{\ep}=G_{V_{\ep}}$ given by Definition \ref{def. flow RFM}, such that
\begin{align*}
W_2(P_*, G_{\ep}\#\prob) &< \ep\\
\|V_*-V_{\ep}\|_{L^2(M_*)} = \sqrt{2}\sqrt{L(V_{\ep})-L(V_*)} &< \ep
\end{align*}
where $M_*$ is the joint distribution (\ref{eq. joint interpolant distribution}) and $L$ is the loss (\ref{eq. diffusion NF loss}).
\end{proposition}

\section{Memorization}
\label{sec. memorization}

The previous section has shown that the distribution learning models, from all known distribution representations, satisfy the universal approximation property.
In particular, they are capable of approximating the empirical distribution $P_*^{(n)}$ and thus have the potential for memorization.
This section confirms that memorization is inevitable for some of the models.
Specifically, we survey our results on the universal convergence property, that is, the ability of a model to converge to any given distribution during training.
We believe that this property holds for other models as well, and it should be satisfied by any desirable model for learning probability distributions.

\subsection{Bias-potential model}

Recall that the bias-potential model is parametrized by potential functions $V$ (\ref{eq. potential representation}) and minimizes the loss $L$ (\ref{eq. bias-potential model}).
For any compact set $K \subseteq \R^d$ with positive Lebesgue measure, let the base distribution $\prob$ be uniform over $K$.
Parametrize $V$ by random feature functions $\HS(\R^d,\R)$, and define the training trajectory $V_t$ by continuous time gradient descent (\ref{eq. population trajectory}) on $L$ with any initialization $V_0 \in \HS$.
Denote the modeled distribution by $P_t = \frac{1}{Z} e^{-V_t} \prob$.
Similarly, let $V_t^{(n)}$ be the training trajectory on the empirical loss (\ref{eq. empirical loss}) and denote its modeled distribution by $P_t^{(n)}$.

\begin{proposition}[Lemma 3.8 of \cite{yang2022potential}]
\label{prop. potential universal convergence}
Given Assumption \ref{assume. ReLU}, for any target distribution $P_*\in \PS(K)$,
if $P_t$ has only one weak limit, then $P_t$ converges weakly to $P_*$
\begin{equation*}
\lim_{t\to\infty} W_2(P_*,P_t) = 0
\end{equation*}
\end{proposition}

\begin{corollary}[Memorization (Proposition 3.7 of \cite{yang2022potential})]
Given Assumption \ref{assume. ReLU}, the training trajectory $P_t^{(n)}$ can only converge to the empirical distribution $P_*^{(n)}$.
Moreover, both the test error and the norm of the potential function diverge
\begin{equation*}
\lim_{t\to\infty} \KL(P_*\|P_t^{(n)}) = \lim_{t\to\infty} \|V_t^{(n)}\|_{\HS} = \infty
\end{equation*}
\end{corollary}

In the setting of Proposition \ref{prop. potential universal convergence}, a limit point always exists, but we have to exclude the possibility of more than one limit point, e.g. the trajectory $P_t$ may converge to a limit circle.
We believe that with a more refined analysis one can prove that such exotic scenario cannot happen.

\subsection{GAN discriminator}
\label{sec. GAN convergence}

As will be demonstrated in Section \ref{sec. training}, the training and convergence of models with the free generator representation is in general difficult to analyze.
Thus, we consider the simplified GAN model from \cite{yang2022GAN} such that the representation $G\#\prob$ replaced by a density function $P$.

Consider the GAN loss (\ref{eq. dual norm squared}), and parametrize the discriminator by $D\in\HS(\R^d,\R)$.
Equivalently, we set the penalty term $\|D\|$ to be the RKHS norm $\|\cdot\|_{\HS}$ and the loss (\ref{eq. dual norm squared}) becomes
\begin{align}
L(P) &= \sup_{D} \int D(\x) d(P-P_*)(\x) - \|D\|^2_{\HS} \nonumber \\
&= \max_a \iint a(\w,b)\sigma(\w\cdot\x+b) d\rho(\w,b) d(P-P_*)(\x) - \|a\|^2_{L^2(\rho)} \nonumber \\
&= \frac{1}{2} \iint k(\x,\x') d(P-P_*)(\x)d(P-P_*)(\x') \label{eq. MMD GAN}
\end{align}
where $k$ is the kernel function from Definition \ref{def. RKHS}.
This loss is an instance of the maximum mean discrepancy \cite{gretton2012kernel}.
Model the density $P$ as a function in $L^2([0,1]^d)$, and define the training trajectory $P_t$ by continuous time gradient descent
\begin{equation}
\label{eq. MMD GAN training}
\frac{d}{dt} P_t = -\nabla_P L(P_t) = -k*(P_t-P_*)
\end{equation}
where $(k*P)(\x)=\int k(\x,\x') dP(\x')$.
Let $\Pi_{\Delta}$ be the nearest point projection from $L^2([0,1]^d)$ to the convex subset $\PS([0,1]^d) \cap L^2([0,1]^d)$.
We measure the test error by $W_2(P_*,\Pi_{\Delta}(P))$.

\begin{proposition}[Lemma 13 of \cite{yang2022GAN}]
Given Assumption \ref{assume. ReLU}, for any target distribution $P_* \in \PS([0,1]^d)$ and any initialization $P_0 \in L^2([0,1]^d)$, the distribution $\Pi_{\Delta}(P_t)$ converges weakly to $P_*$
\begin{equation*}
\lim_{t\to\infty} W_2\big(P_*,\Pi_{\Delta}(P_t)\big) = 0
\end{equation*}
\end{proposition}

\section{Generalization Error}
\label{sec. generalization}

Despite that Sections \ref{sec. approximation} and \ref{sec. memorization} have demonstrated that distribution learning models have the capacity for memorization, this section shows that solutions with good generalization are still achievable.
For the four classes highlighted in Table \ref{table. categorization}, we show that their models escape from the curse of dimensionality with either early stopping or regularization.
Specifically, their generalization errors scale as $O(n^{-\alpha})$ where $\alpha$ are absolute constants, instead of dimension-dependent terms such as $\alpha/d$.

These results depend on either of the two forms of regularizations
\begin{itemize}
\item Implicit regularization: As depicted in Figure \ref{fig: generalizable path} (Left), the training trajectory $P_t$ comes very close to the hidden target distribution $P_*$ before eventually turning towards the empirical distribution $P_*^{(n)}$
\item Explicit regularization: Analogous to the above picture, we consider some regularized loss $L^{(n)} + R(\lambda)$ with strength $\lambda \geq 0$.
With an appropriate regularization strength, the minimizer $P_{\lambda}$ becomes very close to the hidden target $P_*$.
\end{itemize}
The mechanism underlying both scenarios is that the function representations of the models are insensitive to the sampling error $P_*-P_*^{(n)}$.
Thus, we resolve the seeming paradox between good generalization and the inevitable memorization.

Without a good function representation, this behavior cannot be guaranteed.
For instance, as argued in \cite{yang2022GAN}, if a distribution $P_t$ is trained by Wasserstein gradient flow (i.e. without any function parametrization) on the empirical loss
\begin{equation*}
L^{(n)}(P) = W_2(P_*^{(n)}, P)
\end{equation*}
and if the initialization $P_0 \neq P_*$ is in $\PS_{2,ac}(\R^d)$, then the training trajectory $P_t$ follows the $W_2$ geodesic that connects $P_0$ and $P_*^{(n)}$.
Since the Wasserstein manifold has positive curvature \cite{ambrosio2008gradient}, the geodesic in general can never come close to the hidden target $P_*$, as depicted in Figure \ref{fig: generalizable path} (Right).

\begin{figure}[H]
\centering
\subfloat{\includegraphics[scale=0.27]{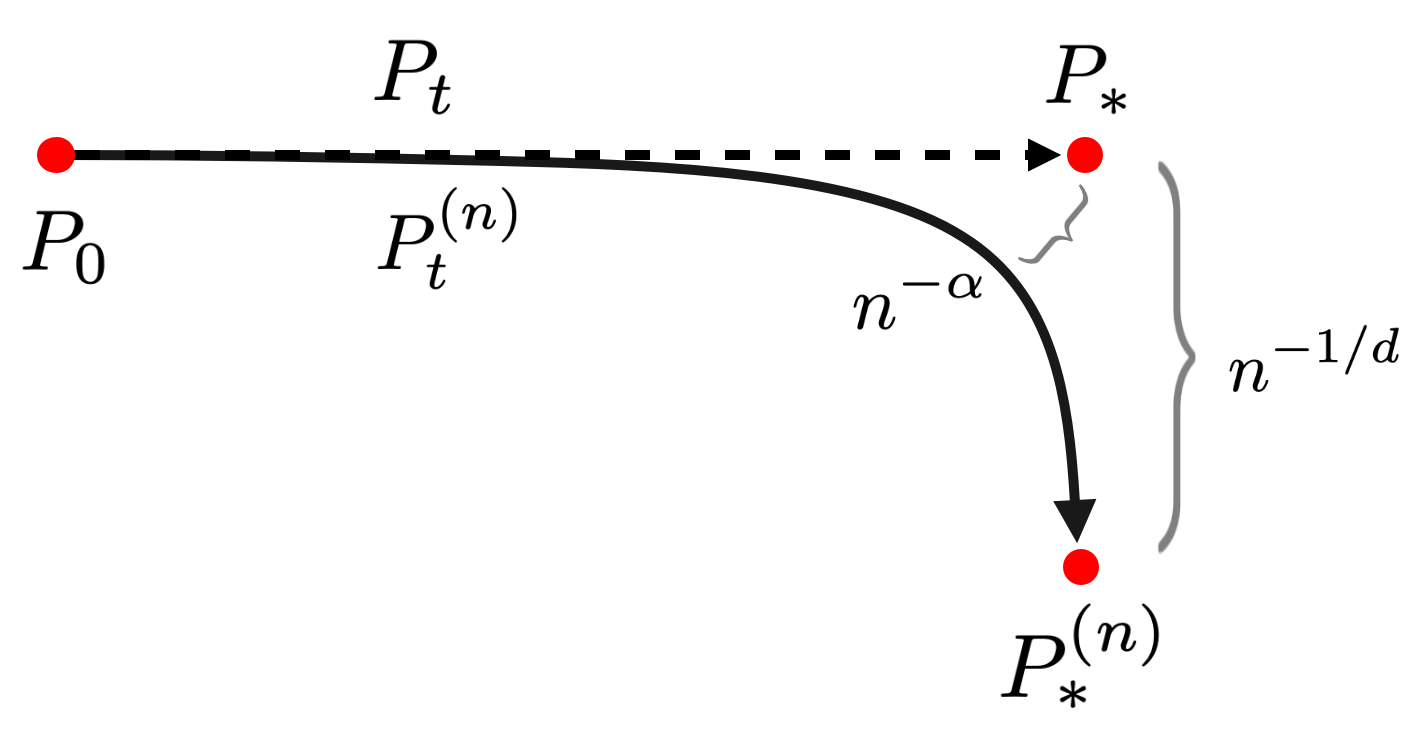}}
\quad
\subfloat{\includegraphics[scale=0.27]{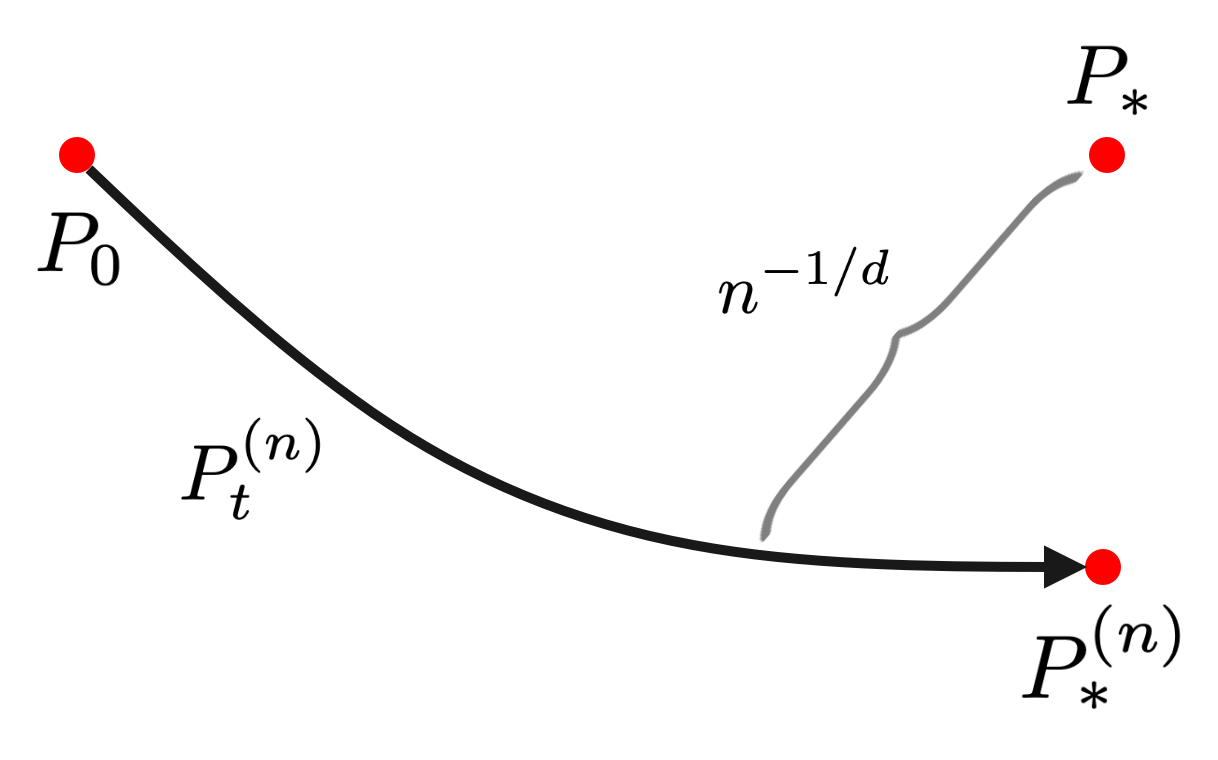}}
\caption{Left: Implicit regularization enables $P^{(n)}_t$ to stay close to $P_t$ and thus approximate $P_*$ better than $P_*^{(n)}$. Right: Wasserstein gradient flow on $W_2$ loss.}
\label{fig: generalizable path}
\end{figure}

In the following five subsections, we survey our results on three models that have rigorous proofs, and then analyze two models with heuristic calculations.

\subsection{Bias-potential model}

We start with the bias-potential model (\ref{eq. bias-potential model}) since
it enjoys the most convexity and thus the arguments are the most transparent.

Consider the domain $\Omega=[0,1]^d$ with base distribution $\prob \in \PS(\Omega)$.
Let $V_t, V^{(n)}_t \in \HS$ be potential functions trained on the population loss $L$ (\ref{eq. bias-potential model}) and the empirical loss $L^{(n)}$ (\ref{eq. empirical loss}) respectively, using continuous time gradient descent (\ref{eq. population trajectory}).
Denote their induced distributions (\ref{eq. potential representation}) by $P_t,P^{(n)}_t$.

\begin{theorem}[Theorem 3.3 of \cite{yang2022potential}]
\label{thm. bias potential generalization}
Given Assumption \ref{assume. ReLU}, assume that the target distribution $P_*$ has the form (\ref{eq. potential representation}) with a potential function $V_* \in \HS$.
For any $\delta >0$, with probability $1-\delta$ over the sampling of $P_*^{(n)}$,
\begin{equation}
\label{eq. bias potential generalization error}
\KL\big(P_*\|P_t^{(n)}\big) \leq \frac{\|V_*\|^2_{\mathcal{H}}}{2t} + \frac{8\sqrt{2\log 2d} + 2\sqrt{2\log(2/\delta)}}{\sqrt{n}} t
\end{equation}
\end{theorem}

\begin{corollary}
\label{cor. bias potential early stopping}
Given the condition of Theorem \ref{thm. bias potential generalization}, if we choose an early-stopping time $T$ such that
\begin{equation*}
T = \Theta\Big(\|V_*\|_{\mathcal{H}}\big(\frac{n}{\log d}\big)^{1/4}\Big)
\end{equation*}
then the test error satisfies
\begin{equation*}
\KL\big(P_*\|P_T^{(n)}\big) \lesssim \|V_*\|_{\mathcal{H}} \Big(\frac{\log d}{n}\Big)^{1/4}
\end{equation*}
\end{corollary}

Hence, the generalization error escapes from the curse of dimensionality.

The two terms in the upper bound (\ref{eq. bias potential generalization error}) are the training error and generalization gap.
The former is a consequence of convexity, while the latter follows from the observation that the landscapes of $L$ and $L^{(n)}$ differ very little
\begin{align*}
\Big\|\frac{\delta L-L^{(n)}}{\delta a}\Big\|_{L^2(\rho)} &\leq \sup_{\|V\|_{\HS}\leq 1} \Big| \int V(\x) d(P_*-P_*^{(n)})(\x) \Big|\\
&\lesssim Rad_n\big(\{\|V\|_{\HS} \leq 1\}\big) + \frac{\sqrt{\log 1/\delta}}{\sqrt{n}}
\end{align*}
where the $Rad_n$ term is the Rademacher complexity and scales as $O(1/\sqrt{n})$.
Then, $V_t, V^{(n)}_t$ remain close during training
\begin{equation*}
\|V_t-V_t^{(n)}\|_{\HS} \lesssim \int_0^t \Big\|\frac{\delta L-L^{(n)}}{\delta a}\Big\|_{L^2(\rho)} \lesssim \frac{t}{\sqrt{n}}
\end{equation*}
which confirms the depiction in Figure \ref{fig: generalizable path} (Left).

\vs
In the meantime, the bias potential model also generalizes well in the explicit regularization setting.
Here we consider the Ivanov and Tikhonov regularizations \cite{ivanov1976theory,tikhonov1977solutions,oneto2016tikhonov}.

\begin{proposition}[Proposition 3.9 of \cite{yang2022potential}]
\label{prop. bias potential Ivanov}
Given Assumption \ref{assume. ReLU}, assume that the target $P_*$ is generated by a potential $V_* \in \HS$.
Let $V_R^{(n)}$ be the minimizer of the regularized loss
\begin{equation*}
\min_{\|V\|_{\HS}\leq R} L^{(n)}(V)
\end{equation*}
where $R$ is any constant such that $R \geq \|V_*\|_{\HS}$.
For any $\delta>0$, with probability $1-\delta$ over the sampling of $P_*^{(n)}$, the distribution $P^{(n)}_R$ generated by the potential $V^{(n)}_R$ satisfies
\begin{equation*}
\KL(P_*\|P^{(n)}_R) \leq \frac{8\sqrt{2\log 2d} + 2\sqrt{2\log(2/\delta)}}{\sqrt{n}} R
\end{equation*}
\end{proposition}

\begin{proposition}
\label{prop. bias potential Tikhonov}
Given the condition of Proposition \ref{prop. bias potential Ivanov},
for any $\delta>0$, let $V_{\lambda}^{(n)}$ be the minimizer to the regularized loss
\begin{equation*}
\min_{V \in \HS} L^{(n)}(V) + \frac{\lambda}{\sqrt{n}} \|V\|_{\HS}, \quad \lambda \geq 4\sqrt{2\log 2d} + \sqrt{2\log(2/\delta)}
\end{equation*}
With probability $1-\delta$ over the sampling of $P_*^{(n)}$, the distribution $P^{(n)}_{\lambda}$ generated by $P^{(n)}_R$ satisfies
\begin{equation*}
\KL(P_*\|P^{(n)}_{\lambda}) \leq \frac{2\lambda\|V_*\|_{\HS}}{\sqrt{n}}
\end{equation*}
\end{proposition}

\subsection{Normalizing flow with stochastic interpolants}

Consider the normalizing flow with stochastic interpolants (\ref{eq. diffusion NF loss}), and model the velocity field $V$ and generator $G_V$ by Definition \ref{def. flow RFM}.
Denote by $V_t,V_t^{(n)}$ the training trajectories on the population and empirical losses (\ref{eq. diffusion NF loss}, \ref{eq. empirical loss}) using gradient flow (\ref{eq. population trajectory}, \ref{eq. empirical trajectory}).
Denote the generated distributions by $P_t=G_{V_t}\#\prob$ and $P^{(n)}_t=G_{V^{(n)}_t}\#\prob$.

First, we bound the generalization gap.

\begin{theorem}[Theorem 3.4 of \cite{yang2022flow}]
\label{thm. NF generalization gap}
Given Assumption \ref{assume. sigmoid}, for any compactly-supported base and target distributions $\prob$ and $P_*$, if the velocity field $V_*$ (\ref{eq. Radon-Nikodym velocity}) satisfies $V_* \in \HS(\R^{d+1},\R^d)$,
then with probability $1-\delta$ over the sampling of $P_*^{(n)}$,
\begin{equation*}
W_2(P_t, P_t^{(n)}) \leq \|V_*\|_{\F} \frac{(1+\sqrt{2\ln 2} + \sqrt{2\ln(2/\delta)}) (\frac{4}{3}Rt^{3/2} + 2Rt) + \sqrt{R^2+2} ~t}{\sqrt{n}}
\end{equation*}
where $R$ is the radius
\begin{equation*}
R = \sup\big\{\|\x\| ~\big|~ \x \in \sprt \prob \cup \sprt P_* \big\}
\end{equation*}
\end{theorem}

Next, to estimate the generalization error, we need a sharper norm to bound the training error.

\begin{definition}[Proposition 2.4 of \cite{yang2022flow}]
Given any distribution $M\in\PS(\R^{d+1})$, let $K$ be the integral operator (\ref{eq. integral operator}) over $L^2(M)$.
Given Assumption \ref{assume. sigmoid}, Proposition 2.4 of \cite{yang2022flow} implies that $K$ is a symmetric positive compact operator, and thus have an eigandecomposition with positive eigenvalues $\{\lambda_i\}_{i=1}^{\infty}$ and eigenfunctions $\{\phi_i\}_{i=1}^{\infty}$, which form an orthonormal basis of $L^2(M)$.
Define the subspace $\HS^2(M) \subseteq L^2(M)$ with the following norm
\begin{equation*}
\|V\|_{\HS^2(M)} = \sum_{i=1}^{\infty} \frac{\|\vb_i\|^2}{\lambda_i^2}
\end{equation*}
where the coefficients $\vb_i$ are from the decomposition $V=\sum_i \vb_i \phi_i$.
Furthermore, we have $\|V\|_{\HS(\sprt M)} \leq \|V\|_{\HS^2(M)}$ and thus $\|V\|_{\F} \leq \exp \|V\|_{\HS^2(M)}$.
\end{definition}

\begin{theorem}[Theorem 3.7 of \cite{yang2022flow}]
\label{thm. diffusion NF generalization}
Given Assumption \ref{assume. sigmoid}, assume that the base distribution $\prob$ is compactly-supported and has a $C^2$ density.
Let $P_*$ be any compactly-supported target distribution such that the velocity $V_*$ (\ref{eq. Radon-Nikodym velocity}) satisfies $V_*\ \in \HS^2(M_*)$, where $M_*$ is the joint distribution (\ref{eq. joint interpolant distribution}).
Then,
\begin{align*}
W_2(P_*, P^{(n)}_t) &\leq \frac{\|V_*\|_{\F} \|V_*\|_{\HS^2(M_*)}}{2\sqrt{t}}\\
&\quad\quad + \|V_*\|_{\F} \frac{(1+\sqrt{2\ln 2} + \sqrt{2\ln(2/\delta)}) (\frac{4}{3}Rt^{3/2} + 2Rt) + \sqrt{R^2+2} ~t}{\sqrt{n}}
\end{align*}
where the radius $R=\sup\big\{\|\x\| ~\big|~ \x \in \sprt \prob \cup \sprt P_* \big\}$.
\end{theorem}

In short, the generalization error scales as
\begin{equation*}
W_2(P_*,P^{(n)}_t) \lesssim \frac{1}{\sqrt{t}} + \frac{1}{\sqrt{n}} ~t^{3/2}
\end{equation*}
and thus with early-stopping $T=\Theta(n^{1/4})$, the model escapes from the curse of dimensionality
\begin{equation*}
W_2(P_*,P^{(n)}_T) \lesssim n^{1/8}
\end{equation*}

The condition $V_* \in \HS^2(M_*)$ may seem strict, so we present the following corollary that holds for general target distributions.
\begin{corollary}[Corollary 3.8 of \cite{yang2022flow}]
\label{cor. diffusion NF approximate}
Given Assumption \ref{assume. sigmoid}, assume that the base distribution $\prob$ is compactly-supported and has $C^2$ density.
Let $P_*$ be any compactly-supported target distribution.
For any $\ep > 0$, there exists a distribution $M_{\ep} \in \PS(\R^{d+1})$ and a velocity field $V_{\ep} \in \HS^2(M_{\ep})$ such that
\begin{align*}
W_2(P_*, P^{(n)}_t) &< \frac{\|V_*\|_{\F} \|V_*\|_{\HS^2(M_{\ep})}}{2\sqrt{t}} + \ep ~\big(\|V_{\ep}\|_{\F} ~t^{3/2} + 1\big)\\
&\quad\quad + \|V_*\|_{\F} \frac{(1+\sqrt{2\ln 2} + \sqrt{2\ln(2/\delta)}) (\frac{4}{3}Rt^{3/2} + 2Rt) + \sqrt{R^2+2} ~t}{\sqrt{n}}
\end{align*}
\end{corollary}

\subsection{GAN}

Similar to Section \ref{sec. GAN convergence}, we consider the simplified GAN model such that the modeled distribution is represented by a density function $P$.
We show that the discriminator $D$ alone is sufficient to enable good generalization.

With the discriminator modeled by $D\in\HS(\R^d,\R)$, the GAN loss (\ref{eq. dual norm squared}) becomes the maximum mean discrepancy $L$ in (\ref{eq. MMD GAN}).
Consider the training trajectory $P_t \in L^2([0,1]^d)$ defined by (\ref{eq. MMD GAN training}).
Similarly, define the empirical loss $L^{(n)}$ and empirical training trajectory $P^{(n)}_t$ by
\begin{align*}
L^{(n)}(P) &= \frac{1}{2} \iint k(\x,\x') d(P-P^{(n)}_*)^2(\x,\x')\\
P^{(n)}_t &= -\nabla_P L^{(n)}(P^{(n)}_t) = -k*(P^{(n)}_t-P^{(n)}_*)
\end{align*}
Fix some initialization $P_0=P^{(n)}_0 \in L^2([0,1]^d)$.
We measure the test error by $W_2(P_*,\Pi_{\Delta}(P))$, where $\Pi_{\Delta}$ is the nearest point projection onto $\PS([0,1]^d) \cap L^2([0,1]^d)$.

\begin{theorem}[Theorem 2 of \cite{yang2022GAN}]
Given Assumption \ref{assume. ReLU}, for any target density function $P_*$ such that $P_*-P_0 \in \HS$, with probability $1-\delta$ over the sampling of $P_*^{(n)}$,
\begin{equation*}
W_2\big(P_*,\Pi_{\Delta}(P_t^{(n)})\big) \leq \sqrt{d}\frac{\|P_*-P_0\|_{\HS}}{\sqrt{t}} + \sqrt{d}~ \frac{4\sqrt{2\log 2d} + \sqrt{2\log (2/\delta)}}{\sqrt{n}}t
\end{equation*}
\end{theorem}

It follows that with an early stopping time $T = \Theta(n^{1/3})$, the generalization error scales as $O(n^{-1/6})$ and escapes from the curse of dimensionality.

Part of this result can be extended to the usual case with a generator.
As shown in \cite{yang2022GAN}, if we consider the GAN loss $L(G)=L(G\#\prob)$, then the generator is insensitive to the difference between the landscapes of the population loss $L$ and empirical loss $L^{(n)}$,
\begin{equation*}
\Big\|\frac{\delta L}{\delta G}-\frac{\delta L^{(n)}}{\delta G}\Big\|_{L^2(\prob)} = \|\nabla k * (P_*-P_*^{(n)}) \|_{L^2(\prob)} \lesssim \frac{1}{\sqrt{n}}
\end{equation*}
The mechanism is that the sampling error $P_*-P^{(n)}_*$ is damped by $D$.

Furthermore, we can model the generator as a random feature function, $G \in \HS(\R^d,\R^d)$, and consider the population trajectory $G_t$ and empirical trajectory $G^{(n)}_t$, which are trained respectively on $L, L^{(n)}$ with continuous time gradient descent (\ref{eq. population trajectory}, \ref{eq. empirical trajectory}) and with the same initialization.
Assume that the activation $\sigma$ is $C^1$, then the difference $G_t-G^{(n)}_t$ grows slowly at $t=0$
\begin{equation*}
\frac{d}{dt} (G_t-G^{(n)}_t)(\z) \big|_{t=0} = k * \nabla\big( k * (P_*-P^{(n)}_*) \big)
\end{equation*}
Note that the sampling error $P_*-P^{(n)}_*$ is damped twice by both $G$ and $D$.

Finally, we remark that the generalization error of GANs have been studied in several related works using other kinds of simplified models, for instance when the generator $G$ is a linear map \cite{feizi2020LQG}, one-layer network (without hidden layer, of the form $G(\x)=\sigma(A\x)$ for $A\in\R^{d\times d}$) \cite{wu2019onelayer,lei2020sgd}, or polynomial with bounded degree \cite{dou2020making}.

\subsection{Score-based diffusion model}
As a further demonstration of the techniques from the previous sections,
this section presents an informal estimation of the generalization error of the score-based diffusion model (\ref{eq. diffusion Gaussian loss}).
By heuristic calculations, we derive a bound in the implicit regularization setting that resembles Theorem \ref{thm. diffusion NF generalization}.

Model the score function by $\mathbf{s} \in \HS(\R^{d+1},\R^d)$.
Let $\mathbf{s}_t, \mathbf{s}^{(n)}_t$ be the training trajectories on the population loss $L$ (\ref{eq. diffusion Gaussian loss}) and empirical loss $L^{(n)}$ (\ref{eq. empirical loss}) using gradient flow (\ref{eq. population trajectory}, \ref{eq. empirical trajectory}) with zero initialization $\mathbf{s}_0=\mathbf{s}^{(n)}_0 \equiv \mathbf{0}$.
Model the generators $G_t, G_t^{(n)}$ by the reverse-time SDE (\ref{eq. diffusion generator SDE}) with scores $\mathbf{s}_t, \mathbf{s}^{(n)}_t$.
Denote the generated distributions by $P_t=G_t\#\prob$ and $P_t^{(n)}=G^{(n)}_t\#\prob$.

For any target distribution $P_*$, denote the target score function by $\mathbf{s}_*=\nabla \log P_{\tau}$ with $P_{\tau}$ given by (\ref{eq. diffusion interpolant distribution}).
By inequality (\ref{eq. diffusion KL bound}),
\begin{equation*}
\KL(P_*\|P_t^{(n)}) \leq L(\mathbf{s}^{(n)}_t)-L(\mathbf{s}_*) + \KL(P_T\|\prob)
\end{equation*}

Assume that $\mathbf{s}_* \in \HS$.
Then, by convexity (see for instance \cite[Proposition 3.1]{yang2022potential})
\begin{equation*}
L(\mathbf{s}_t) - L(\mathbf{s}_*) \lesssim \frac{\|\mathbf{s}_*\|_{\HS}^2}{t}
\end{equation*}
Meanwhile, by the growth rate bound $\|\s_t\|_{\HS}, \|\s^{(n)}_t\|_{\HS} \lesssim \sqrt{t}$ from \cite[Proposition 5.3]{yang2022flow},
\begin{equation*}
L(\mathbf{s}^{(n)}_t)-L(\mathbf{s}_t) \lesssim (\|V^{(n)}_t\|_{C_0} + \|V_t\|_{C_0}) \|\mathbf{s}_t-\mathbf{s}^{(n)}_t\|_{\HS} \lesssim \sqrt{t} ~\|\mathbf{s}_t-\mathbf{s}^{(n)}_t\|_{\HS}
\end{equation*}
Then, using a calculation analogous to the proof of \cite[Theorem 3.4]{yang2022flow}
\begin{equation*}
\|\mathbf{s}_t-\mathbf{s}^{(n)}_t\|_{\HS} \lesssim Rad_n\big(\{\|\mathbf{s}\|_{\HS} \leq \sqrt{t}\}\big) t \lesssim \frac{t^{3/2}}{\sqrt{n}}
\end{equation*}
Combining these inequalities, we obtain
\begin{equation*}
\KL(P_*\|P_t^{(n)}) \lesssim \frac{\|\mathbf{s}_*\|_{\HS}^2}{t} + \frac{t^2}{\sqrt{n}} + \KL(P_T\|\prob)
\end{equation*}
Hence, if we ignore the approximation error $\KL(P_T\|\prob)$ due to finite $T$, the generalization error with early stopping scales as $O(n^{-1/6})$ and escapes from the curse of dimensionality.

\subsection{Normalizing flow}

This section presents an informal estimation of the generalization error of the normalizing flow model (\ref{eq. NF NLL}).
We conjecture an upper bound in the explicit regularization setting that resembles Proposition \ref{prop. bias potential Ivanov}.

Let the velocity field be modeled by $V \in \HS(\R^{d+1},\R^d)$, let $G_V$ be the flow map from Definition \ref{def. flow RFM}, and define the reverse-time flow map for $\tau\in[0,1]$,
\begin{equation*}
F_V(\x_1,\tau) = \x_{\tau}, \quad \frac{d}{d\tau}\x_{\tau} = V(\x_{\tau},\tau)
\end{equation*}
Let $L, L^{(n)}$ be the population and empirical losses (\ref{eq. NF NLL}, \ref{eq. empirical loss}).
Let the base distribution be the unit Gaussian $\prob=\N$.
For any target distribution $P_* \in \PS_2(\R^d)$ such that $P_*=G_{V_*}\#\prob$ for some $V_* \in \HS$, and for any $R \geq \|V_*\|_{\F}$, consider the problem with explicit regularization:
\begin{equation*}
\min_{\|V\|_{\F} \leq R} L^{(n)}(V)
\end{equation*}
Let $V^{(n)}_R$ be a minimizer.
It follows that
\begin{align*}
L(V^{(n)}_R) &\leq L^{(n)}(V^{(n)}_R) + \sup_{\|V\|_{\F} \leq R} L(V)-L^{(n)}(V)\\
&\leq L^{(n)}(V_*) + \sup_{\|V\|_{\F} \leq R} L(V)-L^{(n)}(V)\\
&\leq L(V_*) + 2 \sup_{\|V\|_{\F} \leq R} L(V)-L^{(n)}(V)
\end{align*}
Then,
\begin{align*}
&\quad \sup_{\|V\|_{\F} \leq R} L(V)-L^{(n)}(V)\\
&\leq \sup_{\|V\|_{\F} \leq R} \iint_0^1 \text{Tr}\big[\nabla V\big(F_V(\x_{\tau},\tau), \tau\big)\big] d\tau d(P_*-P^{(n)}_*)(\x)\\
&\quad + \sup_{\|V\|_{\F} \leq R} \iint_0^1 \frac{1}{2}\|F_V(\x,1)\|^2 d(P_*-P^{(n)}_*)(\x)
\end{align*}
Denote the two terms by the random variables $A, B$.
Using the techniques of \cite[Theorem 2.11]{e2019residual} and \cite[Theorem 3.3]{han2021class} for bounding the Rademacher complexity of flow-induced functions, one can try to bound the following expectations
\begin{align*}
\E[A] &\lesssim \frac{R}{\sqrt{n}} \E\big[\max_{1\leq i \leq n}\|X_i\|~\big|~X_i\sim^{i.i.d.} P_* \big] \lesssim \frac{R^2}{\sqrt{n}}\\
\E[B] &\lesssim \frac{R^2}{\sqrt{n}} \E\big[\max_{1\leq i \leq n}\|X_i\|^2~\big|~X_i\sim^{i.i.d.} P_* \big] \lesssim \frac{R^4}{\sqrt{n}}
\end{align*}
Meanwhile, for the random fluctuations $A-\E[A], B-\E[B]$, one can apply the extension of McDiarmid's inequality to sub-Gaussian random variables \cite{kontorovich2014concentration}, and try to show that, with probability $1-\delta$ over the sampling of $P_*^{(n)}$,
\begin{align*}
A-\E[A] &\lesssim \frac{R^2 \sqrt{\log 1/\delta}}{\sqrt{n}}. \quad B-\E[B] \lesssim \frac{R^4 \sqrt{\log 1/\delta}}{\sqrt{n}}
\end{align*}

Combining these inequalities, one can conjecture that the solution $P^{(n)}_R = G_{V^{(n)}_R}\#\prob$ satisfies
\begin{equation*}
\KL(P_*\|P^{(n)}_R) = L(V^{(n)}_R) - L(V_*) \lesssim \frac{1+\sqrt{\log 1/\delta}}{\sqrt{n}} R^4
\end{equation*}

\section{Training}
\label{sec. training}

This section studies the training behavior of distribution learning models, and illustrates the differences between the three distribution representations discussed in Section \ref{sec. distribution representation}.
On one hand, we survey our results on the global convergence rates of models with the potential representation and fixed generator representation.
On the other hand, we present new results on the landscape of models with the free generator representation, and analyze the mode collapse phenomenon of GANs.

\subsection{Potential and fixed generator}

Models with these two representations are easier to analyze since their losses are usually convex over abstract functions.
Specifically, this section considers the bias-potential model (\ref{eq. bias-potential model}) and normalizing flow with stochastic interpolants (\ref{eq. diffusion NF loss}):
\begin{align*}
L(V) &= \int V dP_* + \ln \int e^{-V} d\prob\\
L(V) &= \frac{1}{2} \int_0^1\iint \big\| V\big((1-\tau)\x_0+\tau\x_1, \tau\big) - (\x_1-\x_0) \big\|^2 d\prob(\x_0)dP_*(\x_1)d\tau
\end{align*}
If we choose a convex function representation for $V$, then the optimization problem becomes convex.

To estimate the rate of convergence, one approach is to bound the test error by a stronger norm.
The following toy example shows how to bound the $L^2$ error by the RKHS norm $\|\cdot\|_{\HS}$.

\begin{example}[Kernel regression]
\normalfont
Fix any base distribution $\prob\in\PS(\R^d)$ and assume that the activation $\sigma$ is bounded.
For any target function $f_* \in \HS$, consider the regression loss $L(f)=\frac{1}{2}\|f-f_*\|_{L^2(\prob)}^2$.
Parametrize $f$ by $f_a \in \HS$ and train the parameter function $a$ by continuous time gradient descent (\ref{eq. GD continuous time}) with initialization $a \equiv 0$.
Then,
\begin{equation*}
L(f_t) \leq \frac{\|f_*\|_{\HS}^2}{2t}
\end{equation*}
\end{example}

\begin{proof}[Proof one]
Denote the loss by $L(a)=L(f_a)$.
Since $\sigma$ is bounded, $L$ has continuous Fr\'{e}chet derivative in $a \in L^2(\rho)$, so the gradient descent (\ref{eq. GD continuous time}) is well-defined.
Choose $a_* \in L^2(\rho)$ such that $f_*=f_{a_*}$ and $\|f_*\|_{\HS}=\|a_*\|_{L^2(\rho)}$.
Define the Lyapunov function
\begin{equation*}
E(t) = t~\big(L(a_t)-L(a_*)\big) + \frac{1}{2} ||a_*-a_t||^2_{L^2(\rho_0)}
\end{equation*}
Then,
\begin{align*}
\frac{d}{dt} E(t) &= \big(L(a_t)-L(a_*)\big) + t \cdot \frac{d}{dt} L(a_t) + \big\lb a_t-a_*, ~\frac{d}{dt} a_t \big\rb_{L^2(\rho_0)}\\
&\leq \big(L(a_t)-L(a_*)\big) - \big\lb a_t-a_*, ~\nabla L(a_t) \big\rb_{L^2(\rho_0)}
\end{align*}
By convexity, for any $a_0,a_1$,
\begin{equation*}
L(a_0) + \lb a_1-a_0, ~\nabla L(a_0) \rb \leq L(a_1)
\end{equation*}
Hence, $\frac{d}{dt}E \leq 0$. We conclude that $E(t) \leq E(0)$.
\end{proof}

\begin{proof}[Proof two]
Since $a_t$ evolves by (\ref{eq. GD continuous time}), $f_t$ evolves by (\ref{eq. population trajectory}):
\begin{equation}
\label{eq. kernel regression dynamics}
\frac{d}{dt} f_t = -K(f_t-f_*)
\end{equation}
where $K$ is the integral operator (\ref{eq. integral operator}) over $L^2(\prob)$.
Since $K$ is symmetric, positive semidefinite and compact, there exists an eigendecomposition with non-negative eigenvalues $\{\lambda_i\}_{i=1}^{\infty}$ and eigenfunctions $\{\phi_i\}_{i=1}^{\infty}$ that form an orthonormal basis of $L^2(\prob)$.
Consider the decomposition
\begin{equation*}
f_t = \sum_{i=1}^{\infty} c^i_t \phi_i, \quad f_* = \sum_{i=1}^{\infty} c_*^i \phi_i
\end{equation*}
It is known that the RKHS norm satisfies \cite{rahimi2008uniform,cucker2002mathematical}
\begin{equation*}
\|f_*\|_{\HS}^2 = \sum_{i=1}^{\infty} \frac{(c_*^i)^2}{\lambda_i}
\end{equation*}
Since $c^i_0=0$ by assumption, (\ref{eq. kernel regression dynamics}) implies that $c^i_t = (1-e^{-\lambda_i t})c^i_*$.
Hence,
\begin{align*}
L(f_t) &= \frac{1}{2} \sum_{i=1}^{\infty} (c^i_t-c^i_*)^2 = \frac{1}{2} \sum_{i=1}^{\infty} (c_*^i)^2 e^{-2\lambda_i t}\\
&\leq \frac{1}{2} \sum_{i=1}^{\infty} \frac{(c_*^i)^2}{\lambda_i} \sup_{\lambda\geq 0} \lambda e^{-2\lambda t} = \frac{\|f_*\|_{\HS}^2}{4et}
\end{align*}
\end{proof}

Using similar arguments, we have the following bounds on the test error of models trained on the population loss.

\begin{proposition}[Bias-potential, Proposition 3.1 of \cite{yang2022potential}]
\label{prop. bias potential RFM}
Given the setting of Theorem \ref{thm. bias potential generalization},
the distribution $P_t$ generated by the potential $V_t$ satisfies
\begin{equation*}
\KL(P_*\|P_t) \leq \frac{\|V_*\|^2_{\HS}}{2t}
\end{equation*}
\end{proposition}

\begin{proposition}[NF, Proposition 3.5 of \cite{yang2022flow}]
\label{prop. training error}
Given the setting of Theorem \ref{thm. diffusion NF generalization}, the distribution $P_t = G_{V_t}\#\prob$ generated by the trajectory $V_t$ satisfies
\begin{equation*}
W_2(P_*, P_t) \leq \frac{\|V_*\|_{\F} \|V_*\|_{\HS^2(M_*)}}{2\sqrt{t}}
\end{equation*}
\end{proposition}

\begin{corollary}[NF, Corollary 3.6 of \cite{yang2022flow}]
Given the setting of Corollary \ref{cor. diffusion NF approximate}, let $P_*$ be any compactly-supported target distribution and $V_*$ be the target velocity field (\ref{eq. Radon-Nikodym velocity}).
For any $\ep > 0$, there exists a distribution $M_{\ep}\in\PS(\R^{d+1})$ and velocity field $V_{\ep}\in\HS^2(M_{\ep})$ such that
\begin{equation*}
W_2(P_*, P_t) < \frac{\|V_*\|_{\F} \|V_*\|_{\HS^2(M_{\ep})}}{2\sqrt{t}} + \ep ~\big(\|V_{\ep}\|_{\F} ~t^{3/2} + 1\big)
\end{equation*}
\end{corollary}

It seems probable that the $O(\ep t^{3/2})$ term can be strengthened to $O(\ep)$, which would imply universal convergence, i.e. convergence to any target distribution.

\subsection{Free generator: Landscape}

Models with the free generator representation are more difficult to analyze, since the loss $L(G)$ is not convex in $G$.
For instance, it is straightforward to check that if $\prob$ and $P_*$ are uniform over $[0,1]$, then the solution set $\{G_*~|~G_*\#\prob=P_*\}$, or equivalently the set of minimizers of $L$, is an infinite and non-convex subset of $L^2(\prob)$, and thus $L$ is non-convex.

Despite the non-convexity, there is still hope for establishing global convergence through a careful analysis of the critical points.
For instance, one can conjecture that there are no spurious local minima and that the saddle points can be easily avoided.
This section offers two results towards this intuition: a characterization of critical points for general loss functions of the form $L(G\#\prob)$, and a toy example such that global convergence can be determined from initialization.

In general, no matter how we parametrize the generator, the modeled distribution $P_t=G_t\#\prob$ satisfies the continuity equation during training:
Assume that $G(\x,\theta)$ is $C^1$ in $\x\in\R^d$ and $\theta\in\Theta$, where $\Theta$ is some Hilbert space.
Then, given any path $\theta_t$ that is $C^1$ in $t$,
for any smooth test function $\phi$,
\begin{align*}
\frac{d}{dt} \int \phi dP_t &= \frac{d}{dt} \int \phi\big(G(\x,\theta_t)\big) d\prob(\x) = \int \nabla \phi\big(G(\x,\theta_t)\big) \cdot \nabla_{\theta}G(\x,\theta_t) \dot{\theta}_t d\prob(\x)\\
&= \int \nabla \phi(\x) \cdot V_t(\x) dP_t(\x) = -\int \phi d\nabla \cdot (V_t P_t)
\end{align*}
where the velocity field $V_t$ is defined by
\begin{equation*}
\int \mathbf{f}(\x) \cdot V_t(\x) ~dP_t(\x) = \int \mathbf{f}\big(G(\x,\theta_t)\big) \cdot \nabla_{\theta}G(\x,\theta_t) \dot{\theta}_t ~d\prob(\x)
\end{equation*}
for any test function $\mathbf{f} \in L^2(P_t,\R^d)$.
Thus, $P_t$ is a weak solution to the continuity equation
\begin{equation*}
\partial_t P_t + \nabla \cdot(V_t P_t) = 0
\end{equation*}
In particular, the equation implies that no matter how $G$ is parametrized, the ``particles" of $P_t$ during training can only move continuously without jumps or teleportation.

For abstraction, it is helpful to consider the Wasserstein gradient flow \cite{santambrogio2015optimal,ambrosio2008gradient}:
For any loss function $L$ over $\PS_2(\R^d)$ and any initialization $P_0\in\PS_2(\R^d)$, define the training trajectory $P_t$ by
\begin{equation*}
\partial_t P_t + \nabla \cdot\Big(P_t\nabla\frac{\delta L}{\delta P}(P_t)\Big) = 0
\end{equation*}
where $\delta_P L$ is the first variation, which satisfies
\begin{equation*}
\frac{d}{d\ep} L\big(P+\ep(Q-P)\big) \big|_{\ep=0} = \int \frac{\delta L}{\delta P}(P_t) d(Q-P)
\end{equation*}
for any bounded and compactly-supported density function $Q$.

The Wasserstein gradient flow abstracts away the parametrization $\theta \mapsto G_{\theta}$ of the generator, and thus simplifies the analysis of critical points.
To relate to our problem, the following result shows that the Wasserstein gradient flow often shares the same critical points as the parametrized loss $L(\theta)=L(G_{\theta}\#\prob)$, and thus we are allowed to study the former instead.

\begin{proposition}[Comparison of critical points]
\label{prop. compare critical}
Given any loss $L$ over $\PS_2(\R^d)$, assume that the first variation $\delta_P L$ exists, is $C^1$, and $\nabla_{\x} \delta_P L(P) \in L^2(P,\R^d)$ for all $P\in\PS_2(\R^d)$.
Define the set of critical points
\begin{equation*}
CP = \big\{P\in\PS_2(\R^d) ~\big|~ \nabla_{\x} \delta_P L(P)(\x) = \mathbf{0} \text{ for } P \text{ almost all } \x \big\}
\end{equation*}
Given any base distribution $\prob\in\PS_2(\R^k)$ and any generator $G_{\theta}:\R^k\to\R^d$ parametrized by $\theta\in\Theta$, where $\Theta$ is a Hilbert space.
Assume that $G_{\theta}(\x)$ is Lipschitz in $\x$ for any $\theta$, and $C^1$ in $\theta$ for any $\x$, and that the gradient $\nabla_{\theta}G_{\theta}$ at any $\theta$ is a continuous linear operator $\Theta \to L^2(\prob,\R^d)$.
Define the set of generated distributions $\PS_{\Theta} = \{ G_{\theta}\#\prob, ~\theta\in\Theta\}$ and the set of critical points
\begin{equation*}
CP_{\Theta} = \big\{G_{\theta}\#\prob ~\big|~ \theta\in\Theta, ~\nabla_{\theta} L(G_{\theta}\#\prob) = 0 \big\}
\end{equation*}
Then, $CP \cap \PS_{\Theta} \subseteq CP_{\Theta}$.
Furthermore, if the parametrization $G_{\theta}$ is the random feature functions $G_{\a}\in\HS(\R^k,\R^d)$, and either Assumption \ref{assume. ReLU} or \ref{assume. sigmoid} holds, then
\begin{equation*}
CP_{\Theta} = CP \cap \PS_{\Theta}
\end{equation*}
\end{proposition}

Below is an example use case of this proposition.

\begin{example}
Consider the GAN loss $L(P)$ in (\ref{eq. MMD GAN}) induced by discriminators that are random feature functions (\ref{eq. RFM}).
Assume that the target distribution $P_* \in \PS_{2,ac}(\R^d)$ is radially symmetric, the parameter distribution $\rho$ in (\ref{eq. RFM}) is radially symmetric in $\w$ conditioned on any $b$, and that the activation $\sigma$ is $C^1$.
Consider the point mass $P=\delta_{\mathbf{0}}$.
Then,
\begin{align*}
\nabla\frac{\delta L}{\delta P}(P)(\mathbf{0}) &= \nabla \int k(\x,\x') d(P-P_*)(\x') \big|_{\x=\mathbf{0}}\\
&= \iint \w\sigma'(\w\cdot\mathbf{0}+b) \sigma(\w\cdot\x'+b) d(P-P_*)(\x') d\rho(\w,b)\\
&= \frac{1}{2} \iint \w\sigma'(b) \sigma(\w\cdot\x'+b) d(P-P_*)(\x') d\rho(\w|b) d\rho(b)\\
&\quad +\frac{1}{2} \iint (-\w)\sigma'(b) \sigma\big((-\w)\cdot(-\x')+b\big) d(P-P_*)(\x') d\rho(\w|b) d\rho(b)\\
&= \mathbf{0}
\end{align*}
Thus, $\nabla \delta_P L(P)= \mathbf{0}$ for $P$ almost all $\x$, and $P \in CP$.
It follows from Proposition \ref{prop. compare critical} that $P$ is a critical point ($P \in CP_{\Theta}$) for any parametrized generator $G_{\theta}$ that can express the constant zero function.
\end{example}

A probability measure $P$ is called singular if $P$ cannot be expressed as a density function ($P \in \PS(\R^d)-\PS_{ac}(\R^d)$).
The critical points of an expectation-based loss are often singular distributions such as $\delta_{\mathbf{0}}$ from the previous example.
The following toy model shows that the critical points may consist of only global minima and saddle points that are singular.

Consider a one-dimensional setting.
Model the generator by any function $G \in L^2(\prob,\R)$, where the base distribution $\prob$ is conveniently set to be uniform over $[0,1]$.
Given any target distribution $P_* \in \PS_{2,ac}(\R)$ and any initialization $G_0 \in L^2(\prob,\R)$, consider the dynamics
\begin{equation}
\label{eq. generalized W2 GD}
\frac{d}{dt} G_t(z) = \int x ~d\pi_t\big(x\big|G_t(z)\big) - G_t(z)
\end{equation}
where $\pi_t(\cdot|\cdot)$ is the conditional distribution of the optimal transport plan $\pi_t \in \PS(\R\times\R)$ between $P_t = G_t\#\prob$ and $P_*$.
By Brennier's theorem \cite{villani2003topics}, since $P_*$ is absolutely continuous, the optimal transport plan is unique.

Note that if $P_t$ is also absolutely continuous, then Brennier's theorem implies that the transport plan is deterministic: $\pi(\cdot|x_0) = \delta_{\nabla\phi(x_0)}$ for some convex potential $\phi$.
Since the first variation of $W_2^2(P, P_*)$ is exactly the potential $\phi$ \cite{santambrogio2015optimal}, the dynamics (\ref{eq. generalized W2 GD}) when restricted to $\{G~|~G\#\prob \in \PS_{ac}(\R)\}$ becomes equivalent to the gradient flow on the loss
\begin{equation*}
L(G) = W_2^2(G\#\prob, P_*)
\end{equation*}

Since the dynamics (\ref{eq. generalized W2 GD}) is not everywhere differentiable, we consider stationary points instead of critical points, and extend the definition of saddle points:
A stationary point $G\in L^2(\prob,\R^d)$ is a generalized saddle point if for any $\ep>0$, there exists a perturbation $h$ ($\|h\|_{L^2(\prob)} < \ep$) such that $L(G+h)<L(G)$.

\begin{proposition}
\label{prop. W2 GD landscape}
For any target distribution $P_*\in\PS_{2,ac}(\R)$, the stationary points of the dynamics (\ref{eq. generalized W2 GD}) consist only of global minima and generalized saddle points.
If $G$ is a generalized saddle point, then $G\#\prob$ is a singular distribution, and there exists $x$ such that $(G\#\prob)(\{x\})>0$.
Moreover, global convergence holds
\begin{equation*}
\lim_{t\to\infty}W_2(P_*, P_t) = 0
\end{equation*}
if and only if the initialization $P_0\in\PS_{2,ac}(\R)$.
\end{proposition}

This toy example confirms the intuition that despite the loss $L(G)$ is nonconvex, global convergence is still achievable.
Moreover, all saddle points have one thing in common, that part of the mass has collapsed onto one point.

\subsection{Free generator: Mode collapse}

The previous section has presented a general study of the critical points of models with free generator, while this section focuses on a particular training failure that is common to GANs, the mode collapse phenomenon. Mode collapse is characterized as the situation when the generator $G_t$ during training maps a positive amount of mass of $\prob$ onto the same point \cite{salimans2016improved,kodali2017convergence,mao2019mode,pei2021alleviating}, and is identified as the primary obstacle for GAN convergence.
This characterization is analogous to the saddle points from Proposition \ref{prop. W2 GD landscape} that are also singular distributions.
Despite that in the setting of Proposition \ref{prop. W2 GD landscape}, the singular distributions can be avoided and the toy model enjoys global convergence, how mode collapse occurs in practice remains a mystery.

To provide some insight into the mode collapse phenomenon, we demonstrate with toy examples two mechanisms that can lead to mode collapse.

Denote by $U[x,y]$ the uniform distribution over the interval $[x,y]$.
Let the base and target distributions be $\prob=P_*=U[0,1]$.
Model the generator by $G(x)=ax$, and discriminator by $D(x)=b\phi(x)$ for some $\phi$ to be specified.
Consider the following GAN loss based on (\ref{eq. dual norm squared})
\begin{equation*}
\min_a \max_b L(a,b) = \int D ~d(G\#\prob-P_*) + \frac{c}{2}|b|^2
\end{equation*}
where $c \geq 0$ is the strength of regularization.
Train $a,b$ by continuous time gradient flow
\begin{equation*}
\frac{d}{dt}a_t = -\partial_a L(a_t,b_t), \quad \frac{d}{dt}b_t = \partial_b L(a_t,b_t)
\end{equation*}
with initialization $a_0>0$ and $b_0=0$.
Mode collapse happens when $G_t\#\prob=[-a_t,a_t]$ becomes a singular distribution, i.e. when $a_t=0$.

\vs
\textbf{Case one.}
This example shows that a non-differentiable discriminator can lead to mode collapse.
Set $\phi(x) = |x| = \frac{\text{ReLU}(x)+\text{ReLU}(-x)}{2}$.
Restricting to the half space $\{(a,b), ~a>0\}$, the loss and training dynamics become
\begin{align*}
L(a,b) &= \frac{a-1}{2}b - \frac{cb^2}{2}\\
\frac{d}{dt}a_t &= -\frac{b}{2}, \quad \frac{d}{dt} b_t = \frac{a-1}{2} - cb
\end{align*}
Assume that $c \in [0,1)$.
Then, the unique solution is given by
\begin{align*}
a_t  &= 1 + (a_0-1) e^{-ct/2} \Big(\cos\big(\sqrt{1-c^2}~t/2\big) + \frac{c}{\sqrt{1-c^2}} \sin\big(\sqrt{1-c^2}~t/2\big) \Big)\\
b_t &= \frac{a_0}{\sqrt{1-c^2}} e^{-ct/2} \sin\big(\sqrt{1-c^2}~t/2\big)
\end{align*}
At time $T = \frac{2\pi}{\sqrt{1-c^2}}$, we have
\begin{equation*}
a_T = 1 - (a_0-1) e^{-c\pi/\sqrt{1-c^2}}
\end{equation*}
Thus, if $a_0 > 1+e^{c\pi/\sqrt{1-c^2}}$, then the trajectory $a_t$ must hit the line $\{a=0\}$ at some $t \in (0,T)$, and thus mode collapse happens.

This process is depicted in Figure \ref{fig: mode collapse} (Left).
In general, one can conjecture that mode collapse may occur at the locations where $\nabla D_t$ is discontinuous.

\begin{figure}[H]
\centering
\subfloat{\includegraphics[scale=0.4]{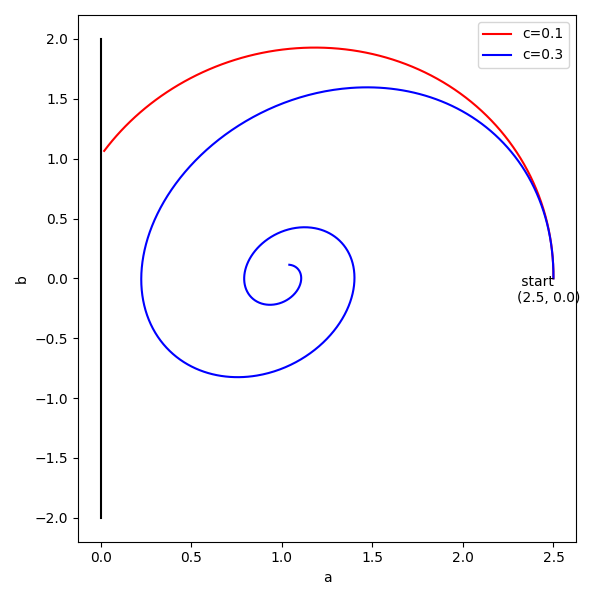}}
\quad
\subfloat{\includegraphics[scale=0.4]{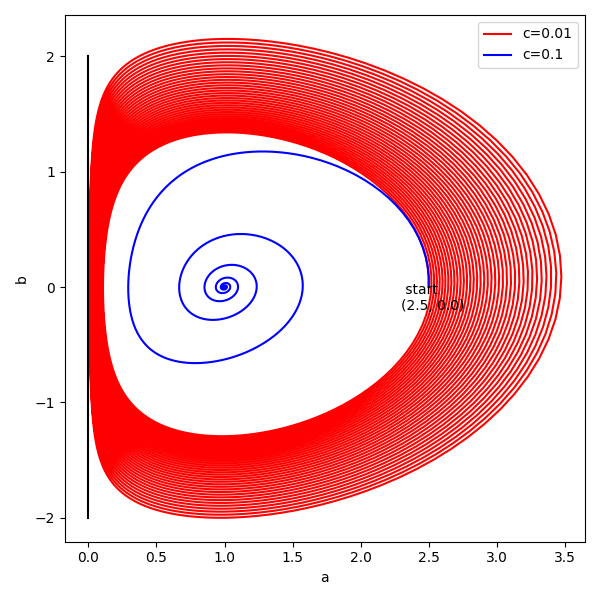}}
\caption{Mode collapse during GAN training. Left: Case 1. Right: Case 2. The horizontal and vertical axes are the parameters $a$ and $b$. The red curves are trained with small regularization $c$ and end up in mode collapse $a=0$, while the blue curves have larger $c$ and converge to the stationary point $(1,0)$. The initialization is $(2.5,0)$ and the learning rate for Case 2 is $\gamma=0.1$.}
\label{fig: mode collapse}
\end{figure}

\textbf{Case two.}
This example shows that the accumulation of numerical error due to the finite learning rate can lead to mode collapse.
This result is analogous to \cite{mescheder2018training} which shows that if $P_t$ is a point mass, numerical error can cause it to diverge to infinity.

Set $\phi=\frac{1}{2}x^2$.
The loss and training dynamics become
\begin{align*}
L(a,b) &= \frac{a^2-1}{6}b - \frac{cb^2}{2}\\
\frac{d}{dt}a_t &= -\frac{ab}{3}, \quad \frac{d}{dt} b_t = \frac{a^2-1}{6} - cb
\end{align*}
Given a learning rate $\gamma>0$, consider the discretized training dynamics
\begin{equation*}
a_{(k+1)\gamma} = a_{k\gamma} - \gamma \frac{a_{k\gamma} b_{k\gamma}}{3}, \quad b_{(k+1)\gamma} = b_{k\gamma} + \gamma \Big( \frac{a_{k\gamma}^2-1}{6} - cb_{k\gamma} \Big)
\end{equation*}
for every $k \in \mathbb{N}$

In general, given a discrete dynamics in the form of $\theta_{(k+1)\gamma} = \theta_{k\gamma} + \gamma f(\theta_{k\gamma})$, one can fit the sequence $\theta_{k\gamma}$ by the continuous time solution of an ODE $\frac{d}{dt}\theta_t = \tilde{f}(\theta_t)$, where
$\tilde{f}=f+\gamma h$ for some function $h$.
Assume that the two trajectories match at each $t=k\gamma$,
\begin{align*}
f(\theta_{k\gamma}) &= \gamma^{-1} (\theta_{(k+1)\gamma}-\theta_{k\gamma}) = \gamma^{-1} \int_0^{\gamma} \tilde{f}(\theta_{k\gamma+s}) ds = \int_0^1 (f+\gamma h)(\theta_{(k+s)\gamma}) ds\\
&= \int_0^1 f(\theta_{k\gamma}) + s\gamma \nabla f(\theta_{k\gamma}) \cdot f(\theta_{k\gamma}) + \gamma h(\theta_{k\gamma}) + O(\gamma^2) ds\\
&= f(\theta_{k\gamma}) + \frac{\gamma}{2} \nabla f(\theta_{k\gamma}) \cdot f(\theta_{k\gamma}) + \gamma h(\theta_{k\gamma}) + O(\gamma^2)
\end{align*}
It follows that
\begin{equation*}
\tilde{f} = f - \frac{\gamma}{2} \nabla f \cdot f + O(\gamma^2)
\end{equation*}

Plugging in $\theta_t = [a_t,b_t]$, we obtain the following approximate ODE
\begin{align*}
\frac{d}{dt}a_t &= -\frac{ab}{3} + \gamma \Big(-\frac{ab^2}{18} + \frac{a(a^2-1)}{36} - \frac{cab}{6}\Big)\\
\frac{d}{dt}b_t &= \frac{a^2-1}{6} - cb + \gamma \Big(\frac{a^2b}{18} + \frac{c(a^2-1)}{12} - \frac{c^2b}{2}\Big)
\end{align*}
Define the energy function
\begin{equation*}
H(a,b) = \frac{b^2}{2} + \frac{a^2-1}{4} - \frac{\log a}{2}
\end{equation*}
Then
\begin{equation*}
\frac{d}{dt} H(a_t,b_t) = -cb^2 + \gamma \Big( \frac{(a^2+1)b^2}{36} + \frac{(a^2-1)^2}{72} - \frac{c^2b^2}{2}\Big)
\end{equation*}
Thus, if $c \leq \min(\frac{1}{17}, \frac{\gamma}{144})$,
\begin{equation*}
\frac{d}{dt} H(a_t,b_t) \geq \frac{\gamma}{72}\big[(a^2+1)b^2 + (a^2-1)^2\big] > 0
\end{equation*}

The energy $H$ is strictly convex and its only minimizer is the stationary point $(a,b)=(1,0)$ where $H=0$.
Assume that the initialization $(a_0,b_0)\neq (1,0)$ and $a_0 > 0$.
Since the energy is non-decreasing, $H(a_t,b_t)\geq H(a_0,b_0) > 0$, and thus the trajectory will never enter the set $S_0 = \{(a,b), ~H(a,b)< H(a_0,b_0)\}$.
Since $S_0$ is an open set that contains $(1,0)$, there exists $r>0$ such that the open ball $B_r((1,0)) \subseteq S_0$.
Since $a_t \geq 0$, the trajectory $(a_t,b_t)$ satisfies
\begin{equation*}
\frac{d}{dt} H(a_t,b_t) \geq \frac{\gamma}{72}\big[(a^2+1)b^2 + (a+1)^2(a-1)^2\big] \geq \frac{\gamma r^2}{72}
\end{equation*}
and thus
\begin{equation*}
H(a_t,b_t) \geq H(a_0,b_0) + \frac{\gamma r^2}{72}t
\end{equation*}
Consider the four subsets $S_1=\{b\leq 0, ~a>1\}$, $S_2=\{b>0, ~0<a \leq 1\}$, $S_3=\{b\geq 0, ~0<a< 1\}$, $S_4=\{b< 0, ~a\geq 1\}$ that partition the space $\{(a,b)~|~(a,b)\neq (1,0), ~a>0\}$.
Then, the trajectory $(a_t,b_t)$ repeatedly moves from $S_1$ to $S_2$ to $S_3$ to $S_4$ and back to $S_1$.
In particular, it crosses the line $\{0<a<1, b=0\}$ for times $t_1, t_2 \dots$ that go to infinity.
It follow that
\begin{align*}
\inf_{t\geq 0} a_t &\leq \inf_{i\in\mathbb{N}} a_{t_i} \leq \inf_{i\in\mathbb{N}} \exp\Big[-2\Big(\frac{-\log a_{t_i}}{2}\Big)\Big]\\
&\leq \inf_{i\in\mathbb{N}} \exp\big[-2H(a_{t_i},b_{t_i})\big]\\
&\leq \inf_{i\in\mathbb{N}} \exp\Big[-2\big(H(a_0,b_0) + \frac{\gamma r^2}{72}t_i\big)\Big]\\
&= 0
\end{align*}
Hence, $a_t$ converges to $0$ exponentially fast, and then numerical underflow would lead to mode collapse.
This process is depicted in Figure \ref{fig: mode collapse} (Right).

\vs
In summary, these two examples indicate two possible causes for mode collapse.
On one hand, if the discriminator is non-smooth, then the gradient field $\nabla D_t$ may squeeze the distribution $G_t\#\prob$ at the places where $\nabla D_t$ is discontinuous and thus form a singular distribution.
On the other hand, the numerical error due to the finite learning rate may amplify the oscillatory training dynamics of $G_t$ and $D_t$, and if this error is stronger than the regularization on $D_t$, then the norm of $G_t$ can diverge and lead to collapse.

Note that, however, if we consider two-time-scale training such that $b_t$ is always the maximizer, then the loss becomes proportional to $L(a)=(|a|-1)^2$ or $(a^2-1)^2$, and training converges to the global minimum as long as $a_0 \neq 0$.

\section{Discussion}
\label{sec. discussion}

This paper studied three aspects of the machine learning of probability distributions.

First, we proposed a mathematical framework for generative models and density estimators that allows for a unified perspective on these models.
Abstractly speaking, the diversity of model designs mainly arises from one factor, the vertical or horizontal way of modeling discussed in Section \ref{sec. background}.
When applied to distribution representation, it leads to the options of potential representation and transport representation, and the latter ramifies into the free generator and fixed generator depending on whether a probabilistic coupling can be fixed.
Similarly, when applied to the loss, it leads to three loss types, depending on whether the difference is measured vertically or horizontally, with or without a fixed target.
Then, the rest of the design process is to try to realize each category in Table \ref{table. categorization} by satisfying the various constraints of implementation, for instance, whether to compute the density of $G\#\prob$ directly or indirectly, and which random path is chosen to achieve the product coupling $\prob\times P_*$.
Thereby, all the major models are derived.
By isolating the factors of distribution representation, loss type and function representation, this framework allows for a more transparent study of training and landscape (who depend more on distribution representation and loss type) and generalization error (which depends more on function representation).

Second, we studied the seeming conflict between the memorization phenomenon and the generalization ability of the models, and reviewed our results that resolve this conflict.
On one hand, we confirmed that the models satisfy the universal approximation property (some models even enjoy universal convergence) and thus memorization is inevitable.
On the other hand, function representations defined by expectations are insensitive to the sampling error $P_*-P_*^{(n)}$, so that the training trajectory tends to approximate the hidden target distribution before eventually diverging towards memorization.
In particular, our results established generalization error bounds of the form $O(n^{-\alpha})$ with $\alpha=\frac{1}{4},\frac{1}{6},\frac{1}{8}$ for several models.
There should be room for improvement, but for now we are content that the models can escape from the curse of dimensionality.
Considering that this generalization ability is mostly an effect of the function representations, it seems reasonable to expect that good generalization is enjoyed by all models regardless of the choice of distribution representation and loss type.

Third, we discussed the training dynamics and loss landscape.
For the potential and fixed generator representations, the convexity of their distribution parametrizations and loss functions enable the estimation of the rates of global convergence.
For the free generator representation, despite the loss is non-convex, we demonstrated that the critical points have tractable forms, and also identified two general mechanisms common to the min-max training of GANs that can provably lead to mode collapse.
It seems worthwhile to devote more effort in the design of models with the fixed generator representation, since they are as expressive as the free generator models while their convexity greatly eases training.
It is not clear at this moment whether the product coupling $\prob\times P_*$ is the best choice for the fixed generators, and whether the diffusion SDE and linear interpolants are the most effective random paths, so there is much to explore.

\vs
To conclude, we list a few interesting topics that have not been covered in this paper
\begin{itemize}
\item Unstructured data: Our analysis was conducted only in the Euclidean space, whereas most of the applications of distribution learning models involve unstructured data such as images, texts and molecules.
One thing of practical importance is that the performance of generative models for unstructured data is judged by human perception or perceptual loss \cite{johnson2016perceptual}, which can greatly differ from the Euclidean metric and thus the $W_2$ metric.
To train the model to have higher fidelity, one approach is to use the adversarial loss of GANs such that the hidden features of the discriminators can be seen as an embedding space that captures fidelity.
A related approach is to rely on a pretrained feature embedding as in \cite{rombach2022high,hou2017deep}.

\item Prior knowledge:
For supervised tasks involving unstructured data, it is often helpful to instill prior knowledge from humans into the models through self-supervised pretraining.
A well-known example is the approximate invariance of image classification with respect to color distortion and cropping, and models that are pretrained to be insensitive to these augmentations can achieve higher test accuracy after training \cite{chen2020simple,grill2020bootstrap,caron2021emerging}.
It could be beneficial to try to boost distribution learning models by prior knowledge.
One example is given by \cite{zhang2018monge} such that a generative model for sampling a thermodynamic system is designed to respect the spatial symmetry of the system.

\item Conditional generation: In practice, people are more interested in estimating conditional distributions, e.g. generate images conditioned on text descriptions \cite{ramesh2021zero,rombach2022high}. Incorporating a context variable can be done by simply allowing an additional input variable in the parameter functions \cite{mirza2014conditional}, but it can also be accomplished with tricks that minimize additional training \cite{song2020score}.

\item Factor discovery:
For generative models, instead of using $G\#\prob$ as a blackbox for sampling, it could be useful to train $G$ in a way such that $\prob$ has semantic meaning, e.g. for image synthesis, given a random image $G(Z), Z\sim\prob$ of a face, one coordinate of $Z$ may control hair style and another may control head orientation.
This unsupervised task is known as factor discovery \cite{tabak2018explanation}, and some solutions are provided by \cite{chen2016infogan,higgins2016beta,karras2019style} with application to semantic photo editing \cite{ling2021editgan}.

\item Density distillation:
The basic setting considered in this paper is to estimate a target distribution given a sample set; yet, another task common to scientific computing is to train a generative model to sample from a distribution $\frac{1}{Z}e^{-V}$ given a potential function $V$.
One popular approach \cite{li2018renormalization,zhang2018monge,noe2019boltzmann,cao2022learning} is to use a modified normalizing flow with the reverse KL-divergence.

%\item Procedural generation: 
\end{itemize}

\section{Proofs}
\label{sec. proof}

\subsection{Loss function}

\begin{proof}[Proof of Proposition \ref{prop. density loss NLL}]
For any $P_*,P\in\PS_{ac}(\R^d)$, the assumption on global minimum implies that
\begin{align*}
\lim_{t\to 0^+} \frac{1}{t} \big(L((1-t)P_*+tP)-L(P_*)\big) = \int f'(P(\x))(P(\x)-P_*(\x)) dP_*(\x) & \geq 0
\end{align*}
Define the map $g(p) = p f'(p)$. Then,
\begin{equation*}
\E_{P} [g(P_*(\x))] \geq \E_{P_*} [g(P_*(\x))]
\end{equation*}
Assume for contradiction that $g$ is nonconstant on $(0,\infty)$, then there exist $a,b > 0$ such that $g(a) < g(b)$.
Let $A,B \subseteq \R^d$ be two disjoint hyperrectangles with volumes $1/2a$ and $1/2b$. Define $P$ as the uniform distribution over $A$, and $P_*$ as
\begin{equation*}
P_* = a \mathbf{1}_{A} + b \mathbf{1}_B
\end{equation*}
Then the above inequality is violated. It follows that $g$ is constant and $f$ has the form $c \log p + c'$. Finally, the assumption on global minimum implies that $c \leq 0$.
\end{proof}

\subsection{Universal approximation theorems}

\begin{proof}[Proof of Proposition \ref{prop. RFM W2 approximation}]
By Brennier's theorem (\cite{brenier1991polar} and \cite[Theorem 2.12]{villani2003topics}), since both $\prob$ and $P_*$ have finite second moments and $\prob$ is absolutely continuous, there exists an $L^2(\prob,\R^d)$ function $G_*$ such that $G_*\#\prob = P_*$.
By Lemma \ref{lemma. universal approximation L2}, there exists a sequence $\{G_n\}_{n=1}^{\infty} \subset \HS$ that converges to $G_*$ in $L^2(\prob)$ norm.
Hence,
\begin{equation*}
\lim_{n\to\infty}W_2(G\#\prob,G_n\#\prob) \leq \lim_{n\to\infty}\|G-G_n\|_{L^2(\prob)} = 0
\end{equation*}
\end{proof}

The preceding proof uses the following lemma, which is a slight extension of the classical universal approximation theorem \cite{hornik1990universal,hornik1991approximation} to cases with possibly unbounded base distributions $\prob$.
Such extension is needed since in practice $\prob$ is usually set to be the unit Gaussian.

\begin{lemma}
\label{lemma. universal approximation L2}
Given either Assumption \ref{assume. ReLU} or \ref{assume. sigmoid}, for any $\prob\in\PS(\R^d)$ and any $k\in\mathbb{N}$, the space $\HS(\R^d,\R^k)$ is dense in $L^2(\prob,\R^k)$.
\end{lemma}

\begin{proof}
It suffices to consider the case with output dimension $k=1$.
Denote $\HS$ by $\HS_{\sigma}$ to emphasize the choice of the activation $\sigma$.

Given Assumption \ref{assume. ReLU}, the activation $\sigma$ is ReLU.
Define
\begin{equation*}
\sigma_*(x) = \int \sigma(x+1+\ep)-2\sigma(x+\ep)+\sigma(x-1+\ep) dh(\ep)
\end{equation*}
where $h$ is a continuous distribution supported in $[-0.1, 0.1]$.
Since ReLU is homogeneous, $\sigma_*$ can be expressed by (\ref{eq. RFM}) with some bounded parameter function $a$, and thus $\sigma_* \in \HS_{\sigma}(\R,\R)$.
Similarly, given Assumption \ref{assume. sigmoid}, the activation $\sigma$ is sigmoid.
Define
\begin{equation*}
\sigma_*(x) = \int \sigma(x+1+\ep)-\sigma(x-1+\ep) dh(\ep)
\end{equation*}
Since the parameter distribution $\rho$ is bounded below by some positive constant over the ball $B_3$, the function $\sigma_*$ can be expressed by (\ref{eq. RFM}) with bounded parameter function $a$, and thus $\sigma_* \in \HS_{\sigma}(\R,\R)$.

Hence, we always have $\sigma_* \in \HS_{\sigma}(\R,\R)$ and it is integrable ($\|\sigma_*\|_{L^1(\R)} < \infty$).
Define the subspace $\HS^c_{\sigma_*} \subseteq \HS_{\sigma_*}$ of functions $f_a$ whose parameter functions $a$ are compactly-supported.
Since $\HS^c_{\sigma_*} \subseteq \HS_{\sigma}$, it suffices to show that $\HS^c_{\sigma_*}$ is dense in $L^2(\prob)$.
Without loss of generality, we denote $\sigma_*$ by $\sigma$.
Since $\sigma$ is $L^1$, its Fourier transform $\hat{\sigma}$ is well-defined.
Since $\sigma$ is not constant zero, there exists a constant $c \neq 0$ such that $\hat{\sigma}(c) \neq 0$. Perform scaling if necessary, assume that $\hat{\sigma}(1)=1$.

It suffices to approximate the subspace $C_c^{\infty}(\R^d)$, which is dense in $L^2(\prob)$.
Fix any $f \in C_c^{\infty}$.
Its Fourier transform $\hat{f}$ is integrable.
Then, Step 1 of the proof of \cite[Theorem 3.1]{hornik1990universal} implies that,
\begin{equation*}
f(\x) = \int_{\R} \int_{\R^d} \sigma(\w\cdot\x+b) m(\w,b) ~d\w db, \quad m(\w,b) = \text{Re}\big[\hat{f}(\w) e^{-2\pi i b}\big]
\end{equation*}
For any $R>0$, define the signed distribution $m_R$ by
\begin{equation*}
m_R(\w,b) = \text{Re}\big[\hat{f}(\w) e^{-2\pi i b}\big] \mathbf{1}_{[-R,R]^{d+1}}(\w,b)
\end{equation*}
Define the function $f_R$ by
\begin{align*}
f_R(\x) &= \int \sigma(\w\cdot\x+b) dm_R(\w,b)\\
&= \int a_R(\w,b) \sigma(\w\cdot\x+b) d\rho(\w,b)
\end{align*}
If Assumption \ref{assume. sigmoid} holds, then define the parameter function $a_R$ by the compactly supported function
\begin{equation*}
a_R(\w,b) = \frac{m_R(\w,b)}{\rho(\w,b)}
\end{equation*}
Then,
\begin{align*}
\|f_R\|_{\HS} &\leq \|a_R\|_{L^2(\rho)} \leq \frac{\|f\|_{L^1(\R^d)}}{\sqrt{\min_{\w,b\in [-R,R]^{d+1}}\rho(\w,b)}} < \infty
\end{align*}
If Assumption \ref{assume. ReLU} holds, then define $a_R$ by the following function over the $l^1$ sphere $\{\|\w\|_1+|b| = 1\}$
\begin{equation*}
a_R(\w,b) = \frac{\int_0^{\infty} m_R(\lambda\w,\lambda b) d\lambda}{\rho(\w,b)}
\end{equation*}
Then,
\begin{align*}
\|f_R\|_{\HS} &\leq \|a_R\|_{L^2(\rho)} \leq \frac{\|f\|_{L^1(\R^d)}}{\sqrt{\min_{\|\w\|+|b|=1}\rho(\w,b)}} < \infty
\end{align*}
Thus, we always have $f_R \in \HS^c_{\sigma}(\R^d,\R)$.

The approximation error is bounded by
\begin{align*}
\|f-f_R\|_{L^2(\prob)}^2 &\leq \iint \iint \int h_{R}(\x,\w,b,\w',b') dP(\x) d\w d\w'db db'\\
h_{R}(\x,\w,b,\w',b') &= |\sigma(\w\cdot\x+b)| |\hat{f}(\w)| \mathbf{1}_{\R^d-[-R,R]^{d+1}}(\w,b)\\
& \cdot |\sigma(\w'\cdot\x+b')| |\hat{f}(\w')| \mathbf{1}_{\R^d-[-R,R]^{d+1}}(\w',b')
\end{align*}
Note that $0 \leq h_R \leq h_0$, and $h_R\to 0$ pointwise, and $h_0$ is integrable:
\begin{equation*}
\iint \iint \int h_0(\x,\w,b,\w',b') dP(\x) d\w d\w'db db' \leq \|\sigma\|_{L^1(\R)}^2 \|\hat{f}\|_{L^1(\R^d)}^2 < \infty
\end{equation*}
Hence, the dominated convergence theorem implies that $\lim_{R\to\infty} \|f-f_R\|_{L^2(\prob)} \to 0$, which completes the proof.
\end{proof}

\begin{proof}[Proof of Proposition \ref{prop. universal approximation potential}]
If Assumption \ref{assume. ReLU} holds, then \cite{sun2018RFM} implies that $\HS$ is dense in $C(K)$ with respect to the supremum norm.
Else, Assumption \ref{assume. sigmoid} holds, then \cite{hornik1990universal} implies that $\HS$ is dense in $C(K)$.
Thus, we can apply Proposition 2.1 of \cite{yang2022potential} to conclude the proof.
\end{proof}

\begin{proof}[Proof of Proposition \ref{prop. flow approximation}]
It suffices to approximate the compactly-supported distributions, which are dense with respect to the $W_2$ metric.
Fix any compactly-supported distribution $P$. Choose $R > 0$ such that the support is contained in $B_R$.
Assume that the base distribution $\prob$ has full support over $\R^d$. The case without full support will be discussed in the end.

By Brennier's theorem \cite[Theorem 2.12]{villani2003topics}, there exists a convex function $\psi$ over $\R^d$ such that $\nabla \psi \# \prob = P$.
For any $\delta > 0$, we can define the mollified function $\psi_{\delta}$
\begin{equation*}
\psi_{\delta}(\x) = h_{\delta} * \psi = \int h\Big(\frac{\mathbf{y}}{\delta}\Big) \psi(\x-\mathbf{y}) d\mathbf{y} = \int h\Big(\frac{\x-\mathbf{y}}{\delta}\Big) \psi(\mathbf{y}) d\mathbf{y}
\end{equation*}
where $h$ is a mollifier (i.e. $h$ is $C^{\infty}$, non-negative, supported in the unit ball, and $\int h = 1$).
Then, $\psi_{\delta}$ is $C^{\infty}$ and convex, and
\begin{equation*}
\lim_{\delta\to 0} W_2(P, \nabla\psi_{\delta}\#\prob) \leq \lim_{\delta\to 0} \|\nabla\psi-h_{\delta}*\nabla\psi\|_{L^2(\prob)} = 0
\end{equation*}
Thus, without loss of generality, we can assume that $\psi$ is $C^{\infty}$.

Define the time-dependent transport map for $\tau \in [0, 1]$
\begin{equation*}
T_{\tau}(\x) = (1-\tau) \x + \tau \nabla\psi(\x) = \nabla \Big[(1-\tau)\frac{\|\x\|^2}{2} + \tau\psi\Big]
\end{equation*}
which is $C^{\infty}$ in $\x$ and $\tau$,
and define the distributions $P_{\tau} = T_{\tau}\#\prob$, which are known as McCann interpolation.
Then, define the vector field
\begin{equation*}
V_{\tau} = \nabla \psi \circ T_{\tau}^{-1}
\end{equation*}
The Jacobian of $T_{\tau}$ is positive definite for $\tau \in [0,1)$
\begin{equation*}
\nabla_{\x} T_{\tau} = \text{Hess} \Big[(1-\tau)\frac{\|\x\|^2}{2} + \tau\psi\Big] \geq (1-\tau)I
\end{equation*}
so the inverse function theorem implies that the inverse $T^{-1}_{\tau}$ exists and is $C^{\infty}$ over $(\x,\tau) \in \R^d\times[0,1)$, with
\begin{equation*}
\nabla_{\x}T^{-1}_{\tau} \leq \frac{1}{1-\tau}I, \quad \partial_{\tau} T^{-1}_{\tau}(\x) = -\nabla_{\x}T^{-1}(\x) \cdot \partial_{\tau} T_{\tau}(T^{-1}_{\tau}(\x))
\end{equation*}
Since $\lim_{\delta\to 0^+} W_2(P_{1-\delta}, P) = 0$, we can replace the approximation target $P$ by the sequence $\{P_{1-1/n}\}$, restrict $\tau$ to the interval $[0, 1-1/n]$, and thus assume without loss of generality that $\sup \nabla_{\x}T^{-1}_{\tau} < \infty$.
It follows that $T_{\tau}^{-1}$ and $V_{\tau}$ are $C^{\infty}$ over $\R^d \times [0, 1]$.
By Picard-Lindel\"{o}f theorem, the ODE $\dot{\x}_{\tau} = V_{\tau}(\x_{\tau})$ has unique solution locally, and it is straightforward to check that $\x_t = T_t(\x_0)$ is exactly the solution.

For any $r > R+1$ and any $\ep \in (0, 1)$, the universal approximation theorem \cite{hornik1991approximation} implies that there exists a random feature function $V_{r,\ep} \in \HS(\R^{d+1},\R^d)$ such that
\begin{equation}
\label{eq. velocity C0 error}
\| V-V_{r,\ep} \|_{C^0(B_r\times[0,1])} < \ep
\end{equation}
Denote its flow map (\ref{eq. flow generator}) by $G_{\tau}$.
For any $\x_0 \in B_r$, let $\x_{\tau}, \y_{\tau}$ denote the solutions to the ODEs
\begin{equation*}
\dot{\x}_{\tau} = V(\x_{\tau},\tau), \quad \dot{\y}_{\tau} = V_{r,\ep}(\y_{\tau},\tau), \quad \x_0=\y_0
\end{equation*}
Recall that $V_{\tau}(\x_{\tau})$ is simply the vector $\x_1-\x_0$, where $\x_1 \in \sprt P \subseteq B_R$. Then, for any $\x \in B_r - B_{R+1}$ and $\tau \in [0,1]$,
\begin{equation*}
\frac{\x}{\|\x\|} \cdot V(\x,\tau) \leq -1, \quad \frac{\x}{\|\x\|} \cdot V_{r,\ep}(\x,\tau) \leq -1+\ep
\end{equation*}
and this remains true despite the surgeries we performed on $V$ (the mollification of $\psi$ and the restriction of $t$ to $[0, 1-1/n]$).
It follows that the solutions $\x_{\tau}, \y_{\tau}$ are contained in $B_r$.
By Theorem 2.8 of \cite{teschl2012ODE},
\begin{equation*}
\|\x_1-\y_1\| \leq \| V-V_{r,\ep} \|_{C^0(B_r\times[0,1])} \frac{e^{\|\nabla V\|_{C^0(B_r\times[0,1])}}-1}{\|\nabla V\|_{C^0(B_r\times[0,1])}} < \ep \exp \|\nabla V\|_{C^0(B_r\times[0,1])}
\end{equation*}
Thus
\begin{equation*}
W_2\big(T_1\#\prob|_{B_r}, G_1\#\prob|_{B_r}\big) < \ep \exp \|\nabla V\|_{C^0(B_r\times[0,1])}
\end{equation*}

Meanwhile, for any $m>0$, define the random feature function $g_m \in \HS$
\begin{equation*}
g_m(\x,\tau) = m \int \mathbf{n} ~\sigma\big(m(\mathbf{n}\cdot\x-(r+1))\big) d(h_{1/m}*U\mathbb{S}_{d-1})(\mathbf{n})
\end{equation*}
where $U\mathbb{S}_{d-1}$ is the uniform distribution over the unit sphere and $h_{1/m}$ is a mollifier.
Since $\sigma$ is sigmoid, as $m\to\infty$, we have $\|g_m(\x)\| \to 0$ uniformly over $B_r$ while $\|g_m(\x)\| \to \infty$ uniformly over $\R^d-B_{r+2}$.
Replace $V_{r,\ep}$ by $V_{r,\ep} - g_m$ with $m$ sufficiently large enough such that (\ref{eq. velocity C0 error}) continues to hold, while for all $\x \in \R^d-B_{r+2}$ and $t\in[0,1]$
\begin{equation*}
\frac{\x}{\|\x\|}V_{r,\ep}(\x,t) < 0
\end{equation*}
It follows that $\|G_1(\x_0)\| \leq \|\x_0\|$ for all $\x_0 \in \R^d-B_{r+2}$.
Thus
\begin{equation*}
W_2\big(T_1\#\prob|_{\R^d-B_r}, G_1\#\prob|_{\R^d-B_r}\big) < \int_{\R^d-B_{r+2}} (\|\x\|+r+2)^2 d\prob(\x)
\end{equation*}

Hence,
\begin{align*}
&\quad \lim_{r\to\infty} \lim_{\ep\to 0} W_2(P, G_1\#\prob)\\
&\leq \lim_{r\to\infty} \lim_{\ep\to 0} W_2\big(T_1\#\prob|_{B_r}, G_1\#\prob|_{B_r}\big) + W_2\big(T_1\#\prob|_{\R^d-B_r}, G_1\#\prob|_{\R^d-B_r}\big)\\
&\leq \lim_{r\to\infty} \lim_{\ep\to 0} \ep\exp\|\nabla V\|_{C^0(B_r\times[0,1])} + \int_{\R^d-B_{r+2}} (\|\x\|+r+2)^2 d\prob(\x)\\
&\leq \lim_{r\to\infty} \int_{\R^d-B_{r+2}} (\|\x\|+r+2)^2 d\prob(\x)\\
&\leq 0
\end{align*}
We conclude that $P$ is a limit point of $\mathcal{G}\#\prob$.

Finally, consider the general case when the support of the base distribution $\prob$ is not necessarily $\R^d$.
For any $\ep \in (0, 1)$, define $\prob_{\ep} = (1-\ep)\prob + \ep\N$, where $\N$ is the unit Gaussian distribution.
Choose a velocity field $V_{\ep} \in \HS(\R^{d+1},\R^d)$ such that its flow $G_{\tau}$ satisfies $W_2(P, G_1\#\prob_{\ep}) < \ep$.
Then
\begin{align*}
W_2(P, G_1\#\prob) &\leq W_2(P, G_1\#\prob_{\ep}) + W_2(G_1\#\prob_{\ep}, G_1\#\prob)\\
&< \ep + W_2(G_1\#\ep\N, G_1\#\ep\prob)\\
&\leq \ep + W_2(\ep G_1\#\N, \ep\delta_{\mathbf{0}}) + W_2(\ep\delta_{\textbf{0}}, \ep G_1\#\prob)\\
&\leq \ep + \ep \Big(\sqrt{\int\|\x\|^2 d\N(\x)} + \sqrt{\int\|\x\|^2 d\prob(\x)}\Big)
\end{align*}
Taking $\ep\to 0$ completes the proof.
\end{proof}

\subsection{Generalization error}

\begin{proof}[Proof of Proposition \ref{prop. bias potential Tikhonov}]

This proof is a slight modification of the proof of Proposition \ref{prop. bias potential Ivanov} (Proposition 3.9 of \cite{yang2022potential}).
By equation (17) of \cite{yang2022potential}, with probability $1-\delta$ over the sampling of $P_*^{(n)}$,
\begin{equation}
\label{sampling gap RKHS bound}
|L(a)-L^{(n)}(a)| \leq \frac{4\sqrt{2\log 2d} + \sqrt{2\log (2/\delta)}}{\sqrt{n}} \|a\|_{L^2(\rho)}
\end{equation}
Since the regularized loss is strongly convex in $a$, the minimizer $a_{\lambda}^{(n)}$ exists and is unique.
Then
\begin{align*}
L(a^{(n)}_{\lambda}) &\leq L^{(n)}(a^{(n)}_{\lambda}) + \frac{4\sqrt{2\log 2d} + \sqrt{2\log (2/\delta)}}{\sqrt{n}} \|a^{(n)}_{\lambda}\|_{L^2(\rho)}\\
&\leq L^{(n)}(a^{(n)}_{\lambda}) + \frac{\lambda}{\sqrt{n}} \|a^{(n)}_{\lambda}\|_{L^2(\rho)}\\
&\leq L^{(n)}(a_*) + \frac{\lambda}{\sqrt{n}} \|a_*\|_{L^2(\rho)}\\
&\leq L(a_*) + \Big(\lambda + \frac{4\sqrt{2\log 2d} + \sqrt{2\log (2/\delta)}}{\sqrt{n}}\Big) \|a_*\|_{L^2(\rho)}
\end{align*}
where the first and last inequalities follow from (\ref{sampling gap RKHS bound}) and the third inequality follows from the fact that $a_{\lambda}^{(n)}$ is the global minimizer.

Hence,
\begin{equation*}
\KL(P_*\|P^{(n)}_{\lambda}) = L(a^{(n)}_{\lambda}) - L(a_*) \leq \frac{2\lambda \|a_*\|_{L^2(\rho)}}{\sqrt{n}}
\end{equation*}
\end{proof}

\subsection{Training and loss landscape}

\begin{proof}[Proof of Proposition \ref{prop. compare critical}]
Since $G_{\theta}$ is Lipschitz for any $\theta\in\Theta$, the set of generated distributions $\PS_{\Theta} \subseteq \PS_2(\R^d)$, so the loss $L(\theta)=L(G_{\theta}\#\prob)$ is well-defined.
Denote the evaluation of any linear operator $l$ over the Hilbert space $\Theta$ by $\lb l, \theta\rb$.
By assumption, for any $\theta_0, \theta_1 \in \Theta$, the velocity field
\begin{equation*}
\big\lb \nabla_{\theta}G_{\theta_0}, \theta_1 \big\rb \in L^2(\prob,\R^k)
\end{equation*}
Thus, we can define a linear operator $V_{\theta}: \Theta\to L^2(G_{\theta}\#\prob,\R^d)$ by
\begin{align*}
\int f(\x) \cdot \big\lb V_{\theta_0}, \theta_1 \big\rb(\x)  ~d(G_{\theta}\#\prob)(\x) := \int f(G(\x)) \cdot \big\lb \nabla_{\theta}G_{\theta_0}, \theta_1 \big\rb(\x) ~d\prob(\x)
\end{align*}
where $f \in L^2(G_{\theta_0}\#\prob,\R^d)$ is any test function.
This operator is continuous since $\|V_{\theta}\|_{op} \leq \|\nabla_{\theta}G_{\theta}\|_{op} < \infty$.

For any perturbation $h \in \Theta$,
\begin{align*}
\frac{d}{d\ep} L(\theta+\ep h) \big|_{\ep=0} &= \lim_{\ep\to 0} \frac{1}{\ep} \big( L(G_{\theta+\ep h}\#\prob) - L(G_{\theta}\#\prob) \big)\\
&= \lim_{\ep\to 0} \frac{1}{\ep} \int \frac{\delta L}{\delta P}\big|_{G_{\theta+\ep h}\#\prob} d(G_{\theta+\ep h}\#\prob-G_{\theta}\#\prob)\\
&= \lim_{\ep\to 0} \frac{1}{\ep} \int \frac{\delta L}{\delta P}\big(G_{\theta+\ep h}(\x)\big) - \frac{\delta L}{\delta P}\big(G_{\theta}(\x)\big) d\prob(\x)\\
&= \lim_{\ep\to 0} \frac{1}{\ep} \int \nabla \frac{\delta L}{\delta P}\big(G_{\theta}(\x)\big) \cdot \big(G_{\theta+\ep h}(\x) - G_{\theta}(\x)\big) d\prob(\x)\\
&= \int \nabla \frac{\delta L}{\delta P}\big(G_{\theta}(\x)\big) \cdot \big\lb \nabla_{\theta} G_{\theta}, h\big\rb (\x) ~d\prob(\x)\\
&= \int \nabla \frac{\delta L}{\delta P}(\x) \cdot \big\lb V_{\theta}, h\big\rb (\x) ~d(G_{\theta}\#\prob)(\x)
\end{align*}
Hence, the loss $L$ is differentiable in $\theta$ and the derivative is the following operator
\begin{equation*}
\lb \nabla_{\theta} L(\theta), \cdot \rb = \int \nabla \frac{\delta L}{\delta P}(\x) \cdot \big\lb V_{\theta}, \cdot \big\rb (\x) ~d(G_{\theta}\#\prob)(\x)
\end{equation*}
For any $\theta\in\Theta$ such that $G_{\theta}\#\prob \in CP$,
\begin{equation*}
\lb \nabla_{\theta} L(\theta), \cdot \rb = \int \mathbf{0} \cdot \big\lb V_{\theta}, \cdot \big\rb (\x) ~d(G_{\theta}\#\prob)(\x) = 0
\end{equation*}
Thus, $CP \cap \PS_{\Theta} \subseteq CP_{\Theta}$.

For the second claim, we first verify that the random feature functions $G_{\theta}$ satisfy the smoothness assumptions:
For any $\a\in L^2(\rho,\R^d)$ and $\x,\x'\in\R^k$,
\begin{align*}
\|G_{\a}(\x)-G_{\a}(\x')\| &\leq \int \|\a(\w,b)\| \big| \sigma(\w\cdot\x+b) - \sigma(\w\cdot\x'+b) \big| d\rho(\w,b)\\
&\leq \|\a\|_{L^2(\rho)} \|\sigma\|_{\Lip} \sqrt{\int \|\w\|^2 d\rho(\w,b)} \|\x-\x'\|\\
&\leq C\|\a\|_{L^2(\rho)} \|\x-\x'\|
\end{align*}
where $C=\sqrt{d}$ if Assumption \ref{assume. ReLU} holds and $C=1$ if Assumption \ref{assume. sigmoid} holds. Thus, $G_{\a}$ is Lipschitz.
Meanwhile, since the parametrization $\a\mapsto G_{\a}$ is linear
\begin{align*}
\lb \nabla_{\a} G_{\a}(\x), \a'\rb = G_{\a'}
\end{align*}
Then, for any $\a,\a'\in L^2(\rho,\R^d)$ and any $\x\in\R^k$
\begin{align*}
\|\lb \nabla_{\a} G_{\a}(\x), \a'\rb\| &= \|G_{\a'}(\x)\| \leq \|\a'\|_{L^2(\rho)} \Big(|\sigma(0)| + \sqrt{\int \|\w\|^2+|b|^2 d\rho(\w,b)} \sqrt{\|\x\|^2+1}\Big)\\
&\leq C \|\a'\|_{L^2(\rho)} \sqrt{\|\x\|^2+1}\\
\|\lb \nabla_{\a} G_{\a}(\cdot), \a'\rb\|_{L^2(\prob)} &= \|G_{\a'}\|_{L^2(\prob)} \leq C \Big(1+\sqrt{\int \|\x\|^2 d\prob(\x)}\Big) \|\a'\|_{L^2(\rho)}
\end{align*}
where $C=\sqrt{d+1}$ if Assumption \ref{assume. ReLU} holds and $C=1$ if Assumption \ref{assume. sigmoid} holds.
Thus, $G_{\a}(\x)$ is $C^1$ in $\a$ for any $\x$, and $\nabla_{\a} G_{\a}$ is a continuous operator $L^2(\rho,\R^d)\to L^2(\prob,\R^d)$.

Consider any $\a\in L^2(\rho,\R^d)$ such that $G_{\a}\#\prob \in CP_{\Theta}$.
Then, for any $\a' \in L^2(\rho,\R^d)$,
\begin{align*}
0 &= \int \nabla \frac{\delta L}{\delta P}\big(G_{\a}(\x)\big) \cdot \big\lb \nabla_{\a} G_{\a}, \a'\big\rb (\x) ~d\prob(\x)\\
&= \int \a'(\w,b) \cdot \int \nabla \frac{\delta L}{\delta P}\big(G_{\a}(\x)\big) \sigma(\w\cdot\x+b) ~d\prob(\x) d\rho(\w,b)
\end{align*}
If $\sigma$ is Lipschitz, then the integrand is continuous in $(\w,b)$, so for any $(\w,b) \in \sprt\rho$, we have 
\begin{equation*}
\int \nabla \frac{\delta L}{\delta P}\big(G_{\a}(\x)\big) \sigma(\w\cdot\x+b) ~d\prob(\x) = \mathbf{0}
\end{equation*}
Thus, if either Assumption \ref{assume. ReLU} or \ref{assume. sigmoid} holds, the equality holds for all $(\w,b)\in\R^{k+1}$.
It follows that, by universal approximation theorem \cite{hornik1990universal}, $\nabla \delta_P L\big(G_{\a}(\x)\big) = \mathbf{0}$ for $\prob$ almost all $\x$.
Or equivalently, $\delta_P L(\x) = \mathbf{0}$ for $G_{\a}\#\prob$ almost all $\x$.
Hence, the reverse inclusion $CP_{\Theta} \subseteq CP \cap \PS_{\Theta}$ holds.
\end{proof}

\begin{proof}[Proof of Proposition \ref{prop. W2 GD landscape}]
For any $G\in L^2(\prob,\R^d)$, let $\pi$ be the optimal transport plan between $G\#\prob$ and $P_*$.
Define the function
\begin{equation}
\label{eq. OT conditional mean}
m(x) = \int x' ~d\pi\big(x'\big|x\big)
\end{equation}
Then, $G$ is a stationary point if and only if $G = m\circ G$.
If $G\#\prob$ is absolutely continuous, then $\pi$ is concentrated on the graph of $\nabla \phi$ for some convex function $\phi$ by Brennier's theorem \cite{villani2003topics}.
Then, $m=\nabla \phi$, and $G\#\prob = (m\circ G)\#\prob = \nabla\phi\#(G\#\prob) = P_*$, so $G$ is a global minimum.

It follows that if $G$ is a stationary point but not a global minimum, then $P=G\#\prob$ must be singular.
Since the cumulant function of $P$ is non-decreasing, there are at most countably many jumps.
Thus, $P$ can be expressed as $P=P_{ac} + P_{sg}$ where $P_{ac}$ is a density function (the Lebesgue derivative of the continuous part of the cumulant) and $P_{sg}$ is a countable sum of point masses at the jumps $x_i$.
Since $P_{sg}$ is nonzero, there exists $x_* \in \sprt P$ such that $P(\{x_*\})>0$.
Let $S_+, S_-$ be a disjoint partition of $G^{-1}(\{x_*\})$ such that $\prob(S_+)=\prob(S_-) = P(x_*)/2$.
Since $P_*$ is absolutely continuous, Brennier's theorem implies that the transport plan $\pi(x_0,x_1)$ has the conditional distribution $\pi(\cdot|x_1)=\delta_{\nabla\psi(x_1)}$ for some convex potential $\psi$.
Since $\nabla\psi$ is non-decreasing, there exists an interval $[x_-,x_+]$ such that $\nabla\psi$ maps $[x_-,x_+]$ to $\{x_*\}$ and $P_*([x_-, x_+])=P(\{x_*\})$.
Choose $x_o \in (x_-,x_+)$ such that $P_*([x_-, x_o])=P_*((x_o, x_+])=P(\{x_*\})/2$.
Define the means
\begin{equation*}
m_- = \frac{2}{P(\{x_*\})}\int_{x_-}^{x_o} x~dP_*(x), \quad m_+ = \frac{2}{P(\{x_*\})}\int_{x_o}^{x_+} x~dP_*(x)
\end{equation*}
Then, $x_- < m_- < x_o < m_+ < x_+$ and $x_*=(m_-+m_+)/2$.
For any $\ep$, define the function
\begin{equation*}
h = \ep ((x_+-m_+)\mathbf{1}_{S_+} + (x_--m_-) \mathbf{1}_{S_-})
\end{equation*}
and the map
\begin{equation*}
\forall x \in \sprt P_*, \quad F(x) = \begin{cases}(1-\ep)x_*+\ep m_- \text{ if } x\in [x_-,x_o]\\
(1-\ep)x_*+\ep m_+ \text{ if } x\in (x_o,x_+]\\
\nabla\psi(x) \text{ else}\end{cases}
\end{equation*}
Then, for any $\ep \in (0,1)$,
\begin{align*}
&\quad L(G)-L(G+h)\\
&\geq \frac{1}{2} \int |x-\nabla\psi(x)|^2 - |x-F(x)|^2 dP_*(x)\\
&= \frac{1}{2} \int_{x_-}^{x_o} |x-x_*|^2 - |x-(1-\ep)x_*-\ep m_-|^2 dP_*(x)\\
&\quad + \frac{1}{2} \int_{x_o}^{x_+} |x-x_*|^2 - |x-(1-\ep)x_*-\ep m_+|^2 dP_*(x)\\
&= \frac{\ep^2}{2} \int_{x_-}^{x_o} |x-x_*|^2 - |x-m_-|^2 dP_*(x) + \frac{\ep^2}{2} \int_{x_o}^{x_+} |x-x_*|^2 - |x-m_+|^2 dP_*(x)\\
&>0
\end{align*}
Thus, $G$ is a generalized saddle point.

For any initialization $G_0 \in L^2(\prob)$, let $G_t$ be a trajectory defined by the dynamics (\ref{eq. generalized W2 GD}).
Let $m_t$ be the function (\ref{eq. OT conditional mean}) defined by the optimal transport plan $\pi_t$ between $P_t=G_t\#\prob$ and $P_*$.
Then, the dynamics (\ref{eq. generalized W2 GD}) can be written as
\begin{equation*}
\frac{d}{dt} G_t = m_t \circ G_t - G_t
\end{equation*}
By cyclic monotonicity \cite{villani2003topics}, the coupling $\pi_t$ must be monotone: for any $(x_0,x_1),(x_0',x_1') \in \sprt \pi_t$,
\begin{equation*}
(x_0-x_0')(x_1-x_1') \geq 0
\end{equation*}
Thus, the trajectory $G_t$ is order-preserving
\begin{equation*}
G_0(z) < G(z') \to G_t(z) < G_t(z')
\end{equation*}
It follows that the cumulant function
\begin{equation*}
F(z) := P_t((-\infty,G_t(z)]) = \int \mathbf{1}_{G_t(z')\leq G_t(z)} d\prob(z) = \int \mathbf{1}_{G_0(z')\leq G_0(z)} d\prob(z)
\end{equation*}
is constant in $t$.

Since $P_*$ is absolutely continuous, its cumulant function $F_*(x) = P_*((-\infty,x])$ is a continuous, non-decreasing function from $\R$ to $[0,1]$.
For any $p \in [0,1]$, define its inverse by
\begin{equation*}
F_*^{-1}(p) = \argmin_x \{F_*(x)\leq p\}
\end{equation*}
It follows that for any $z$, the support of $\pi_t(\cdot|G_t(z))$ must lie in the closed interval
\begin{align*}
I(z) &= \Big[F_*^{-1}\big(P_t((-\infty,G_t(z)))\big), ~F_*^{-1}\big(P_t((-\infty,G_t(z)])\big)\Big]\\
&= \Big[F_*^{-1}\big(\lim_{\ep\to 0^+}F(z-\ep)\big), ~F_*^{-1}(F(z))\Big]
\end{align*}
It follows that for any $z \in \sprt\prob$, the conditional distribution $\pi(\cdot|G(z))$ is exactly $P_*$ conditioned on the subset $I(z)$.
Hence, the function
\begin{equation*}
m_t \circ G_t(z) = \begin{cases}\frac{\int_{I(z)} xdP_*(x)}{P_*(I(z))} \text{ if } |I(z)|>0\\
F_*^{-1}(F(z)) \text{ else}\end{cases}
\end{equation*}
is constant in $t$.
It follows that the trajectory $G_t$ satisfies
\begin{equation*}
\frac{d}{dt} G_t = m_0\circ G_0 - G_t
\end{equation*}
and thus
\begin{equation*}
G_t = e^{-t} G_0 + (1-e^{-t}) m_0\circ G_0
\end{equation*}

It follows that $P_t$ converges to $m_0 \# P_0$.
If $P_0 \in \PS_{ac}$, then as discussed above, $m_0$ is exactly the optimal transport map from $P_0$ to $P_*$, and thus we have global convergence.
Else, there exists some $x_0$ such that $P_o(\{x_0\}) > 0$, so $(m_0\# P_0)(\{m_0(x_0)\}) \geq P_0(\{x_0\}) > 0$, and thus $m_0\# P_0$ is not absolutely continuous. Since $P_* \in \PS_{ac}$, we have $m_0\#P_0 \neq P_*$.
\end{proof}

\bibliography{main}

\begin{thebibliography}{100}

\bibitem{ackley1985learning}
{\sc Ackley, D.~H., Hinton, G.~E., and Sejnowski, T.~J.}
\newblock A learning algorithm for boltzmann machines.
\newblock {\em Cognitive science 9}, 1 (1985), 147--169.

\bibitem{albergo2022building}
{\sc Albergo, M.~S., and Vanden-Eijnden, E.}
\newblock Building normalizing flows with stochastic interpolants.
\newblock {\em arXiv preprint arXiv:2209.15571\/} (2022).

\bibitem{ambrosio2008gradient}
{\sc Ambrosio, L., Gigli, N., and Savar{\'e}, G.}
\newblock {\em Gradient flows: in metric spaces and in the space of probability
  measures}.
\newblock Springer Science \& Business Media, 2008.

\bibitem{an2019ae}
{\sc An, D., Guo, Y., Lei, N., Luo, Z., Yau, S.-T., and Gu, X.}
\newblock {AE-OT}: a new generative model based on extended semi-discrete
  optimal transport.
\newblock {\em ICLR 2020\/} (2019).

\bibitem{an2020GAN}
{\sc An, D., Guo, Y., Zhang, M., Qi, X., Lei, N., and Gu, X.}
\newblock {AE-OT-GAN}: Training gans from data specific latent distribution.
\newblock In {\em European Conference on Computer Vision\/} (2020), Springer,
  pp.~548--564.

\bibitem{anderson1982reverse}
{\sc Anderson, B.~D.}
\newblock Reverse-time diffusion equation models.
\newblock {\em Stochastic Processes and their Applications 12}, 3 (1982),
  313--326.

\bibitem{arjovsky2017wasserstein}
{\sc Arjovsky, M., Chintala, S., and Bottou, L.}
\newblock Wasserstein {GAN}.
\newblock {\em arXiv preprint arXiv:1701.07875\/} (2017).

\bibitem{arora2017generalization}
{\sc Arora, S., Ge, R., Liang, Y., Ma, T., and Zhang, Y.}
\newblock Generalization and equilibrium in generative adversarial nets
  ({GAN}s).
\newblock {\em arXiv preprint arXiv:1703.00573\/} (2017).

\bibitem{arora2018GAN}
{\sc Arora, S., Risteski, A., and Zhang, Y.}
\newblock Do {GANs} learn the distribution? {S}ome theory and empirics.
\newblock In {\em International Conference on Learning Representations\/}
  (2018).

\bibitem{balaji2021understanding}
{\sc Balaji, Y., Sajedi, M., Kalibhat, N.~M., Ding, M., St{\"o}ger, D.,
  Soltanolkotabi, M., and Feizi, S.}
\newblock Understanding overparameterization in generative adversarial
  networks.
\newblock {\em arXiv preprint arXiv:2104.05605\/} (2021).

\bibitem{barron1991approximation}
{\sc Barron, A.~R., and Sheu, C.-H.}
\newblock Approximation of density functions by sequences of exponential
  families.
\newblock {\em The Annals of Statistics\/} (1991), 1347--1369.

\bibitem{bengio2021gflownet}
{\sc Bengio, Y., Deleu, T., Hu, E.~J., Lahlou, S., Tiwari, M., and Bengio, E.}
\newblock {GFlowNet} foundations.
\newblock {\em arXiv preprint arXiv:2111.09266\/} (2021).

\bibitem{bonati2019enhanced}
{\sc Bonati, L., Zhang, Y.-Y., and Parrinello, M.}
\newblock Neural networks-based variationally enhanced sampling.
\newblock In {\em Proceedings of the National Academy of Sciences\/} (2019),
  vol.~116, pp.~17641--17647.

\bibitem{brenier1991polar}
{\sc Brenier, Y.}
\newblock Polar factorization and monotone rearrangement of vector-valued
  functions.
\newblock {\em Communications on Pure and Applied Mathematics 44}, 4 (1991),
  375--417.

\bibitem{brock2018BigGAN}
{\sc Brock, A., Donahue, J., and Simonyan, K.}
\newblock Large scale {GAN} training for high fidelity natural image synthesis.
\newblock {\em arXiv preprint arXiv:1809.11096\/} (2018).

\bibitem{brown2020language}
{\sc Brown, T.~B., Mann, B., Ryder, N., Subbiah, M., Kaplan, J., Dhariwal, P.,
  Neelakantan, A., Shyam, P., Sastry, G., Askell, A., Agarwal, S.,
  Herbert-Voss, A., Krueger, G., Henighan, T., Child, R., Ramesh, A., Ziegler,
  D.~M., Wu, J., Winter, C., Hesse, C., Chen, M., Sigler, E., Litwin, M., Gray,
  S., Chess, B., Clark, J., Berner, C., McCandlish, S., Radford, A., Sutskever,
  I., and Amodei, D.}
\newblock Language models are few-shot learners, 2020.

\bibitem{cao2022learning}
{\sc Cao, Y., and Vanden-Eijnden, E.}
\newblock Learning optimal flows for non-equilibrium importance sampling.
\newblock {\em arXiv preprint arXiv:2206.09908\/} (2022).

\bibitem{caron2021emerging}
{\sc Caron, M., Touvron, H., Misra, I., J{\'e}gou, H., Mairal, J., Bojanowski,
  P., and Joulin, A.}
\newblock Emerging properties in self-supervised vision transformers.
\newblock In {\em Proceedings of the IEEE/CVF International Conference on
  Computer Vision\/} (2021), pp.~9650--9660.

\bibitem{chavdarova2018sgan}
{\sc Chavdarova, T., and Fleuret, F.}
\newblock {SGAN}: An alternative training of generative adversarial networks.
\newblock In {\em Proceedings of the IEEE Conference on Computer Vision and
  Pattern Recognition\/} (2018), pp.~9407--9415.

\bibitem{che2016mode}
{\sc Che, T., Li, Y., Jacob, A., Bengio, Y., and Li, W.}
\newblock Mode regularized generative adversarial networks.
\newblock {\em arXiv preprint arXiv:1612.02136\/} (2016).

\bibitem{chen2018neural}
{\sc Chen, R.~T., Rubanova, Y., Bettencourt, J., and Duvenaud, D.~K.}
\newblock Neural ordinary differential equations.
\newblock In {\em Advances in neural information processing systems\/} (2018),
  pp.~6571--6583.

\bibitem{chen2020simple}
{\sc Chen, T., Kornblith, S., Norouzi, M., and Hinton, G.}
\newblock A simple framework for contrastive learning of visual
  representations.
\newblock In {\em International conference on machine learning\/} (2020), PMLR,
  pp.~1597--1607.

\bibitem{chen2016infogan}
{\sc Chen, X., Duan, Y., Houthooft, R., Schulman, J., Sutskever, I., and
  Abbeel, P.}
\newblock {InfoGAN}: Interpretable representation learning by information
  maximizing generative adversarial nets.
\newblock {\em Advances in neural information processing systems 29\/} (2016).

\bibitem{chizat2018global}
{\sc Chizat, L., and Bach, F.}
\newblock On the global convergence of gradient descent for over-parameterized
  models using optimal transport.
\newblock In {\em Advances in neural information processing systems\/} (2018),
  pp.~3036--3046.

\bibitem{chowdhery2022palm}
{\sc Chowdhery, A., Narang, S., Devlin, J., Bosma, M., Mishra, G., Roberts, A.,
  Barham, P., Chung, H.~W., Sutton, C., Gehrmann, S., et~al.}
\newblock Palm: Scaling language modeling with pathways.
\newblock {\em arXiv preprint arXiv:2204.02311\/} (2022).

\bibitem{cucker2002mathematical}
{\sc Cucker, F., and Smale, S.}
\newblock On the mathematical foundations of learning.
\newblock {\em Bulletin of the American mathematical society 39}, 1 (2002),
  1--49.

\bibitem{denton2015LAPGAN}
{\sc Denton, E., Chintala, S., Arthur, S., and Fergus, R.}
\newblock Deep generative image models using a {L}aplacian pyramid of
  adversarial networks.
\newblock In {\em Advances in neural information processing systems\/} (2015),
  pp.~1486--1494.

\bibitem{devlin2018bert}
{\sc Devlin, J., Chang, M.-W., Lee, K., and Toutanova, K.}
\newblock {BERT}: Pre-training of deep bidirectional transformers for language
  understanding.
\newblock {\em arXiv preprint arXiv:1810.04805\/} (2018).

\bibitem{dinh2016density}
{\sc Dinh, L., Sohl-Dickstein, J., and Bengio, S.}
\newblock Density estimation using real {NVP}.
\newblock {\em arXiv preprint arXiv:1605.08803\/} (2016).

\bibitem{e2017flow}
{\sc E, W.}
\newblock A proposal on machine learning via dynamical systems.
\newblock {\em Communications in Mathematics and Statistics 5}, 1 (2017),
  1--11.

\bibitem{e2019residual}
{\sc E, W., Ma, C., and Wang, Q.}
\newblock A priori estimates of the population risk for residual networks.
\newblock {\em arXiv preprint arXiv:1903.02154 1}, 7 (2019).

\bibitem{e2020NNML}
{\sc E, W., Ma, C., Wojtowytsch, S., and Wu, L.}
\newblock Towards a mathematical understanding of neural network-based machine
  learning: what we know and what we don't, 2020.

\bibitem{e2018priori}
{\sc E, W., Ma, C., and Wu, L.}
\newblock A priori estimates for two-layer neural networks.
\newblock {\em arXiv preprint arXiv:1810.06397\/} (2018).

\bibitem{e2019barron}
{\sc E, W., Ma, C., and Wu, L.}
\newblock Barron spaces and the compositional function spaces for neural
  network models.
\newblock {\em arXiv preprint arXiv:1906.08039\/} (2019).

\bibitem{e2019machine}
{\sc E, W., Ma, C., and Wu, L.}
\newblock Machine learning from a continuous viewpoint.
\newblock {\em arXiv preprint arXiv:1912.12777\/} (2019).

\bibitem{e2019min}
{\sc E, W., Ma, C., and Wu, L.}
\newblock On the generalization properties of minimum-norm solutions for
  over-parameterized neural network models.
\newblock {\em arXiv preprint arXiv:1912.06987\/} (2019).

\bibitem{e2020kolmogorov}
{\sc E, W., and Wojtowytsch, S.}
\newblock Kolmogorov width decay and poor approximators in machine learning:
  Shallow neural networks, random feature models and neural tangent kernels.
\newblock {\em arXiv preprint arXiv:2005.10807\/} (2020).

\bibitem{feizi2020LQG}
{\sc {Feizi}, S., {Farnia}, F., {Ginart}, T., and {Tse}, D.}
\newblock Understanding {GANs} in the {LQG} setting: Formulation,
  generalization and stability.
\newblock {\em IEEE Journal on Selected Areas in Information Theory 1}, 1
  (2020), 304--311.

\bibitem{goodfellow2014generative}
{\sc Goodfellow, I., Pouget-Abadie, J., Mirza, M., Xu, B., Warde-Farley, D.,
  Ozair, S., Courville, A., and Bengio, Y.}
\newblock Generative adversarial nets.
\newblock In {\em Advances in neural information processing systems\/} (2014),
  pp.~2672--2680.

\bibitem{grathwohl2018ffjord}
{\sc Grathwohl, W., Chen, R.~T., Bettencourt, J., Sutskever, I., and Duvenaud,
  D.}
\newblock Ffjord: Free-form continuous dynamics for scalable reversible
  generative models.
\newblock {\em arXiv preprint arXiv:1810.01367\/} (2018).

\bibitem{gretton2012kernel}
{\sc Gretton, A., Borgwardt, K.~M., Rasch, M.~J., Sch{\"o}lkopf, B., and Smola,
  A.}
\newblock A kernel two-sample test.
\newblock {\em The Journal of Machine Learning Research 13}, 1 (2012),
  723--773.

\bibitem{grill2020bootstrap}
{\sc Grill, J.-B., Strub, F., Altch{\'e}, F., Tallec, C., Richemond, P.,
  Buchatskaya, E., Doersch, C., Avila~Pires, B., Guo, Z., Gheshlaghi~Azar, M.,
  et~al.}
\newblock Bootstrap your own latent-a new approach to self-supervised learning.
\newblock {\em Advances in neural information processing systems 33\/} (2020),
  21271--21284.

\bibitem{gui2021review}
{\sc Gui, J., Sun, Z., Wen, Y., Tao, D., and Ye, J.}
\newblock A review on generative adversarial networks: Algorithms, theory, and
  applications.
\newblock {\em IEEE Transactions on Knowledge and Data Engineering\/} (2021).

\bibitem{gulrajani2017improved}
{\sc Gulrajani, I., Ahmed, F., Arjovsky, M., Dumoulin, V., and Courville, A.}
\newblock Improved training of wasserstein {GAN}s, 2017.

\bibitem{gulrajani2020towards}
{\sc Gulrajani, I., Raffel, C., and Metz, L.}
\newblock Towards {GAN} benchmarks which require generalization.
\newblock {\em arXiv preprint arXiv:2001.03653\/} (2020).

\bibitem{han2021class}
{\sc Han, J., Hu, R., and Long, J.}
\newblock A class of dimensionality-free metrics for the convergence of
  empirical measures.
\newblock {\em arXiv preprint arXiv:2104.12036\/} (2021).

\bibitem{heusel2017gans}
{\sc Heusel, M., Ramsauer, H., Unterthiner, T., Nessler, B., and Hochreiter,
  S.}
\newblock {GANs} trained by a two time-scale update rule converge to a local
  {N}ash equilibrium.
\newblock In {\em Advances in neural information processing systems\/} (2017),
  pp.~6626--6637.

\bibitem{higgins2016beta}
{\sc Higgins, I., Matthey, L., Pal, A., Burgess, C., Glorot, X., Botvinick, M.,
  Mohamed, S., and Lerchner, A.}
\newblock beta-{VAE}: Learning basic visual concepts with a constrained
  variational framework.

\bibitem{ho2020denoising}
{\sc Ho, J., Jain, A., and Abbeel, P.}
\newblock Denoising diffusion probabilistic models.
\newblock {\em Advances in Neural Information Processing Systems 33\/} (2020),
  6840--6851.

\bibitem{hoang2018mgan}
{\sc Hoang, Q., Nguyen, T.~D., Le, T., and Phung, D.}
\newblock Mgan: Training generative adversarial nets with multiple generators.
\newblock In {\em International conference on learning representations\/}
  (2018).

\bibitem{hornik1991approximation}
{\sc Hornik, K.}
\newblock Approximation capabilities of multilayer feedforward networks.
\newblock {\em Neural Networks 4}, 2 (1991), 251–257.

\bibitem{hornik1990universal}
{\sc Hornik, K., Stinchcombe, M., and White, H.}
\newblock Universal approximation of an unknown mapping and its derivatives
  using multilayer feedforward networks.
\newblock {\em Neural networks 3}, 5 (1990), 551--560.

\bibitem{hou2017deep}
{\sc Hou, X., Shen, L., Sun, K., and Qiu, G.}
\newblock Deep feature consistent variational autoencoder.
\newblock In {\em 2017 IEEE winter conference on applications of computer
  vision (WACV)\/} (2017), IEEE, pp.~1133--1141.

\bibitem{hu2017unifying}
{\sc Hu, Z., Yang, Z., Salakhutdinov, R., and Xing, E.~P.}
\newblock On unifying deep generative models.
\newblock {\em arXiv preprint arXiv:1706.00550\/} (2017).

\bibitem{huang2018neural}
{\sc Huang, C.-W., Krueger, D., Lacoste, A., and Courville, A.}
\newblock Neural autoregressive flows.
\newblock {\em arXiv preprint arXiv:1804.00779\/} (2018).

\bibitem{hyvarinen2005estimation}
{\sc Hyv{\"a}rinen, A., and Dayan, P.}
\newblock Estimation of non-normalized statistical models by score matching.
\newblock {\em Journal of Machine Learning Research 6}, 4 (2005).

\bibitem{ivanov1976theory}
{\sc Ivanov, A.}
\newblock {\em The theory of approximate methods and their applications to the
  numerical solution of singular integral equations}, vol.~2.
\newblock Springer Science \& Business Media, 1976.

\bibitem{jiang2021transgan}
{\sc Jiang, Y., Chang, S., and Wang, Z.}
\newblock {TransGAN}: Two pure transformers can make one strong {GAN}, and that
  can scale up.
\newblock {\em Advances in Neural Information Processing Systems 34\/} (2021),
  14745--14758.

\bibitem{johnson2016perceptual}
{\sc Johnson, J., Alahi, A., and Fei-Fei, L.}
\newblock Perceptual losses for real-time style transfer and super-resolution.
\newblock In {\em European conference on computer vision\/} (2016), Springer,
  pp.~694--711.

\bibitem{kantorovich1958W1}
{\sc Kantorovich, L., and Rubinstein, G.~S.}
\newblock On a space of totally additive functions.
\newblock {\em Vestnik Leningrad. Univ 13\/} (1958), 52--59.

\bibitem{kantorovich1960mathematical}
{\sc Kantorovich, L.~V.}
\newblock Mathematical methods of organizing and planning production.
\newblock {\em Management science 6}, 4 (1960), 366--422.

\bibitem{karras2019style}
{\sc Karras, T., Laine, S., and Aila, T.}
\newblock A style-based generator architecture for generative adversarial
  networks.
\newblock In {\em Proceedings of the IEEE/CVF conference on computer vision and
  pattern recognition\/} (2019), pp.~4401--4410.

\bibitem{kingma2016improved}
{\sc Kingma, D.~P., Salimans, T., Jozefowicz, R., Chen, X., Sutskever, I., and
  Welling, M.}
\newblock Improved variational inference with inverse autoregressive flow.
\newblock In {\em Advances in neural information processing systems\/} (2016),
  pp.~4743--4751.

\bibitem{kingma2013auto}
{\sc Kingma, D.~P., and Welling, M.}
\newblock Auto-encoding variational bayes.
\newblock {\em arXiv preprint arXiv:1312.6114\/} (2013).

\bibitem{kobyzev2020normalizing}
{\sc Kobyzev, I., Prince, S.~J., and Brubaker, M.~A.}
\newblock Normalizing flows: An introduction and review of current methods.
\newblock {\em IEEE transactions on pattern analysis and machine intelligence
  43}, 11 (2020), 3964--3979.

\bibitem{kodali2017convergence}
{\sc Kodali, N., Abernethy, J., Hays, J., and Kira, Z.}
\newblock On convergence and stability of {GANs}.
\newblock {\em arXiv preprint arXiv:1705.07215\/} (2017).

\bibitem{kontorovich2014concentration}
{\sc Kontorovich, A.}
\newblock Concentration in unbounded metric spaces and algorithmic stability.
\newblock In {\em International Conference on Machine Learning\/} (2014), PMLR,
  pp.~28--36.

\bibitem{larochelle2011neural}
{\sc Larochelle, H., and Murray, I.}
\newblock The neural autoregressive distribution estimator.
\newblock In {\em Proceedings of the Fourteenth International Conference on
  Artificial Intelligence and Statistics\/} (2011), pp.~29--37.

\bibitem{lei2020sgd}
{\sc Lei, Q., Lee, J.~D., Dimakis, A.~G., and Daskalakis, C.}
\newblock Sgd learns one-layer networks in wgans, 2020.

\bibitem{li2018renormalization}
{\sc Li, S.-H., and Wang, L.}
\newblock Neural network renormalization group.
\newblock {\em Physical review letters 121}, 26 (2018), 260601.

\bibitem{dou2020making}
{\sc Li, Y., and Dou, Z.}
\newblock Making method of moments great again?--how can {GANs} learn
  distributions.
\newblock {\em arXiv preprint arXiv:2003.04033\/} (2020).

\bibitem{li2015GMMM}
{\sc Li, Y., Swersky, K., and Zemel, R.}
\newblock Generative moment matching networks.
\newblock In {\em International Conference on Machine Learning\/} (2015),
  pp.~1718--1727.

\bibitem{lin2020gradient}
{\sc Lin, T., Jin, C., and Jordan, M.}
\newblock On gradient descent ascent for nonconvex-concave minimax problems.
\newblock In {\em International Conference on Machine Learning\/} (2020), PMLR,
  pp.~6083--6093.

\bibitem{ling2021editgan}
{\sc Ling, H., Kreis, K., Li, D., Kim, S.~W., Torralba, A., and Fidler, S.}
\newblock {EditGAN}: High-precision semantic image editing.
\newblock {\em Advances in Neural Information Processing Systems 34\/} (2021),
  16331--16345.

\bibitem{liu2022flow}
{\sc Liu, X., Gong, C., and Liu, Q.}
\newblock Flow straight and fast: Learning to generate and transfer data with
  rectified flow.
\newblock {\em arXiv preprint arXiv:2209.03003\/} (2022).

\bibitem{ma2020slow}
{\sc Ma, C., Wu, L., and Weinan, E.}
\newblock The slow deterioration of the generalization error of the random
  feature model.
\newblock In {\em Mathematical and Scientific Machine Learning\/} (2020), PMLR,
  pp.~373--389.

\bibitem{mao2019mode}
{\sc Mao, Q., Lee, H.-Y., Tseng, H.-Y., Ma, S., and Yang, M.-H.}
\newblock Mode seeking generative adversarial networks for diverse image
  synthesis.
\newblock In {\em Proceedings of the IEEE Conference on Computer Vision and
  Pattern Recognition\/} (2019), pp.~1429--1437.

\bibitem{mao2018effectiveness}
{\sc Mao, X., Li, Q., Xie, H., Lau, R.~Y., Wang, Z., and Smolley, S.~P.}
\newblock On the effectiveness of least squares generative adversarial
  networks.
\newblock {\em IEEE transactions on pattern analysis and machine intelligence
  41}, 12 (2018), 2947--2960.

\bibitem{mescheder2018training}
{\sc Mescheder, L., Geiger, A., and Nowozin, S.}
\newblock Which training methods for {GANs} do actually converge?
\newblock In {\em International conference on machine learning\/} (2018), PMLR,
  pp.~3481--3490.

\bibitem{mirza2014conditional}
{\sc Mirza, M., and Osindero, S.}
\newblock Conditional generative adversarial nets.
\newblock {\em arXiv preprint arXiv:1411.1784\/} (2014).

\bibitem{nagarajan2017gradient}
{\sc Nagarajan, V., and Kolter, J.~Z.}
\newblock Gradient descent {GAN} optimization is locally stable.
\newblock {\em arXiv preprint arXiv:1706.04156\/} (2017).

\bibitem{noe2019boltzmann}
{\sc No\'{e}, F., Olsson, S., K\:{o}hler, J., and Wu, H.}
\newblock Boltzmann generators: Sampling equilibrium states of many-body
  systems with deep learning.
\newblock {\em Science 365}, 6457 (2019), eaaw1147.

\bibitem{nowozin2016f}
{\sc Nowozin, S., Cseke, B., and Tomioka, R.}
\newblock $f$-{GAN}: Training generative neural samplers using variational
  divergence minimization.
\newblock In {\em Advances in neural information processing systems\/} (2016),
  pp.~271--279.

\bibitem{oneto2016tikhonov}
{\sc Oneto, L., Ridella, S., and Anguita, D.}
\newblock Tikhonov, ivanov and morozov regularization for support vector
  machine learning.
\newblock {\em Machine Learning 103}, 1 (2016), 103--136.

\bibitem{oord2018parallel}
{\sc Oord, A., Li, Y., Babuschkin, I., Simonyan, K., Vinyals, O., Kavukcuoglu,
  K., Driessche, G., Lockhart, E., Cobo, L., Stimberg, F., et~al.}
\newblock Parallel {WaveNet}: Fast high-fidelity speech synthesis.
\newblock In {\em International conference on machine learning\/} (2018),
  pp.~3918--3926.

\bibitem{oord2016wavenet}
{\sc Oord, A. v.~d., Dieleman, S., Zen, H., Simonyan, K., Vinyals, O., Graves,
  A., Kalchbrenner, N., Senior, A., and Kavukcuoglu, K.}
\newblock {WaveNet}: {A} generative model for raw audio.
\newblock {\em arXiv preprint arXiv:1609.03499\/} (2016).

\bibitem{oord2016pixel}
{\sc Oord, A. v.~d., Kalchbrenner, N., and Kavukcuoglu, K.}
\newblock Pixel recurrent neural networks.
\newblock {\em arXiv preprint arXiv:1601.06759\/} (2016).

\bibitem{ororbia2022neural}
{\sc Ororbia, A., and Kifer, D.}
\newblock The neural coding framework for learning generative models.
\newblock {\em Nature communications 13}, 1 (2022), 1--14.

\bibitem{papamakarios2017masked}
{\sc Papamakarios, G., Pavlakou, T., and Murray, I.}
\newblock Masked autoregressive flow for density estimation.
\newblock In {\em Advances in Neural Information Processing Systems\/} (2017),
  pp.~2338--2347.

\bibitem{pei2021alleviating}
{\sc Pei, S., Da~Xu, R.~Y., Xiang, S., and Meng, G.}
\newblock Alleviating mode collapse in {GAN} via diversity penalty module.
\newblock {\em arXiv preprint arXiv:2108.02353\/} (2021).

\bibitem{radford2015unsupervised}
{\sc Radford, A., Metz, L., and Chintala, S.}
\newblock Unsupervised representation learning with deep convolutional
  generative adversarial networks.
\newblock {\em arXiv preprint arXiv:1511.06434\/} (2015).

\bibitem{radford2018improving}
{\sc Radford, A., Narasimhan, K., Salimans, T., and Sutskever, I.}
\newblock Improving language understanding by generative pre-training.

\bibitem{rahimi2008uniform}
{\sc Rahimi, A., and Recht, B.}
\newblock Uniform approximation of functions with random bases.
\newblock In {\em 2008 46th Annual Allerton Conference on Communication,
  Control, and Computing\/} (2008), IEEE, pp.~555--561.

\bibitem{ramesh2021zero}
{\sc Ramesh, A., Pavlov, M., Goh, G., Gray, S., Voss, C., Radford, A., Chen,
  M., and Sutskever, I.}
\newblock Zero-shot text-to-image generation.
\newblock In {\em International Conference on Machine Learning\/} (2021), PMLR,
  pp.~8821--8831.

\bibitem{rezende2015inference}
{\sc Rezende, D.~J., and Mohamed, S.}
\newblock Variational inference with normalizing flows.
\newblock {\em arXiv preprint arXiv:1505.05770\/} (2015).

\bibitem{rombach2022high}
{\sc Rombach, R., Blattmann, A., Lorenz, D., Esser, P., and Ommer, B.}
\newblock High-resolution image synthesis with latent diffusion models.
\newblock In {\em Proceedings of the IEEE/CVF Conference on Computer Vision and
  Pattern Recognition\/} (2022), pp.~10684--10695.

\bibitem{rotskoff2019trainability}
{\sc Rotskoff, G.~M., and Vanden-Eijden, E.}
\newblock Trainability and accuracy of neural networks: An interacting particle
  system approach.
\newblock {\em stat 1050\/} (2019), 30.

\bibitem{rout2021generative}
{\sc Rout, L., Korotin, A., and Burnaev, E.}
\newblock Generative modeling with optimal transport maps.
\newblock {\em arXiv preprint arXiv:2110.02999\/} (2021).

\bibitem{salehinejad2017recent}
{\sc Salehinejad, H., Sankar, S., Barfett, J., Colak, E., and Valaee, S.}
\newblock Recent advances in recurrent neural networks.
\newblock {\em arXiv preprint arXiv:1801.01078\/} (2017).

\bibitem{salimans2016improved}
{\sc Salimans, T., Goodfellow, I., Zaremba, W., Cheung, V., Radford, A., and
  Chen, X.}
\newblock Improved techniques for training {GANs}.
\newblock In {\em Advances in neural information processing systems\/} (2016),
  pp.~2234--2242.

\bibitem{santambrogio2015optimal}
{\sc Santambrogio, F.}
\newblock Optimal transport for applied mathematicians.
\newblock {\em Birk{\"a}user, NY 55}, 58-63 (2015), 94.

\bibitem{schmidhuber2015deep}
{\sc Schmidhuber, J.}
\newblock Deep learning in neural networks: An overview.
\newblock {\em Neural networks 61\/} (2015), 85--117.

\bibitem{singh2018minimax}
{\sc Singh, S., and P{\'o}czos, B.}
\newblock Minimax distribution estimation in {Wasserstein} distance.
\newblock {\em arXiv preprint arXiv:1802.08855\/} (2018).

\bibitem{sohl2015deep}
{\sc Sohl-Dickstein, J., Weiss, E., Maheswaranathan, N., and Ganguli, S.}
\newblock Deep unsupervised learning using nonequilibrium thermodynamics.
\newblock In {\em International Conference on Machine Learning\/} (2015), PMLR,
  pp.~2256--2265.

\bibitem{song2021maximum}
{\sc Song, Y., Durkan, C., Murray, I., and Ermon, S.}
\newblock Maximum likelihood training of score-based diffusion models.
\newblock {\em Advances in Neural Information Processing Systems 34\/} (2021),
  1415--1428.

\bibitem{song2019generative}
{\sc Song, Y., and Ermon, S.}
\newblock Generative modeling by estimating gradients of the data distribution.
\newblock {\em Advances in Neural Information Processing Systems 32\/} (2019).

\bibitem{song2020score}
{\sc Song, Y., Sohl-Dickstein, J., Kingma, D.~P., Kumar, A., Ermon, S., and
  Poole, B.}
\newblock Score-based generative modeling through stochastic differential
  equations.
\newblock {\em arXiv preprint arXiv:2011.13456\/} (2020).

\bibitem{su2018variational}
{\sc Su, J.}
\newblock Variational inference: A unified framework of generative models and
  some revelations.
\newblock {\em arXiv preprint arXiv:1807.05936\/} (2018).

\bibitem{sun2018RFM}
{\sc Sun, Y., Gilbert, A., and Tewari, A.}
\newblock On the approximation properties of random {ReLU} features.
\newblock {\em arXiv preprint arXiv:1810.04374\/} (2018).

\bibitem{tabak2018explanation}
{\sc Tabak, E.~G., and Trigila, G.}
\newblock Explanation of variability and removal of confounding factors from
  data through optimal transport.
\newblock {\em Communications on Pure and Applied Mathematics 71}, 1 (2018),
  163--199.

\bibitem{tabak2013family}
{\sc Tabak, E.~G., and Turner, C.~V.}
\newblock A family of nonparametric density estimation algorithms.
\newblock {\em Communications on Pure and Applied Mathematics 66}, 2 (2013),
  145--164.

\bibitem{tabak2010density}
{\sc Tabak, E.~G., and Vanden-Eijnden, E.}
\newblock Density estimation by dual ascent of the log-likelihood.
\newblock {\em Communications in Mathematical Sciences 8}, 1 (2010), 217--233.

\bibitem{teschl2012ODE}
{\sc Teschl, G.}
\newblock {\em Ordinary Differential Equations and Dynamical Systems}.
\newblock AMS, 2012, ch.~3.4, pp.~83--84.

\bibitem{tikhonov1977solutions}
{\sc Tikhonov, A.~N., and Arsenin, V.~Y.}
\newblock {\em Solutions of ill-posed problems}.
\newblock V. H. Winston \& Sons, Washington, D.C.: John Wiley \& Sons, New
  York, 1977.

\bibitem{valsson2014potential}
{\sc Valsson, O., and Parrinello, M.}
\newblock Variational approach to enhanced sampling and free energy
  calculations.
\newblock {\em Physical review letters 113\/} (2014).

\bibitem{vaswani2017attention}
{\sc Vaswani, A., Shazeer, N., Parmar, N., Uszkoreit, J., Jones, L., Gomez,
  A.~N., Kaiser, {\L}., and Polosukhin, I.}
\newblock Attention is all you need.
\newblock {\em Advances in neural information processing systems 30\/} (2017).

\bibitem{villani2003topics}
{\sc Villani, C.}
\newblock {\em Topics in optimal transportation}.
\newblock No.~58. American Mathematical Soc., 2003.

\bibitem{wang2022efficient}
{\sc Wang, D., Wang, Y., Chang, J., Zhang, L., Wang, H., et~al.}
\newblock Efficient sampling of high-dimensional free energy landscapes using
  adaptive reinforced dynamics.
\newblock {\em Nature Computational Science 2}, 1 (2022), 20--29.

\bibitem{wang2019nonparametric}
{\sc Wang, Z., and Scott, D.~W.}
\newblock Nonparametric density estimation for high-dimensional
  data—algorithms and applications.
\newblock {\em Wiley Interdisciplinary Reviews: Computational Statistics 11}, 4
  (2019), e1461.

\bibitem{weed2019sharp}
{\sc Weed, J., and Bach, F.}
\newblock Sharp asymptotic and finite-sample rates of convergence of empirical
  measures in wasserstein distance.
\newblock {\em Bernoulli 25}, 4A (2019), 2620--2648.

\bibitem{wu2019onelayer}
{\sc Wu, S., Dimakis, A.~G., and Sanghavi, S.}
\newblock Learning distributions generated by one-layer relu networks.
\newblock In {\em Advances in Neural Information Processing Systems 32}. Curran
  Associates, Inc., 2019, pp.~8107--8117.

\bibitem{xu2020understanding}
{\sc Xu, K., Li, C., Zhu, J., and Zhang, B.}
\newblock Understanding and stabilizing {GANs’} training dynamics using
  control theory.
\newblock In {\em International Conference on Machine Learning\/} (2020), PMLR,
  pp.~10566--10575.

\bibitem{xu2022overview}
{\sc Xu, Z.-Q.~J., Zhang, Y., and Luo, T.}
\newblock Overview frequency principle/spectral bias in deep learning.
\newblock {\em arXiv preprint arXiv:2201.07395\/} (2022).

\bibitem{xu2019frequency}
{\sc Xu, Z.-Q.~J., Zhang, Y., Luo, T., Xiao, Y., and Ma, Z.}
\newblock Frequency principle: Fourier analysis sheds light on deep neural
  networks.
\newblock {\em arXiv preprint arXiv:1901.06523\/} (2019).

\bibitem{yang2022flow}
{\sc Yang, H.}
\newblock Generalization error of normalizing flows with stochastic
  interpolants.
\newblock {\em arXiv preprint (to appear)\/} (2022).

\bibitem{yang2022potential}
{\sc Yang, H., and E, W.}
\newblock Generalization and memorization: The bias potential model.
\newblock In {\em Proceedings of the 2nd Mathematical and Scientific Machine
  Learning Conference\/} (16--19 Aug 2022), J.~Bruna, J.~Hesthaven, and
  L.~Zdeborova, Eds., vol.~145 of {\em Proceedings of Machine Learning
  Research}, PMLR, pp.~1013--1043.

\bibitem{yang2022GAN}
{\sc Yang, H., and E, W.}
\newblock Generalization error of {GAN} from the discriminator’s perspective.
\newblock {\em Research in the Mathematical Sciences 9}, 1 (2022), 1--31.

\bibitem{yang2022diffusion}
{\sc Yang, L., Zhang, Z., Song, Y., Hong, S., Xu, R., Zhao, Y., Shao, Y.,
  Zhang, W., Cui, B., and Yang, M.-H.}
\newblock Diffusion models: A comprehensive survey of methods and applications.
\newblock {\em arXiv preprint arXiv:2209.00796\/} (2022).

\bibitem{yazici2020empirical}
{\sc Yazici, Y., Foo, C.-S., Winkler, S., Yap, K.-H., and Chandrasekhar, V.}
\newblock Empirical analysis of overfitting and mode drop in {GAN} training.
\newblock In {\em 2020 IEEE International Conference on Image Processing
  (ICIP)\/} (2020), IEEE, pp.~1651--1655.

\bibitem{zhang2022unifying}
{\sc Zhang, D., Chen, R.~T., Malkin, N., and Bengio, Y.}
\newblock Unifying generative models with gflownets.
\newblock {\em arXiv preprint arXiv:2209.02606\/} (2022).

\bibitem{zhang2018monge}
{\sc Zhang, L., E, W., and Wang, L.}
\newblock Monge-{A}mp\`{e}re flow for generative modeling.
\newblock {\em arXiv preprint arXiv:1809.10188\/} (2018).

\bibitem{zhang2019optimal}
{\sc Zhang, O., Lin, R.-S., and Gou, Y.}
\newblock Optimal transport based generative autoencoders.
\newblock {\em arXiv preprint arXiv:1910.07636\/} (2019).

\bibitem{zhang2017discrimination}
{\sc Zhang, P., Liu, Q., Zhou, D., Xu, T., and He, X.}
\newblock On the discrimination-generalization tradeoff in {GAN}s.
\newblock {\em arXiv preprint arXiv:1711.02771\/} (2017).

\bibitem{zhao2016energy}
{\sc Zhao, J., Mathieu, M., and LeCun, Y.}
\newblock Energy-based generative adversarial network.
\newblock {\em arXiv preprint arXiv:1609.03126\/} (2016).

\end{thebibliography}
\bibliographystyle{acm}

\end{document}